\documentclass[12pt]{article}

\usepackage{amsmath}
\usepackage{times}
\usepackage[dvipsone]{graphicx}
\usepackage{color}
\usepackage{multirow}
\usepackage[authoryear]{natbib}
\usepackage{rotating}
\usepackage{bbm}
\usepackage{latexsym}
\usepackage{epstopdf}   

\usepackage{romannum}
\AtBeginDocument{\pagenumbering{arabic}}

\makeatletter
\@addtoreset{equation}{section}
\makeatother
\renewcommand{\theequation}{\arabic{section}.\arabic{equation}}

\newcommand{\captionfonts}{\normalsize}

\makeatletter
\long\def\@makecaption#1#2{%
  \vskip\abovecaptionskip
  \sbox\@tempboxa{{\captionfonts #1: #2}}%
  \ifdim \wd\@tempboxa >\hsize
    {\captionfonts #1: #2\par}
  \else
    \hbox to\hsize{\hfil\box\@tempboxa\hfil}%
  \fi
  \vskip\belowcaptionskip}
\makeatother

\usepackage{amsmath}
\usepackage{mathrsfs}
\usepackage{algorithm}
\usepackage{algorithmic}
\makeatletter
\renewcommand{\fnum@algorithm}{\fname@algorithm}
\makeatother

\usepackage{romannum}
\AtBeginDocument{\pagenumbering{arabic}}

\usepackage{amsthm}

\usepackage{pmat}
\usepackage{subfig}
\usepackage{caption}

\usepackage{amssymb}

\usepackage[retainorgcmds]{IEEEtrantools}

\renewcommand{\theequationdis}{{\normalfont (\theequation)}}
\renewcommand{\theIEEEsubequationdis}{{\normalfont (\theIEEEsubequation)}}

\makeatletter
\newcommand*{\rom}[1]{\expandafter\@slowromancap\romannumeral #1@}
\makeatother

\newtheorem{thm}{Theorem}
\newtheorem{lem}{Lemma}
\newtheorem{dfn}{Definition}
\newtheorem{prp}{Proposition}
\newtheorem*{rmk}{Remark}
\newtheorem{rmk-2}{Remark}
\newtheorem{rmk-3}{Remark}
\newtheorem{rmk-4}{Remark}
\newtheorem{rmk-5}{Remark}
\newtheorem{rmk-6}{Remark}
\newtheorem{rmk-7}{Remark}
\newtheorem{rmk-8}{Remark}
\newtheorem{cl}{Corollary}

\usepackage[hyperindex,breaklinks]{hyperref}
\hypersetup{
           breaklinks=true,   
           colorlinks=true,   
           pdfusetitle=true,  
        }
\usepackage{setspace}
\begin{document}

\ \vspace{20mm}\\

{\LARGE \flushleft Understanding Two-Layer Neural Networks with Smooth Activation Functions}

\ \\
{\bf \large Changcun Huang}\\
{cchuang@mail.ustc.edu.cn}\\
{Shuitu Institute of Applied Mathematics, Chongqing 400700, P.R.C}\\
%


\thispagestyle{empty}
\markboth{}{NC instructions}
\ \vspace{-0mm}\\
%
\begin{center} {\bf Abstract} \end{center}
This paper aims to understand the training solution, which is obtained by the back-propagation algorithm, of two-layer neural networks whose hidden layer is composed of the units with smooth activation functions, including the usual sigmoid type most commonly used before the advent of ReLUs. The mechanism contains four main principles: construction of Taylor series expansions, strict partial order of knots, smooth-spline implementation and smooth-continuity restriction. The universal approximation for arbitrary input dimensionality is proved and the explanation of training solutions is given. Through the principles proposed, the mystery of ``black box'' of the solution space is largely revealed. The new proofs employed also enrich approximation theory.


\ \\[-2mm]

\section{Introduction}
Smooth activation functions (e.g., logistical sigmoid and tanh cases) were the most common choice in neural networks \citep*{Bishop1995,Hinton2006, Haykin2009} before the popularity of rectified linear units (ReLUs) \citep*{Nair2010} since 2010. In engineering, the mechanism of two-layer neural networks with smooth activation functions is elusive and referred to as a ``black box''. We study this problem due to the following reasons. First, it is of the simplest network architecture leading to a deeper one and the associated principle may be the foundation of the latter. Second, the activity of neurons \citep*{Kandel2021} could be complex and the logistic sigmoid function $\sigma(x)=1/(1+e^{-x})$ had been employed to model their behavior (e.g., \citet*{Arbib2016}); so the principle is possibly useful in understanding neuroscience.

\subsection{Problem Description}
Given a set $D$ of data points $(\boldsymbol{x}_i, y_i)$ with $\boldsymbol{x}_i \in \mathbb{R}^n$ and $y_i \in \mathbb{R}$, a two-layer neural network $\mathfrak{N}$ can interpolate them by an output function
\begin{equation}
g(\boldsymbol{x})=\sum_j\lambda_j\sigma(\boldsymbol{w}_j^T\boldsymbol{x}+b_j)
\end{equation}
with $g(\boldsymbol{x}_i) = y_i$, whose parameters are derived from the back-propagation algorithm \citep*{Rumelhart1986}, where $\sigma(x)$ is a smooth activation function of hidden-layer units.

The mechanism of $g(\boldsymbol{x})$ interpolating $D$ is unknown and is part of the ``black box'' of neural networks. We want to solve this problem by constructing the solution of universal approximation. The interpolation process can be interpreted as realizing a function that passes through the given data; thus if any function can be constructed in the way associated with training solutions, the black box could be uncovered.

In theoretical physics, the capability of explaining or predicting natural phenomena is a criterion for judging whether a proposed theory is successful. We also use this criterion to test the effectiveness of our theory. More specifically, first, given a training solution, we will explain how it is formulated and what the meaning of the obtained parameters is; second, we will check whether the solution predicted by the theory would occur in experiments. Especially, in the realm of artificial neural networks, the first rule above means that we can manually construct a solution in a deterministic way, which is originally obtained by gradient-descent method; and this is one of the ingredients of black box of neural networks.

Hereafter, we will call $g(\boldsymbol{x})$ of equation 1.1 derived from the back-propagation algorithm a \textsl{training solution} of two-layer neural networks.

\subsection{Main Background}
The universal approximation of two-layer neural networks with ``sigmoidal unit'' \citep*{Cybenko1989}, whose activation function is, for example, $\sigma(x)=1/(1+e^{-x})$, had been proved for more than thirty years in several ways, such as \citet*{Hecht-Nielsen1989}, \citet*{Hornik1989} and \citet*{Chen1992}. The case of more general smooth activation functions was also intensively studied (e.g., \citet*{Stinchcombe1989, Leshno1993, Hornik1993, Pinkus1999, Ismailov2012, Costarelli2013, Guliyev2018, Almira2021, Cantarini2025}). However, those works are mainly of pure mathematical research and have no connection with engineering applications. Thus, so far little was known about the mechanism of training solutions. We try to bridge this gap in this paper.

\subsection{Organization of the Paper}
The paper is mainly classified into three parts. The first that includes sections from 2 to 4 is for univariate-function approximation and the second (sections 5 and 6) is for the multivariate case. The third (section 7) is about experimental verification or explanation. The detailed description is as follows.

\textbf{Section 2} constructs local approximation through Taylor series expansions and the result is applicable to any smooth activation function. The construction method will be generalized to the multivariate case in later section 5. The result of this section has its own meaning and is also the foundation of global solutions of subsequent sections.

\textbf{Section 3} turns to global approximation, in the sense that a two-layer neural network $\mathfrak{N}$ with generalized sigmoidal units (definition 3) can approximate a given function $f(x)$ through a smooth spline $s(x)$, not necessarily within a sufficiently short interval as the method of Taylor series expansions. The spline $s(x)$ is constructed by integrating a piecewise linear approximation to the $k$th derivative of $f(x)$ (theorem 4); $s(x)$ is then implemented by $\mathfrak{N}$ (theorem 6), with the approximation error controlled by theorem 5; and this leads to universal approximation for univariate functions (theorem 7). Corollary 1 investigates the tanh-unit case. We give another solution in proposition 2 that could be observed in experiments.

\textbf{Section 4} reduces the spline implementation to the solution of a system of linear equations, by which a theoretical framework containing all the possible solutions is presented, including the special ones of section 3.

\textbf{Section 5} constructs the Taylor series expansion of a multivariate function $f(\boldsymbol{x})$ for $\boldsymbol{x} \in \mathbb{R}^n$ and $n \ge 2$ via neural networks and is the generalization of section 2. \textbf{Section 6} deals with the global approximation to $f(\boldsymbol{x})$ as section 3. A multivariate spline (theorem 11) approximating $f(\boldsymbol{x})$ is obtained by the integral of a piecewise linear approximation to a directional-derivative hypersurface (definition 14) of $f(\boldsymbol{x})$. A basic property of two smoothly connected polynomial pieces is given in theorem 12. The principle of smooth-continuity restriction is proposed in theorem 15. Theorem 17 constructs a desired spline on a certain type of partition through neural networks. Theorem 18 proves the universal approximation for arbitrary higher-dimensional input. Theorem 19 and proposition 8 provide some rules for two-sided solutions ubiquitous in experiments.

\textbf{Section 7} verifies the theory by some typical experimental examples or uses the preceding results to explain experimental phenomena. \textbf{Section 8} summarizes the main principles of the black box of two-layer neural networks. \textbf{Section 9} ends this paper by a discussion.

\section{Univariate Local Approximation}
Throughout this paper, the ``smooth function'' $f: [0, 1] \to \mathbb{R}$ means that the derivatives of $f(x)$ up to some order required are continuous, and if its $k$th-order derivative satisfies the continuous property, we sometimes say that it is smooth with order $k$ or $f(x) \in C^k([0, 1])$, where $C^k([0, 1])$ denotes the set of functions on $[0, 1]$ whose $k$th derivative is continuous; $C^k(\mathbb{R})$ represents the set of functions defined on $(-\infty, +\infty)$ with smoothness order $k$. The multivariate case is similar and the rigorous definition will be given in section 5.

The term ``local approximation'' is in the sense that the approximation is based on Taylor series expansion whose approximation error can only be ensured within a sufficiently small neighborhood of a point, and so is the multivariate case of section 5. Note that the results of this section are applicable to any smooth activation function, including the generalized sigmoidal type to be discussed later.

\subsection{Polynomial Construction}
The main idea is to use weighed activation functions of the hidden-layer units of a two-layer neural network to construct a desired Taylor-series expansion approximating a given function; theorem 1 below is a foundation.

In this paper, the approximation error is measured by the $L^2$-norm distance, that is, $\lVert f(\boldsymbol{x})\rVert_2=(\int_{U} |f(\boldsymbol{x})|^2 d\boldsymbol{x})^{1/2}$ for $\boldsymbol{x}\in \mathbb{R}^n$, including the special case $n=1$ of this section.
\begin{thm}
Let $T(x) = \sum_{\nu = 0}^{m}a_{\nu}/\nu !(x - x_0)^{\nu}$ for $x_0 \in [0, 1]$ be an arbitrary polynomial of degree $m$. Write
\begin{equation}
\phi_i(x) = \sigma(w_ix + b_i).
\end{equation}
Each $\phi_i(x)$ for $i = 1, 2, \dots, k$ is the smooth activation function of unit $u_i$ whose Taylor series expansion at $x_0$ is
\begin{equation}
\phi_i(x) = t_i(x) + r_i(x),
\end{equation}
where $t_i(x) = \sum_{\nu = 0}^{m}c_{i\nu}/\nu !(x - x_0)^{\nu}$ with $c_{i\nu} = \phi_i^{(\nu)}(x_0)$ and
\begin{equation}
r_i(x) = \frac{1}{(m+1)!}\phi_i^{(m+1)}(\xi_i)(x-x_0)^{m+1}
\end{equation}
is the Lagrange remainder with $\xi_i$ being in a neighborhood $\delta(x_0)$ of $x_0$.

Suppose that $k \ge m+1$ and the rank of
\begin{equation}
A(\phi_i(x_0)) = \begin{pmat}({})
\phi_1(x_0) & \phi_2(x_0) & \cdots & \phi_k(x_0) \cr
\phi_1'(x_0) & \phi_2'(x_0) & \cdots & \phi_k'(x_0) \cr
\vdots & \vdots & \ddots & \vdots \cr
\phi_1^{(m)}(x_0) & \phi_2^{(m)}(x_0) & \cdots & \phi_k^{(m)}(x_0) \cr
\end{pmat}
\end{equation}
is $m+1$, and that
\begin{equation}
\lVert \phi_i(x) - t_i(x) \rVert_2 = \lVert r_i(x) \rVert_2 = \varepsilon_i.
\end{equation}
Then we can find a solution of $\boldsymbol{\lambda} = [\lambda_1, \lambda_2, \dots, \lambda_k]^T$ for
\begin{equation}
T(x) = \sum_{i}\lambda_it_i(x),
\end{equation}
such that
\begin{equation}
\lVert T(x) - \sum_{i}\lambda_i\phi_i(x) \rVert_2  = \lVert \sum_i\lambda_ir_i(x)\rVert_2 \le \sum_i|\lambda_i|\varepsilon_i.
\end{equation}
\end{thm}
\begin{proof}
To satisfy equation 2.6, we obtain a system of linear equations, with each being
\begin{equation}
c_{1\nu}\lambda_1 + c_{2\nu}\lambda_2 + \dots + c_{k\nu}\lambda_k = a_{\nu}
\end{equation}
for $\nu = 0, 1, \dots, m$, which can be expressed as
\begin{equation}
A\boldsymbol{\lambda} = \boldsymbol{a},
\end{equation}
where $A$ is from equation 2.4 and $\boldsymbol{a} = [a_0, a_1, \dots, a_m]^T$. So
\begin{equation}
\begin{aligned}
\lVert T(x) - \sum_{i}&\lambda_i\phi_i(x) \rVert_2 = \lVert \sum_{i}\lambda_it_i(x) - \sum_{i}\lambda_i\phi_i(x) \rVert_2 \\ &= \lVert \sum_i\lambda_ir_i(x)\rVert_2\le \sum_i|\lambda_i|\varepsilon_i,
\end{aligned}
\end{equation}
which is equation 2.7.
\end{proof}

\begin{dfn}[Weighted error]
Under theorem 1, the term
\begin{equation}
\epsilon(m) := \lVert \sum_i\lambda_ir_i(x)\rVert_2 = \lVert \sum_i \frac{\lambda_i\phi_i^{(m+1)}(\xi_i)}{(m+1)!}(x-x_0)^{m+1}\rVert_2
\end{equation}
from equation 2.7 is called the weighted error of the approximation of $\phi_i(x)$'s to $T(x)$.
\end{dfn}

\subsection{Generalized Wronskian Matrix}
Wronskian determinant was studied in the theory of total positivity \citep*{Karlin1966} and the associated term ``Wronskian matrix'' in \citet*{Schumaker2007} is a square matrix. We borrow that terminology with slight modifications.
\begin{dfn}[Univariate generalized Wronskian matrix]
Given the functions $u_i: [0,1] \to \mathbb{R}$ for $i = 1, 2, \dots, k$, each of which is a smooth function, the matrix \begin{equation}
\mathcal{W}(u_i(x_0)) =
\begin{pmat}({})
u_1(x_0) & u_1'(x_0) & \cdots & u_1^{(m)}(x_0) \cr
u_2(x_0) & u_2'(x_0) & \cdots & u_2^{(m)}(x_0) \cr
\vdots & \vdots & \ddots & \vdots \cr
u_k(x_0) & u_k'(x_0) & \cdots & u_k^{(m)}(x_0) \cr
\end{pmat}
\end{equation}
for $x_0 \in [0, 1]$ and $k \ge m+1$ is called a \textsl{generalized Wronskian matrix} of $u_i(x_0)$'s with respect to order $m$.
\end{dfn}

\begin{rmk}
By definition 2, equation 2.9 can be written as
\begin{equation}
\mathcal{W}(\phi_i(x_0))^T\boldsymbol{\lambda} = \boldsymbol{a}.
\end{equation}
\end{rmk}

\begin{thm}
Let $f: [0, 1] \to \mathbb{R}$ be a smooth function whose Taylor series expansion at $x_0 \in [0, 1]$ is $f(x) =  T(x) + R(x)$, where $T(x) = \sum_{\nu = 0}^{m}a_{\nu}/\nu!(x - x_0)^{\nu}$ is a polynomial of degree $m$ and $R(x)$ is the Lagrange remainder. Under the notations of theorem 1, suppose that the generalized Wronskian matrix $\mathcal{W}(\phi_i(x_0))$ for $k=m+1$ and $i=1,2,\dots,m+1$ is nonsingular. Then
\begin{equation}
\lVert f(x) - \sum_{i}\lambda_i\phi_i(x) \rVert_2  \le \lVert R(x)\rVert_2 + \epsilon(m) = \varepsilon,
\end{equation}
where $\epsilon(k)$ is from equation 2.11. Let $\delta(x_0)$ be a neighbourhood of point $x_0$ such that $|x - x_0| < \delta$. Fixing degree $m$ and the derived $\lambda_i$'s, if $\delta$ is small enough, the approximation error $\varepsilon$ of equation 2.14 over $\delta(x_0)$ can be arbitrarily small.
\end{thm}
\begin{proof}
By theorem 1, we have
\begin{equation}
\begin{aligned}
\lVert f(x) - \sum_{i}&\lambda_i\phi_i(x) \rVert_2 = \lVert f(x) - T(x) + T(x) - \sum_{i }\lambda_i\phi_i(x) \rVert_2 \\
&\le \lVert f(x) - T(x)\rVert_2 + \lVert T(x) - \sum_{i}\lambda_i\phi_i(x) \rVert_2 \\
&= \lVert R(x)\rVert_2 + \epsilon(m) = \varepsilon.
\end{aligned}
\end{equation}
Since $\phi_i(x)$'s are smooth, $\phi_i^{(m+1)}(\xi_i)$ of equation 2.11 is bounded. If $m$ and $\lambda_i$'s are fixed, when $|x-x_0|$ is small enough, the weighted error $\epsilon(k)$ could approach zero with arbitrary precision, because $\varepsilon_i$'s of equation 2.7 tend to be zero; and so is $\lVert R(x)\rVert_2$.
\end{proof}

\begin{rmk}
The ``generalized Taylor Expansion'' \citep*{Schumaker2007} gave a formula like equation 2.14. However, both the remainder expression and the associated proof are different; the approximation error is not analyzed because their goal was not function approximation.
\end{rmk}

\subsection{Construction of Local Approximation}
\begin{lem}
If $\psi(x): [0, 1] \to \mathbb{R}$ is a smooth function, then
\begin{equation}
\lim_{w \to 0}\psi^{(k)}(wx) = \lim_{w \to 0}\frac{d^k\psi(wx)}{dx^k} = 0
\end{equation}
for $k \ge 1$.
\end{lem}
\begin{proof}
We know that
\begin{equation}
\psi^{(k)}(wx) = w^k\frac{d^k\psi(y)}{dx^k} = w^k\psi^{(k)}(y),
\end{equation}
where $y = wx$. When $w$ is sufficiently small, we have $y \in [0, 1]$. Then $\psi^{(k)}(y)$ is bounded, that is, $|\psi^{(k)}(y)| \le M$, owing to the smoothness property of $\psi(x)$; the conclusion is obvious.
\end{proof}

\begin{lem}
Notations being from theorem 2, to arbitrary $x_0 \in [0, 1]$, the parameters $w_i$ and $b_i$ of $\phi_i(x) = \sigma(w_ix + b_i)$ for all $i = 1, 2, \dots, m+1$ can be adjusted to yield a nonsingular generalized Wronskian matrix $\mathcal{W}(\phi_i(x_0))$ with arbitrary precision.
\end{lem}
\begin{proof}
The main idea is to control the degree of the Taylor series expansion of each $\phi_i(x)$ by scaling the parameter $w_i$ at different levels, as adjusting the scaling parameter of wavelets. To each $\phi_i(x) = \sigma(w_ix + b_i)$, let $y = w_ix + b_i$ and $y_i = w_ix_0 + b_i$. By equation 2.12, the corresponding generalized Wronskian matrix is
\begin{equation}
\mathcal{W}(\sigma(y_i)) =
\begin{pmat}({})
\sigma(y_1) & \sigma'(y_1)w_1  & \cdots & \sigma^{(m)}(y_1)w_1^{m} \cr
\sigma(y_2) & \sigma'(y_2)w_2  & \cdots & \sigma^{(m)}(y_2)w_2^{m} \cr
\vdots & \vdots & \ddots & \vdots \cr
\sigma(y_{m+1}) & \sigma'(y_{m+1})w_{m+1} & \cdots & \sigma^{(m)}(y_{m+1})w_{m+1}^{m} \cr
\end{pmat}.
\end{equation}
Let $w_1 = \Delta t^{(1+c)}$ for $c > 0$, $w_j = \Delta t^{1/(j-1)}$ for $j = 2, 3, \dots, m$, where $0 < \Delta t < 1$, and $w_{m+1} =\Delta t^{-(m-1)}$; then equation 2.24 becomes
\begin{equation}
\begin{aligned}
&\ \ \ \ \ \ \ \ \ \ \ \ \ \ \ \ \ \ \ \ \ \ \ \ \ \ \ \ \ \ \ \ \ \ \ \ \ \ \ \mathcal{W}(\sigma(y_i), \Delta t) = \\
&\begin{pmat}({})
\sigma(y_1) & \sigma'(y_1)\Delta t^{1+c} & \sigma''(y_1)\Delta t^{2(1+c)} & \cdots & \sigma^{(m)}(y_1)\Delta t^{m(1+c)} \cr
\sigma(y_2) & \sigma'(y_2)\Delta t^{1/1} & \sigma''(y_2)\Delta t^{2/1} & \cdots & \sigma^{(m)}(y_2)\Delta t^{m/1} \cr
\sigma(y_3) & \sigma'(y_3)\Delta t^{1/2} & \sigma''(y_3)\Delta t^{2/2} & \cdots & \sigma^{(m)}(y_3)\Delta t^{m/2} \cr
\vdots & \vdots & \vdots & \ddots & \vdots \cr
\sigma(y_{m+1}) & \sigma'(y_{m+1})\Delta t^{1/m-1} & \sigma''(y_{m+1})\Delta t^{2(1/m-1)} & \cdots & \sigma^{(m)}(y_{k})\Delta t^{-(m-1)} \cr
\end{pmat},
\end{aligned}
\end{equation}
which can be expressed as
\begin{equation}
\mathcal{W}(\sigma(y_i), \Delta t) = \begin{pmat}({})
\sigma(y_1) & o(\Delta t) & \cdots & o(\Delta t) \cr
\sigma(y_2) & \sigma'(y_2)\Delta t & \cdots & o(\Delta t) \cr
\vdots & \vdots & \ddots & \vdots \cr
\sigma(y_{m+1}) & \sigma'(y_{m+1})\Delta t^{1/m-1} & \cdots & \sigma^{(m)}(y_{m+1})\Delta t^{-(m-1)} \cr
\end{pmat},
\end{equation}
where ``$o$'' represents the little-oh notation. Among the entries of $\mathcal{W}(\sigma(y_i), \Delta t)$, neglecting those infinitesimal ones relative to $\Delta t$, we have

\begin{equation}
\mathcal{W}'(\Delta t) = \lim_{\Delta t \to 0}\mathcal{W}(\sigma(y_i), \Delta t) = \begin{pmat}({})
\sigma(y_1) & 0 & \cdots & 0 \cr
\sigma(y_2) & \sigma'(y_2)\Delta t & \cdots & 0 \cr
\vdots & \vdots & \ddots & \vdots \cr
\sigma(y_{m+1}) & \sigma'(y_{m+1})\Delta t & \cdots & \sigma^{(m)}(y_{m+1})\Delta t^{-(m-1)} \cr
\end{pmat},
\end{equation}
whose determinant
\begin{equation}
\det \mathcal{W}'(\Delta t)) = \prod_{i=0}^{m}\sigma^{(i)}(y_{i+1}) \ne 0
\end{equation}
when each bias $b_i$ in $y_i = w_ix_0 + b_i$ is properly set to make $\sigma^{(i)}(y_{i+1}) \ne 0$. Thus, $\mathcal{W}'(\Delta t)$ is nonsingular.
\end{proof}

\begin{rmk-2}
We explain the meaning of the proof of lemma 2. By equation 2.21, write
\begin{equation}
\mathcal{W}'(\Delta t)^T =
\begin{pmat}({})
\sigma(y_1) & \sigma(y_2) & \cdots & \sigma(y_{m+1}) \cr
0 & \sigma'(y_2)\Delta t & \cdots & \sigma'(y_{m+1})\Delta t \cr
\vdots & \vdots & \ddots & \vdots \cr
0 & 0 & \cdots & \sigma^{(m)}(y_{m+1})\Delta t^{-(m-1)} \cr
\end{pmat}.
\end{equation}
We know that $\mathcal{W}'(\Delta t)^T\boldsymbol{\lambda} = \boldsymbol{b}$ (equation 2.13) produces the coefficients of the desired polynomial $p(x) = a_0 + a_1x + \dots +a_{m}x^{m}$. Note that the $i$th column for $i=1,2,\dots,m+1$ of $\mathcal{W}'(\Delta t)^T$ represents the function $\phi_i(x) = \sigma(w_ix + b_i)$. To $\phi_1(x)$, only the constant term of its Taylor series expansion is valid via scaling $w_1$. The case of $\phi_2(x)$, as indicated in the second column of $\mathcal{W}'(\Delta t)^T$, includes the linear portion of the Taylor series expansion of $\phi_2(x)$, also by scaling $w_2$. The remaining columns can be analogously interpreted. Through the parameter setting, the coefficient $a_{m}$ of $p(x)$ is formed by only $\phi_{m+1}(x)$ and the output weight $\lambda_{m+1}$ can be easily determined. By recursively operation from the last one, we can obtain all of $a_i$'s.
\end{rmk-2}

\begin{rmk-2}
Although lemma 2 is general for arbitrary smooth activation functions, its mechanism may not be ubiquitous in nonsingular generalized Wronskian matrices. One reason is that the weight parameters must be in (0, 1). So lemma 2 can only be regarded as a special solution or an existence proof.
\end{rmk-2}

\begin{rmk-2}
We can see the usefulness of the bias parameter in equation 2.22, that is, finding a suitable position of an activation function to produce nonzero $\sigma^{(i)}(y_{i+1})$'s.
\end{rmk-2}

\begin{thm}[Local approximation]
Notations being from theorem 2, to any smooth function $f: [0, 1] \to \mathbb{R}$, within a small enough neighbourhood $\delta(x_0)$ of arbitrary point $x_0 \in [0, 1]$, a two-layer network $\mathfrak{N}$, whose hidden layer is composed of the units with smooth activation functions, can approximate it with arbitrary precision, in terms of realizing its Taylor series expansion $p(x)$ at $x_0$. If the degree of $p(x)$ is $m$, the number $\Theta$ of the units of the hidden layer of $\mathfrak{N}$ required is $\Theta \ge m+1$.
\end{thm}
\begin{proof}
According to theorems 1 and 2, the key point is to construct a nonsingular Wronskian matrix, for which lemma 2 provided a solution.
\end{proof}

\subsection{Related Work}
\citet*{Leshno1993} and \citet*{Pinkus1999} constructed $x^k$ through smooth activation functions; \citet*{Xu2005} also realized Taylor series expansions via a two-layer neural network. Our difference lies in three aspects: (a) The framework of theorem 2 is general and can include their solutions as special cases; (b) Our construction method is novel and can be directly generalized to the multivariate case without requiring additional ridge functions.

\section{Univariate Global Approximation}

The approximation via Taylor series expansion is local and the approximation error can only be assured within a sufficiently small neighborhood of a point, for which a new mechanism related to splines is introduced and the associated solution is called a global one. Both theoretical and experimental results will demonstrate the necessity of this introduction.

\textbf{Section 3.1} provides several definitions, notations and some knowledge of splines. \textbf{Section 3.2} proposes a spline-construction method for the convenience of  network implementation. \textbf{Section 3.3} paves the way for approximation-error control. \textbf{Section 3.4} defines local and global units for different solution patterns. \textbf{Section 3.5} constructs a concrete solution (one-sided case) for universal approximation. \textbf{Section 3.6} gives another solution (two-sided case) to explain experiments.

\subsection{Preliminaries}

\begin{dfn}[Generalized sigmoidal unit]
A \textsl{generalized sigmoidal} unit is one whose activation function $\sigma(x)\in C^k(\mathbb{R})$ for $k \ge 1$ satisfies the following conditions: (a) increasing monotonically on $(-\infty, 0]$;  (b) $\lim_{\boldsymbol{x} \to -\infty}\sigma(x) = 0$. We also call $\sigma(x)$ a generalized sigmoidal function.
\end{dfn}
\begin{rmk}
This is the generalization of logistical sigmoid function $\sigma(x)=1/(1+e^{-x})$. Compared to \citet*{Cybenko1989}'s definition, ours has no restriction $\lim_{\boldsymbol{x} \to +\infty}\sigma(x) = 1$ but requires the monotonically increasing condition. Our monotonic condition differs from \citet*{Hornik1989}'s in that it is needed only on $(-\infty, 0]$, part of the domain on which the activation function is defined.
\end{rmk}

\begin{dfn}[Generalized tanh unit]
If the condition $(b)$ of definition 3 is modified to $\lim_{\boldsymbol{x} \to -\infty}\sigma(x) = \mathcal{C}\ne0$, the corresponding unit is called a generalized tanh unit and the activation function is called a generalized tanh function.
\end{dfn}
\begin{rmk}
This definition is generalized from hyperbolic tangent function $\sigma(x)=(e^{x}-e^{-x})/(e^{x}+e^{-x})$ of tanh units. Because this type of unit can be easily reduced to the case of generalized sigmoidal units (see later corollary 1), we will mainly investigate the latter.
\end{rmk}

We develop the theory on the basis of splines and first introduce some related concepts, notations and elementary results. Write
\begin{equation}
\begin{cases}
\ x_{+} := \max(0, x)
\\
\ \ \ \ x_{+}^m := (x_+)^m
\end{cases}
\end{equation}
\citep*{Chui1992}. The space of the order-$m$ splines discussed in this paper is
\begin{equation}
\begin{aligned}
\mathfrak{S}_1^m(\Delta)  =  \{s: s(x)=s_i(x) \in \mathcal{P}_m \ &\text{for} \ x \in I_i,  D^{\mu}s_{i}(x_{i})=D^{\mu}s_{i+1}(x_{i}) \ \text{for} \ i \ne \zeta, \\ \ &i = 1, 2, \dots, \zeta\},
\end{aligned}
\end{equation}
where $\mathcal{P}_m$ is the set of polynomials of degree $m$, $I_1 = [x_0, x_1]$ and $I_j = (x_{j-1}, x_{j}]$ for $j = 2, 3, \dots, \zeta$ are the intervals formed by $0 = x_0 < x_1 < x_2 < \cdots < x_{\zeta-1} < x_{\zeta} = 1$, $D^{\mu}$ is the $\mu$th derivative operator for $\mu = 0, 1, \dots, m-1$, and
\begin{equation}
\Delta = \{x_{\nu}: \nu = 1, 2, \dots, \zeta-1\}
\end{equation}
is the set of the knots. Throughout this paper, we call $I_1$ and $I_j$'s the subintervals of $\Delta$ or we say that they are derived from $\Delta$. The notation of equation 3.2 is of \citet*{Schumaker2007}'s modified version and $\mathfrak{S}_1^m(\Delta)$ is equivalent to \citet*{Powell1981}'s $\mathscr{S}(m, x_0, x_1, \dots, x_{\zeta})$.

Note that each $s(x) \in \mathfrak{S}_m(\Delta)$ belongs to $C^{m-1}([0, 1])$. By the property of smooth splines, two adjacent polynomials $s_{\nu}(x)$ and $s_{\nu-1}(x)$ of $s(x)$ for $\nu = 2, 3, \dots, \zeta$ satisfy the recurrence formula
\begin{equation}
s_{\nu}(x) = s_{\nu-1}(x) + \alpha_{\nu-1}(x - x_{\nu-1})^m_+
\end{equation}
for $x\in I_{\nu}$, which ensures $D^{\mu}s_{\nu}(x_{\nu-1})=D^{\mu}s_{\nu-1}(x_{\nu-1})$ of equation 3.2 and is the foundation of spline construction in this paper.

\subsection{Splines via Local Linear Approximation}
The approximation capability of a spline $s(x) \in \mathfrak{S}_1^m(\Delta)$ to continuous functions had been proved (e.g., \citet*{Powell1981}). In order to realize a spline through neural networks, we give a new construction method.

\begin{thm}[Univariate-spline construction]
Let $f(x) \in C^{m}([0, 1])$ and $\Delta $ of equation 3.3 be the set of the knots. Based on a piecewise linear approximation to the $m-1$th derivative $f^{(m-1)}(x)$ of $f(x)$, a spline $s(x) \in \mathfrak{S}_1^m(\Delta)$ can be constructed to approximate $f(x)$ with arbitrary precision, provided that the length of each subinterval $I_i$ for $1\le i \le \zeta$ derived from $\Delta$ is sufficiently small.
\end{thm}
\begin{proof}
We first give an example. To any point $x_0 \in (0, 1)$, a linear function $l(x) = ax + b$ can be found to approximate $f'(x)$ within a neighborhood $\delta(x_0)$ of $x_0$, where parameters $a$ and $b$ can be obtained by the first-order Taylor series expansion of $f(x)$ at $x_0$. Then we have
\begin{equation}
f'(x) = ax + b + o(x-x_0).
\end{equation}
The original function can be represented as
\begin{equation}
f(x) = F(x) := \int_{x_0}^xf'(x)dx + f(x_0),
\end{equation}
since $F(x_0) = f(x_0)$ and $F'(x) = f'(x)$; so
\begin{equation}
f(x) = \frac{ax^2}{2} + bx + c + \int_{x_0}^xr(x)dx,
\end{equation}
where $c = -a^2x_0^2/2 - bx_0 + f(x_0)$ and $r(x) = o(x-x_0)$. Because $r(x)$ is continuous and bounded in $\delta(x_0)$, it has a maximum $r(x_{\tau})$ at some point $x_{\tau} \in \delta(x_0)$, which is also $o(x-x_0)$. Thus,
\begin{equation}
\int_{x_0}^xr(x)dx < r(x_{\tau})(x-x_0) = o\big((x-x_0)^2\big).
\end{equation}
That is,
\begin{equation}
f(x) = \frac{ax^2}{2} + bx + c + o\big((x-x_0)^2\big),
\end{equation}

Divide $[0, 1]$ into $\zeta$ intervals $I_i$'s. If the length of each $I_i$ is sufficiently small, the linear approximation above can be done for all of $I_i$'s, yielding
\begin{equation}
\begin{aligned}
f(x) &= \frac{a_ix^2}{2} + b_ix + c_i + r_i(x) \\
&= \frac{a_ix^2}{2} + b_ix + c_i + o\big((x-x_{i0})^2\big)
\end{aligned}
\end{equation}
for $x, x_{i0} \in I_i$. Let $p_{i}(x) = a_ix^2/2 + b_ix + c_i$. When $i = 1$, the parameter $c_1$ of $p_1(x)$ is obtained by fulfilling a selected point of $f(x)$ on $[0, x_1]$. To ensure the continuity between $p_{j}(x)$ and $p_{j-1}(x)$ for $j = 2, 3, \dots, \zeta$, use $p_{j-1}(x_{j-1}) = p_{j}(x_{j-1})$ to determine $c_j$. Then all of $p_{i}(x)$ 's comprise a smooth spline $\mathcal{S}(x) \in \mathfrak{S}_1^2(\Delta)$, which can approximate $f(x)$ with arbitrary precision, provided that the length of each $I_i$ is sufficiently small; the approximation error is
\begin{equation}
\begin{aligned}
E = \lVert \mathcal{S}(x) &- f(x) \rVert_2 =  \big(\sum_i\int_{x_{i-1}}^{x_i} (p_i(x) - f(x))^2dx\big)^{1/2} \\
& < \big(\sum_i\int_{x_{i-1}}^{x_i} r_i(x)^2dx\big)^{1/2} \\
& \le \big(\sum_i\int_{x_{i-1}}^{x_i} \varepsilon^2dx\big)^{1/2}
= \varepsilon,
\end{aligned}
\end{equation}
where
\begin{equation}
\varepsilon = \max_i\{\beta_i = \sup_{x \in I_i}|r_i(x)|: 1\le i\le \zeta\}
\end{equation}
can be arbitrarily small when $l = \max\{|x_{i+1} - x_{i}|: 0 \le  i \le \zeta-1\}$ is small enough.

The above procedure can be easily generalized to higher-degree polynomials. We use a linear function to approximate the $m-1$th derivative of $f(x)$ at each $I_i$, that is,
\begin{equation}
f^{(m-1)}(x) = a_ix + b_i + o(x-x_{i0}),
\end{equation}
for $x, x_{i0} \in I_i$, thereby
\begin{equation}
f(x) = p_{i}(x) + o\big((x-x_{i0})^{m}\big),
\end{equation}
where $p_{i}(x)$ is a polynomial of degree $m$ whose coefficients are derived from the repeated integral of $a_ix + b_i$ of equation 3.13. Polynomials $p_{i}(x)$'s for all $i$ comprise a spline $s(x) \in \mathfrak{S}_1^m(\Delta)$. The approximation error of $s(x)$ to $f(x)$ can be analysed analogously to equation 3.11.
\end{proof}

\begin{rmk}
The main purpose of constructing a spline from a piecewise linear approximation to the derivatives of $f(x)$ is to make the knots of smooth splines evident, providing convenience for neural-network implementation.
\end{rmk}

\subsection{Zero-Part Error of Units}
To activation function $\sigma(x)$ of a generalized sigmoidal unit $\mathcal{U}$, we call the interval $(-\infty, x_0)$ with $x_0 < 0$ the zero part of $\mathcal{U}$, if the contribution of $\mathcal{U}$ on $(-\infty, x_0)$ to function approximation can be ignored due to  $\lim_{x\to-\infty}\sigma(x)=0$. The case of activation function $\phi(x) = \sigma(wx + b)$ can be similarly defined. The zero part of units should be exclusively designed to control approximation error.

Note that if a function $f(x) \in C^m([0,1])$, its $m-1$th derivative $f^{(m-1)}(x)$ is smooth, and hence $f^{(m-1)}(x)$ can be approximated by a piecewise linear function that could lead to a smooth spline of order $m$ approximating $f(x)$. The multivariate case is similar. It's the reason that the functions to be discussed in the remaining part of this paper are often confined to the condition as $f(x) \in C^m([0,1])$.

\begin{figure}[t!]     
\captionsetup{justification=centering}
\subfloat[Parameter-adjusting procedure.]{\includegraphics[width=2.79in, trim = {1.3cm 0.5cm 1.0cm 0.5cm}, clip]{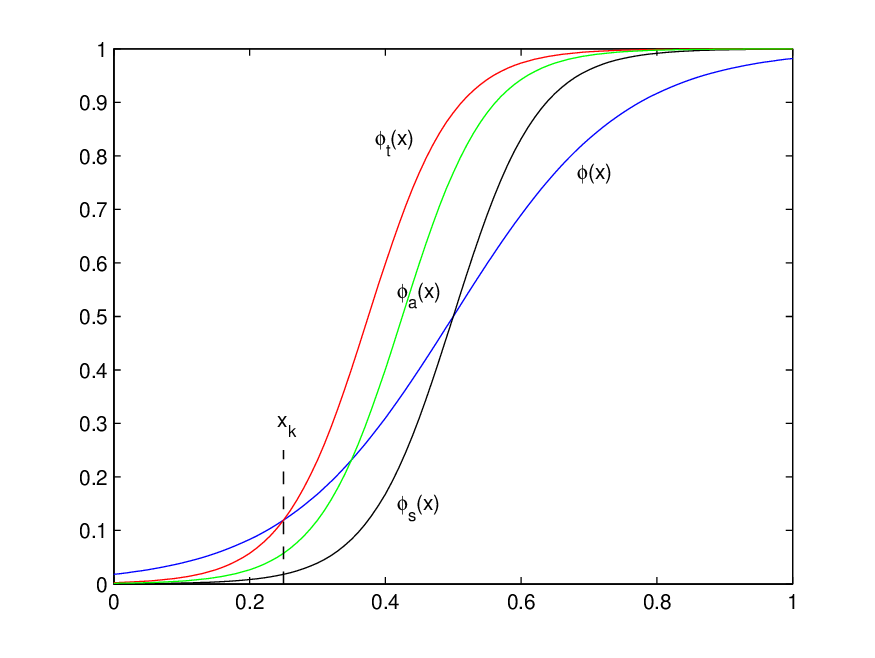}}
\subfloat[Effect of different scaling factors.]{\includegraphics[width=2.81in, trim = {1.3cm 0.5cm 1.0cm 0.5cm}, clip]{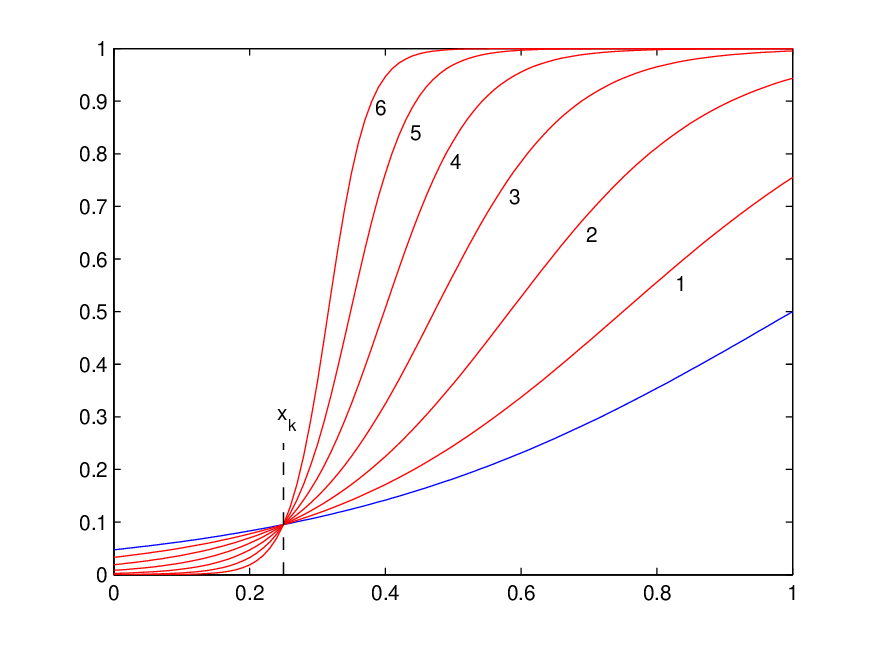}}
\caption{Reduction of zero-part errors.}
\label{Fig.1}
\end{figure}

\begin{lem}
Let $\phi(x) = \sigma(wx + b) \in C^{m}([0,1])$ be the activation function of a generalized sigmoid unit with $w > 0$ and $0<-b/w<1$, satisfying
\begin{equation}
\phi(x) \le \varepsilon
\end{equation}
for $x \le x_k$ and $\varepsilon \le \sigma(0)$, where $x_k \in (0, -b/w]$ and the equality holds when $x = x_k$. Then function $\phi(x)$ can be changed into
\begin{equation}
\phi_a(x) = \sigma(w'x + b') = \sigma(\rho(wx + b) + \gamma)
\end{equation}
by scaling the weight (via $\rho>1$) and adjusting the bias (via $\gamma$), such that the following conclusions hold:
\begin{itemize}
\item[\rm{\Romannum{1}}.] Suppose that $\phi_a(x)$ is approximated by a spline $\mathcal{S}(x) \in \mathfrak{S}_1^m(\Delta)$ to a desired accuracy, with $x_k$ above as one of the knots of $\Delta$. Let $s_i(x)$ for $i = 1, 2, \dots, \zeta$ be the polynomial on the $i$th subinterval derived from $\Delta$; the recurrence formula
\begin{equation}
s_{\nu}(x) = s_{\nu-1}(x) + c_{\nu-1}(x-x_{\nu-1})_+^m
\end{equation}
for $\nu = 2, 3, \dots, \zeta$ holds according to equation 3.4. Then
\begin{equation}
\lim_{\rho(\gamma) \to +\infty}s_{\mu}(x) = 0
\end{equation}
for $\mu = 1, 2, \dots, k$, where $\rho(\gamma)$ means that the change of $\rho$ is followed by the bias alteration $\gamma$ as in equation 3.16. Especially, when $\nu = k+ 1$ in equation 3.17,
\begin{equation}
\lim_{\rho(\gamma) \to +\infty}s_{k+1}(x) = 0 + c_{k}(x-x_{k})_+^m.
\end{equation}

\item[\rm{\Romannum{2}}.]
In equation 3.19,
\begin{equation}
\lim_{\rho(\gamma) \to +\infty}c_{k} = +\infty
\end{equation}
and $c_k$ monotonically increases with $\rho(\gamma)$.

\end{itemize}

\end{lem}
\begin{proof}
Figure \ref{Fig.1}a gives an example intuitively demonstrating the main procedures of the proof. We first modify $\phi(x)$ to be
\begin{equation}
\phi_t(x) = \sigma(w''x + b'') = \sigma(\rho(wx + b) + \gamma')
\end{equation}
where $\rho > 1$, satisfying
\begin{equation}
\phi_t(x_k) = \phi(x_k) = \varepsilon
\end{equation}
and
\begin{equation}
\phi_t(x) < \phi(x) < \varepsilon
\end{equation}
for $x < x_k$, for which two steps are required. First, scale the variable as $\phi_s(x) =  \sigma(\rho(wx + b))$. As shown in Figure \ref{Fig.1}a, the black cure is the scaled function $\phi_s(x)$ and the blue curve is the original $\phi(x)$. Because $\sigma(x)$ monotonically increases on $(-\infty, 0]$ and $\rho(wx_k + b) < wx_k + b < 0$, $\phi_s(x) < \phi(x)$ for $x < x_k$. The second step is to translate $\phi_s(x)$ to $\phi_t(x)$ by adjusting the bias of $\phi_s(x)$, subject to $\phi_t(x_k) = \phi(x_k)$ of equation 3.22; by
\begin{equation}
w''x_k + b'' = \rho(wx_k + b) + \gamma' = wx_k + b,
\end{equation}
we obtain the bias modification value $\gamma' = (wx_k + b)(1-\rho)$ and thus $w'' = \rho w$ and $b'' = (1-\rho)wx_k + b$ in equation 3.21.

See the red cure of Figure \ref{Fig.1}a, illustrating the difference between $\phi_t(x)$ and $\phi(x)$. Since
\begin{equation}
(w''x + b'') - (wx + b) = (\rho-1)w(x-x_k) < 0
\end{equation}
and $wx + b < 0$ when $x < x_k$, we have
\begin{equation}
\phi_t(x) = \sigma(w''x + b'') < \phi(x) = \sigma(wx + b)
\end{equation}
for $x < x_k$, which proves inequality 3.23.

By inequality 3.23, it is obvious that:
\begin{equation}
\lim_{\rho(\gamma') \to +\infty}s_{k}(x) = 0\ \text{for}\ x \in (x_{k-1}, x_{k}),
\end{equation}
without including the right endpoint $x_k$ of $(x_{k-1}, x_{k})$, where $\rho(\gamma')$ is similar to $\rho(\gamma)$ of equation 3.18,
\begin{equation}
\lim_{\rho(\gamma') \to +\infty}s_{\tau}(x) = 0\ \text{for}\ x \in (x_{\tau-1}, x_{\tau}],
\end{equation}
where $\tau = 2, 3, \dots, k-1$, and
\begin{equation}
\lim_{\rho(\gamma') \to +\infty}s_{1}(x) = 0\ \text{for}\ x \in [0, x_1].
\end{equation}
A difference between equations 3.27 and 3.28 is that the interval of the former doesn't contain the right endpoint, because $\phi_t(x_k) = \varepsilon$ (equation 3.22) is fixed by the previous parameter setting; this problem is to be dealt with later. By the recurrence formula of equation 3.17, we have
\begin{equation}
s_{\tau}(x) =  s_1(x) + \sum_{j = 1}^{\tau-1}c_{j}(x-x_j)^m_+
\end{equation}
with $\tau$ from equation 3.28. Equations from 3.28 to 3.30 yield the following equations
\begin{equation}
s_2(x) = s_1(x) + c_1(x-x_1)^m_+ \approx 0
\end{equation}
for $x \in (x_1, x_2]$, where ``$\approx$'' means ``equal as $\rho(\gamma') \to +\infty$'' (similarly for the remaining ones in this proof),
\begin{equation}
s_3(x) = s_1(x) + c_1(x-x_1)^m_+ + c_2(x-x_2)^m_+ \approx 0
\end{equation}
for $x \in (x_2, x_3]$, $\dots$,
\begin{equation}
s_{k-1}(x) = s_1(x) + \sum_{j = 1}^{k-2}c_{j}(x-x_j)^m_+ \approx 0
\end{equation}
for $x \in (x_{k-2}, x_{k-1}]$. In Equation 3.31, because there are infinitely many $x \in (x_1, x_2]$ fulfilling this formula, we can regard $s_1(x)$ and $-c_1(x-x_1)^m_+$ as almost the same polynomial; so $s_1(x) \approx -c_1(x-x_1)^m$ also holds on $[0, x_1]$; we know $s_1(x) \approx 0$ on $[0, x_1]$ (equation 3.29) and $c_1 \approx 0$ follows. Similarly, in equation 3.32, $s_1(x) + c_1(x-x_1)^m_+ \approx 0$ implies $c_2 \approx 0$. This process can be repeatedly done until
\begin{equation}
c_{\kappa} \approx 0
\end{equation}
for all $\kappa = 1, 2, \dots, k-2$.

To the case of $c_{k-1}$, we should solve the aforementioned problem that equation 3.27 is not applicable to the right endpoint $x_k$ of $(x_{k-1}, x_k)$. First add a knot $x_k'$ with
\begin{equation}
x_{k-1} < x_k' < x_k;
\end{equation}
correspondingly, the new set of the knots is denoted by $\Delta'$. Then construct a spline $\mathcal{S}'(x) \in \mathfrak{S}^m_1(\Delta')$ approximating $\phi_t(x)$. Let $s_{k}'(x)$ be the polynomial of $\mathcal{S}'(x)$ on $(x_{k}', x_k]$ and the notations of other polynomials on the rest of the subintervals remain the same as those of $\mathcal{S}(x)$. When $\rho(\gamma')$ is sufficiently large, we have
\begin{equation}
\lim_{\rho(\gamma') \to +\infty}s_{k}(x) = 0\ \text{for}\ x \in (x_{k-1}, x_{k}']
\end{equation}
as the counterpart of equation 3.27, with the subinterval including its right end point $x_k'$. Then
\begin{equation}
\lim_{\rho(\gamma') \to +\infty}c_{k-1} = 0
\end{equation}
holds by the method of equation 3.34, where $c_{k-1}$ is from
\begin{equation}
s_k(x) = s_{k-1}(x) + c_{k-1}(x-x_{k-1})^m_+
\end{equation}
for $x \in (x_{k-1}, x_k']$.

To a small enough positive real number
\begin{equation}
\phi_t(x_k') = \epsilon < \varepsilon,
\end{equation}
equation 3.36 can be expressed as
\begin{equation}
|s_{k}(x)| \le \epsilon.
\end{equation}
Note that to a fixed $\epsilon$, equation 3.40 holds for arbitrary large $x_k' \in (x_{k-1}, x_k)$ as $\rho(\gamma')$ increases; that is, $\sup{x_k'} = x_k$ but $x_k' \ne x_k$, or equivalently,
\begin{equation}
\lim_{\rho(\gamma') \to +\infty}{\{x_k - x_k'\}} = 0.
\end{equation}

We also have
\begin{equation}
s_k'(x) = s_k(x) + c_{k-1}'(x-x_k')^m_+
\end{equation}
for $ x \in (x_{k}', x_k]$. In equation 3.42, because $s_k(x)$ approaches zero in terms of a polynomial of order $m$ (equations 3.36 and 3.38), we can regarded it as constant zero; so equation 3.42 can be written as
\begin{equation}
s_k'(x) \approx 0 + c_{k-1}'(x-x_k')^m_+.
\end{equation}

Given a small positive real number $\epsilon < \varepsilon$, suppose that $\rho(\gamma')$ has been big enough to make $\phi_t(x) \le \epsilon$ on $[0, x_k']$. Then due to the constraint of both $\phi_t(x_k') = \epsilon$ (equation 3.39) and $\phi_t(x_k) = \varepsilon$ (equation 3.22), $s_k'(x)$ on $(x_k', x_{k}]$ should increase from $\epsilon$ to $\varepsilon$, that is,
\begin{equation}
s_k'(x_k) - s_k'(x_k') = \varepsilon - \epsilon.
\end{equation}
When the length of interval $(x_k', x_{k}]$ becomes smaller as $\rho(\gamma') \to +\infty$, under the fixed polynomial degree $m$, $c_{k-1}'$ of equation 3.43 should become larger to achieve the goal of equation 3.44, resulting in
\begin{equation}
\lim_{\rho(\gamma') \to +\infty}c'_{k-1} = +\infty
\end{equation}
and the monotonically increasing of $c_{k-1}'$ with $\rho$.

In Figure \ref{Fig.1}b, the red cures with labels from 1 to 6 are six instantiations of $\phi_t(x)$ whose parameter $\rho$ monotonically increases, from which we can see the behavior of $\phi_t(x)$ with respect to $\rho(\gamma')$ and obtain an intuitive understanding of equation 3.45.

To the last step, we translate $\phi_t(x)$ to $\phi_a(x)$ by adjusting the bias in terms of
\begin{equation}
\phi_a(x) = \sigma(w'x + b') = \sigma(w''x + b'' + \gamma''),
\end{equation}
where $w''$ and $b''$ is from equation 3.21, to satisfy
\begin{equation}
\phi_a(x_k) = \phi_t(x'_k) = \epsilon,
\end{equation}
which implies
\begin{equation}
w'x_k + b' = w''x_k' + b''
\end{equation}
or
\begin{equation}
\rho(wx_k + b) + \gamma' + \gamma'' = \rho(wx_k' + b) + \gamma',
\end{equation}
leading to
\begin{equation}
\gamma'' = \rho w(x_k' - x_k) < 0;
\end{equation}
thus, in equation 3.16,
\begin{equation}
\gamma = \gamma'' + \gamma'.
\end{equation}
We next delete the knot $x_k'$ and reconstruct the polynomials on the subintervals of $(x_{k}, 1]$; adjusting the positions of the knots bigger than $x_k$ or adding new knots on $(x_{k}, 1]$ if necessary. Then the new spline $\mathcal{S}(x) \in \mathfrak{S}^m_1(\Delta)$ can be regarded as zero on $[0, x_k]$ and $s_{k+1}(x) \approx c_{k-1}'(x-x_k')^m_+$; let $c_k = c_{k-1}'$ and this proves equation 3.20 by equation 3.45. In Figure \ref{Fig.1}a, the green curve is $\phi_a(x)$ obtained by slightly translating $\phi_t(x)$. This completes the proof.

\end{proof}

\begin{rmk-3}
Although the selection of $\varepsilon$ in equation 3.15 doesn't affect the conclusion of this lemma, under a fixed scaling parameter $\rho$, it may influence the degree $m$ of polynomials approximating an activation function and is possibly useful in parameters adjusting.
\end{rmk-3}

\begin{rmk-3}
In equation 3.16, when $w < 0$, conclusions analogous to this lemma can be similarly obtained, with the zero-error part on the right of knot $x_k$.
\end{rmk-3}

\begin{thm}
Let $\mathscr{S}(x) \in \mathfrak{S}^m_1(\Delta)$ be a spline with $p_i(x)$ for $i = 1, 2, \dots, \zeta$ being the polynomial on the $i$th subinterval of $\Delta$. Suppose that $p_j(x) = 0$ for $x \in [0, x_{k}]$ and $j = 1, 2, \dots, k$; so
\begin{equation}
p_{k+1} = 0 + d_{k}(x - x_k)^m_+
\end{equation}
for $x \in (x_k, x_{k+1}]$. Then a two-layer neural network with one generalized-sigmoidal unit $\mathcal{U}$ can approximate $\mathscr{S}(x)$ on $[0, x_{k+1}]$ with arbitrary precision, that is,
\begin{equation}
 \big(\int_{0}^{x_{k+1}}(\lambda\phi(x) -  \mathscr{S}(x))^2dx\big)^{1/2} < \epsilon,
\end{equation}
where $\phi(x) = \sigma(wx + b)$ is the activation function of $\mathcal{U}$, $\lambda$ is its output weight and $\epsilon$ is an arbitrarily small positive real number.
\end{thm}
\begin{proof}
The proof is by lemma 3. Equation 3.52 of this theorem corresponds to equation 3.19 of lemma 3; the output weight $\lambda$ can be obtained by
\begin{equation}
\lambda c_k = d_k,
\end{equation}
through which $p_{k+1}(x)$ on $(x_k, x_{k+1}]$ is realized by $\lambda\phi(x)$. To the part $[0, x_k]$, the approximation of $\lambda\phi(x)$ to zero is via equations 3.18 and 3.20. A key point is that if $|\lambda|$ increases as $\rho(\gamma) \to +\infty$, a larger approximation error would occur on $[0, x_k]$; that's why we developed the second conclusion of lemma 3. According to lemma 3, $c_k > 0$ monotonically increases with $\rho(\gamma)$; then in equation 3.53, $|\lambda|$ decreases as $c_k$ becomes larger; in combination with the smaller $\phi(x)$, $|\lambda\phi(x)|$ is reduced as $\rho(\gamma)$ increases, resulting in a smaller approximation error on $[0, x_k]$.
\end{proof}

\subsection{Univariate Local (Global) Unit}
\begin{dfn}[Univariate local (global) unit]
Given a function $f(x) \in C^{(m)}([0, 1])$, suppose that it is approximated by a two-layer neural network $\mathfrak{N}$ with generalized sigmoidal units, in terms of
\begin{equation}
\lVert f(x) - \sum_i\lambda_i\phi_i(x) \rVert_2 < \varepsilon,
\end{equation}
where $\phi_i(x) = \sigma(w_ix + b_i)$ is the activation function of the $i$th unit $u_i$ of the hidden layer of $\mathfrak{N}$. To some $u_i$ with $w_i > 0$, suppose that there exists a point $x_0 \in (0, 1)$ such that if $\phi_i(x)$ is truncated at $x_0$, that is, set $\phi_i(x) = 0$ on $[0, x_0]$, inequality 3.55 still holds, for which we call $x_0$ a potential zero-error point with respect to $\varepsilon$. Write
\begin{equation}
z_i := \sup x_0.
\end{equation}
If $z_i \in (0, 1)$, $z_i(\varepsilon)$ is called a \textsl{zero-error point} of $u_i$ (or $\phi_i(x)$) with respect to $\varepsilon$; when $\varepsilon$ of inequality 3.55 becomes smaller (readjusting some parameters of $\mathfrak{N}$ if necessary as in lemma 3), if
\begin{equation}
\lim_{\varepsilon \to 0}z_i(\varepsilon) \ne 0,
\end{equation}
$u_i$ is said to be a local unit; otherwise, it is a global one. To the case of $w_i < 0$, $\phi(x)$ is set to be zero on $[x_0, 1]$, with equation 3.56 modified to $z_i := \inf{x_0}$.
\end{dfn}

\begin{rmk-4}
We use the existence or nonexistence of local units to determine the approximation mode. If there's at least one local unit, the solution is global and can be explained by splines; otherwise, it is a local solution.
\end{rmk-4}

\begin{rmk-4}
The zero-error point of unit $u_i$ is defined independently with other units; does this lead to accumulated error that cannot be controlled when all the units are simultaneously considered? The answer is not. The truncation operation on $u_i$ can be expressed as
\begin{equation}
\lVert f(x) - \big(\sum_{\nu}\lambda_{\nu}\phi_{\nu}(x) - \lambda_{i}\hat{\phi}_{i}(x)\big)\rVert_2 < \varepsilon,
\end{equation}
where $\hat{\phi}_{i}(x) = 0$ on $[0, z_i(\varepsilon)]$ and $\hat{\phi}_{i}(x) = \phi_{i}(x)$ on $(z_i(\varepsilon), 1]$. When there is more than one local unit, the counterpart of inequality 3.58 is
\begin{equation}
\lVert f(x) - \sum_{\nu}\lambda_{\nu}\phi_{\nu}(x) + \sum_{\mu = 1}^{\gamma}\lambda_{i_{\mu}}\hat{\phi}_{i_{\mu}}(x)\rVert_2 < \epsilon,
\end{equation}
where $\gamma$ is the number of local units and the approximation error $\epsilon$ needs to be evaluated. We show that inequality 3.58 for all $i_{\mu}$ (inequality 3.59) can imply an arbitrary small $\epsilon$ of inequality 3.59. Adding inequalities 3.55 and 3.58 yields
\begin{equation}
\lVert d(x) \rVert_2 + \lVert d(x) + \lambda_{i}\hat{\phi}_{i}(x)\rVert_2 < 2\varepsilon,
\end{equation}
where $d(x) = f(x) - \sum_{\nu}\lambda_{\nu}\phi_{\nu}(x)$. By triangle inequality $\lVert x - y \rVert_2 \le \lVert x \rVert_2 + \lVert y \rVert_2$, inequality 3.60 implies $\lVert\lambda_{i}\hat{\phi}_{i}(x)\rVert_2 < 2\varepsilon$, such that
\begin{equation}
\big\lVert \sum_{\mu = 1}^{\gamma}\lambda_{i_{\mu}}\hat{\phi}_{i_{\mu}}(x) \big\rVert_2 \le \sum_{\mu = 1}^{\gamma}\big\lVert \lambda_{i_{\mu}}\hat{\phi}_{i_{\mu}}(x) \big\rVert_2 < 2\gamma\varepsilon.
\end{equation}
So
\begin{equation}
\begin{aligned}
&\ \ \ \ \ \ \big\lVert f(x) - \sum_{\nu}\lambda_{\nu}\phi_{\nu}(x) + \sum_{\mu = 1}^{\gamma}\lambda_{i_{\mu}}\hat{\phi}_{i_{\mu}}(x) \big\rVert_2 \le \\
\big\lVert f(x) - &\sum_{\nu}\lambda_{\nu}\phi_{\nu}(x) \big\rVert_2 + \big\lVert\sum_{\mu = 1}^{\gamma}\lambda_{i_{\mu}}\hat{\phi}_{i_{\mu}}(x)\big\rVert_2 < (2\gamma+1)\varepsilon = \epsilon.
\end{aligned}
\end{equation}
Thus, if $\varepsilon$ is sufficiently mall, the accumulated error of all the local units can also be ignored.
\end{rmk-4}

\begin{prp}
The unit $\mathcal{U}$ of theorem 5 is a local unit.
\end{prp}
\begin{proof}
By the proof of theorem 5, on $[0, x_k]$,
\begin{equation}
\lim_{\rho(\gamma) \to +\infty}\{\int_{0}^{x_k}(\lambda\phi(x))^2dx\}^{1/2} = 0.
\end{equation}
which means that truncating $\phi(x)$ on $[0, x_k]$ would not influence the establishment of inequality 3.55 when $\rho(\gamma)$ is sufficiently large. Due to the proof of lemma 3, $\lim_{\varepsilon \to 0}z_i(\varepsilon) = x_k \ne 0$ always holds, no matter how small $\varepsilon$. This completes the proof.
\end{proof}


\subsection{One-Sided Solutions}
To approximate univariate functions, the term ``one sided'' means that all $w_i$'s of activation functions $\phi_i(x) = \sigma(w_ix + b_i)$ are greater than (or less than) zero; otherwise, it is the ``two-sided'' case.

\begin{figure}[!t]
\captionsetup{justification=centering}
\centering
\includegraphics[width=3.6in, trim = {1.0cm 0.5cm 1.0cm 0.5cm}, clip]{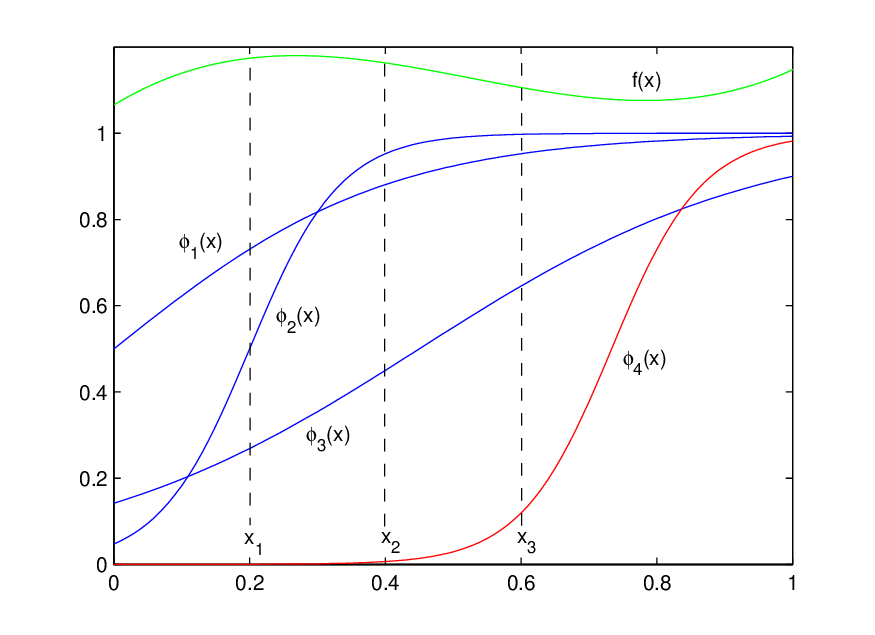}
\caption{Principle of global approximation.}
\label{Fig.2}
\end{figure}

\begin{lem}
Let $s(x) \in \mathfrak{S}^m_1( \Delta)$ be a spline with $\mathcal{P}_k(x)$ the polynomial on the $k$th subinterval derived from $\Delta$. Denote by $\phi_i(x) = \sigma(w_ix + b_i)$ for $i = 1, 2, \dots, \zeta$ and $w_i > 0$ the activation function of the $i$th unit $u_i$ of a two-layer neural network $\mathfrak{N}$, and by $g(x) = \sum_{i = 1}^{\zeta}\lambda_i\phi_i(x)$ the output of $\mathfrak{N}$. Suppose that $\phi_i(x) \in C^{m}([0, 1])$ and that $\mathcal{P}_{k-1}(x)$ on $(x_{k-2}, x_{k-1}]$ of $s(x)$ has been approximated by $\mathfrak{N}$ with a desired precision $\varepsilon$---that is,
\begin{equation}
\{\int_{x_{k-2}}^{x_{k-1}}(\mathcal{P}_{k-1}(x) - g(x))^2dx\}^{1/2} < \varepsilon.
\end{equation}
If the approximation error for the next subinterval $(x_{k-1}, x_k]$ is bigger than $\varepsilon$ (i.e.,
\begin{equation}
\{\int_{x_{k-1}}^{x_{k}}(\mathcal{P}_k(x) - g(x))^2dx\}^{1/2} \ge \varepsilon
\end{equation}
), a new unit $u_{\zeta+1}$ can be added to make
\begin{equation}
\{\int_{x_{k-1}}^{x_{k}}(\mathcal{P}_k(x) - g(x))^2dx\}^{1/2} < \varepsilon
\end{equation}
without influencing inequality 3.64, provided that $|x_k-x_{k-1}|$ and $|x_{k-1}-x_{k-2}|$ are sufficiently small.
\end{lem}
\begin{proof}
We first see an example of Figure \ref{Fig.2}. In this case, $f(x) = s(x)$ is a spline. Suppose that $\phi_4(x)$ of $u_4$ is initially absent, and then $g(x) = \sum_{i=1}^{3}\lambda_i\phi_i(x)$. Because $\phi_i(x) \in C^{m}([0, 1])$, it can be approximated by a spline $s_i(x) \in \mathfrak{S}^m_1(\Delta)$ as precisely as possible, with $p_{k, i}(x)$ the polynomial on the $k$th subinterval satisfying
\begin{equation}
p_{k, i}(x) = p_{k-1, i}(x) + \beta_{k-1,i}(x - x_{k-1})_+^{m}.
\end{equation}

The polynomial $\mathcal{P}_{k-1}(x)$ of inequality 3.64 is the linear combination of $p_{k-1, i}(x)$'s for all $i$, that is,
\begin{equation}
\mathcal{P}_{k-1}(x) = \sum_{i}\lambda_ip_{k-1, i}(x).
\end{equation}
Because $\sum_{i=1}^{3}\lambda_ip_{k, i}(x) \ne \mathcal{P}_k(x)$, inequality 3.66 doesn't hold, for which we add a new unit to change this situation. In Figure \ref{Fig.2}, $u_4$ is the added unit whose activation function is $\phi_4(x)$. There are two main steps to adjust the parameters of $\phi_4(x)$ to satisfy $\sum_{i=1}^{4}\lambda_ip_{k, i}(x) = \mathcal{P}_k(x)$. The first is by lemma 3 to reduce the influence of $\phi_4(x)$ on $[0, x_{k-1}]$ to a desired accuracy. The second is to set $\lambda_4$ such that $\sum_{i=1}^3\lambda_{i}p_{k, i}(x) + \lambda_4p_{k, 4}(x) = \mathcal{P}_k(x)$. Let $\zeta = 3$. The original polynomial $\mathcal{P}_k(x)$ before adding $u_{\zeta + 1}$ can be expressed as
\begin{equation}
\begin{aligned}
\sum_i^{\zeta} \lambda_{i} & p_{k,i}(x) = \sum_{i=1}^{\zeta} \lambda_i\big(p_{k-1,i}(x) + \beta_{k-1,i}(x-x_{k-1})_+^{m}\big) \\
&= \mathcal{P}_{k-1}(x) + \gamma(x - x_{k-1})_+^{m} \ne \mathcal{P}_{k}(x),
\end{aligned}
\end{equation}
where
\begin{equation}
\gamma = \sum_{i=1}^{\zeta}\lambda_i\beta_{k-1,i}.
\end{equation}

We have two additional conditions. One is
\begin{equation}
\mathcal{P}_{k}(x) = \mathcal{P}_{k-1}(x) + c(x - x_{k-1})_+^{m}.
\end{equation}
The other is that equation 3.67 is applicable to the the new unit $u_{4}$ as well, that is,
\begin{equation}
p_{k,\zeta+1}(x) = p_{k-1, \zeta+1}(x) + \beta_{k-1,\zeta+1}(x - x_{k-1})_+^{m}.
\end{equation}
Then the updated version of equation 3.69 is
\begin{equation}
\mathcal{P}'_{k}(x) = \sum_{\mu = 1}^{\zeta + 1} \alpha_{\mu}p_{k{\mu}}(x) = \mathcal{P}'_{k-1}(x) + \gamma'(x - x_{k-1})_+^{m},
\end{equation}
where
\begin{equation}
\mathcal{P}'_{k-1}(x) = \mathcal{P}_{k-1}(x) + \lambda_{\zeta+1}p_{k-1,\zeta+1}(x)
\end{equation}
and
\begin{equation}
\gamma' = \gamma + \lambda_{\zeta+1}\beta_{k-1,\zeta+1}.
\end{equation}
In equation 3.74, setting $p_{k-1,\zeta+1}(x)\approx 0$  by lemma 3, where ``$\approx$'' means ``approximately equal with arbitrary precision'', we have
\begin{equation}
\mathcal{P}'_{k-1}(x) \approx \mathcal{P}_{k-1}(x).
\end{equation}
Since $\lambda_{\zeta+1}$ of equation 3.75 can be arbitrarily set, we can find a solution of $\lambda_{\zeta+1}$ to satisfy $\gamma' = c$ of equation 3.71, which is
\begin{equation}
\lambda_{\zeta+1} = (c - \gamma)/(\beta_{k-1,\zeta + 1}).
\end{equation}
Equations from 3.73 to 3.77 realize $\mathcal{P}'_{k}(x) \approx \mathcal{P}_{k}(x)$. This completes the proof.
\end{proof}

\begin{thm}[Univariate-spline implementation]
Notations being as in lemma 4, any spline $s(x) \in \mathfrak{S}^m_1(\Delta)$ can be approximated by a two-layer network $\mathfrak{N}$ with arbitrary precision, provided that the lengths of the subintervals $I_i$'s for $i = 1, 2, \dots, \zeta$ of $\Delta$ are sufficiently small. If $s(x)$ has $\zeta$ polynomial pieces, the number of the units of $\mathfrak{N}$ satisfies $\Theta \ge \zeta + m$.
\end{thm}
\begin{proof}
According to lemma 2, the first polynomial $\mathcal{P}_1(x)$ on $[0, x_1]$ of $s(x)$ can be realized by a set of $\phi_j(x)$'s for $j = -\mu, -\mu+1, \dots, 0$ with $\mu+1 \ge m+1$. The remaining polynomials are implemented by lemma 4, with each requiring one unit. So the total number of the units is at least
\begin{equation}
(m + 1) + (\zeta - 1) = \zeta + m.
\end{equation}
The approximation error is controlled by theorem 5 and lemma 4.
\end{proof}

\begin{thm}[Univariate universal approximation]
Any function $f(x) \in C^{m}([0, 1])$ can be approximated by a two-layer network $\mathfrak{N}$ with generalized sigmoidal units to arbitrary accuracy, in terms of realizing a spline $\hat{f}(x) \in \mathfrak{S}^m_1(\Delta)$ that could approximate $f(x)$ as precisely as possible. If $\hat{f}(x)$ has $\zeta$ polynomial pieces, the hidden layer of $\mathfrak{N}$ needs at least $\zeta + m$ units, that is,
\begin{equation}
\Theta \ge \zeta + m,
\end{equation}
where $\Theta$ is the number of the units.
\end{thm}
\begin{proof}
Theorem 4 gave the method of constructing spline $\hat{f}(x)$ approximating $f(x)$ and theorem 6 realized it via two-layer neural networks.
\end{proof}

\begin{cl}[Generalized tanh-unit case]
In theorem 7, if the hidden-layer units of network $\mathfrak{N}$ are replaced by the generalized tanh units, the conclusion still holds except that equation 3.79 is modified to
\begin{equation}
\Theta \ge \zeta + m + 1.
\end{equation}
\end{cl}
\begin{proof}
Denote the new neural network by $\mathfrak{N}'$. Under theorem 7, the polynomial $s_i(x)$ on the $i$th subinterval for $\hat{f}(x)$ can be expressed as
\begin{equation}
s_i(x)=\sum_{j=-\mu}^{0}\lambda_j\phi_j(x)+\sum_{j=1}^{i-1}\lambda_j\phi_j(x),
\end{equation}
where $\phi_j(x)$'s are the activation function of the hidden-layer units of $\mathfrak{N}$, among which $\mu+1$ ones for $\mu+1\ge m+1$ and $-\mu\le j\le0$ are of global unit and the remaining ones are of local unit. Letting $\tau_j(x)$ be the activation function of a generalized tanh unit corresponding to $\phi_j(x)$, then
\begin{equation}
\psi_j(x) = \tau_j(x)-\mathcal{C}
\end{equation}
is of a generalized sigmoidal unit, where $\mathcal{C}$ is a constant from definition 4. So $s_i(x)$ of equation 3.81 can also be represented as
\begin{equation}
\begin{aligned}
s_i(x)=\sum_{j=-\mu}^{0}\alpha_j\tau_j(x&)+\sum_{j=1}^{i-1}\alpha_j\psi_j(x) = \sum_{j=-\mu}^{0}\alpha_j\tau_j(x)+\sum_{j=1}^{i-1}\alpha_j(\tau_j(x)-\mathcal{C})\\
&=\sum_{j=-\mu}^{0}\alpha_j\tau_j(x) + \sum_{j=1}^{i-1}\alpha_j\tau_j(x) -\sum_{j=1}^{i-1}\alpha_j\mathcal{C}.
\end{aligned}
\end{equation}

Especially, when $i=\zeta$, we have
\begin{equation}
s_{\zeta}(x)= \sum_{j=-\mu}^{0}\alpha_j\tau_j(x) + \sum_{j=1}^{\zeta-1}\alpha_j\tau_j(x) + \mathfrak{C},
\end{equation}
where
\begin{equation}
\mathfrak{C}=-\sum_{j=1}^{\zeta-1}\alpha_j\mathcal{C}
\end{equation}
is a constant when $\alpha_j$'s for $1\le j\le \zeta-1$ are fixed. If $i < \zeta$, the new added $-\alpha_{i+1}\mathcal{C}$ would not influence $s_{i}(x)$ because $\tau_{i+1}(x)-\mathcal{C}\approx 0$; so $\mathfrak{C}$ stop change when $i=\zeta$, and its update when $i < \zeta$ cannot affect the previously established result. The polynomial $s_1(x)$ of $\hat{f}(x)$ on $[0, x_1]$ can be implemented by a set of any type of smooth activation (theorem 1), for which we can replace $\phi_i(x)$ of equation 3.81 by $\tau_j(x)$ of equation 3.83 when $-\mu\le j\le0$.

Therefore, if the constant $\mathfrak{C}$ of equation 3.85 can be produced by generalized tanh units, the universal approximation of $\mathfrak{N}'$ could be ensured. We provide two solutions. One is that add a unit whose activation function over $[0, 1]$ is nearly a constant $\beta \ne 0$ and set its output weight to be $\mathfrak{C}/\beta$. The other is by adding some global units whose linear combination is $\mathfrak{C}$. This completes the proof.
\end{proof}

\begin{rmk-5}
The general principle underlying this corollary applies to the case of higher-dimensional input.
\end{rmk-5}

\begin{rmk-5}
This solution mode is only for global approximation related to splines and is meaningless to local approximation of section 2.
\end{rmk-5}

\begin{rmk-5}
The constant $\mathfrak{C}$ predicted by this corollary is observed in experiments (see later section 7.4).
\end{rmk-5}

\subsection{A Two-Sided Solution}

\begin{figure}[!t]
\captionsetup{justification=centering}
\centering
\includegraphics[width=3.6in, trim = {1.0cm 0.5cm 1.0cm 0.5cm}, clip]{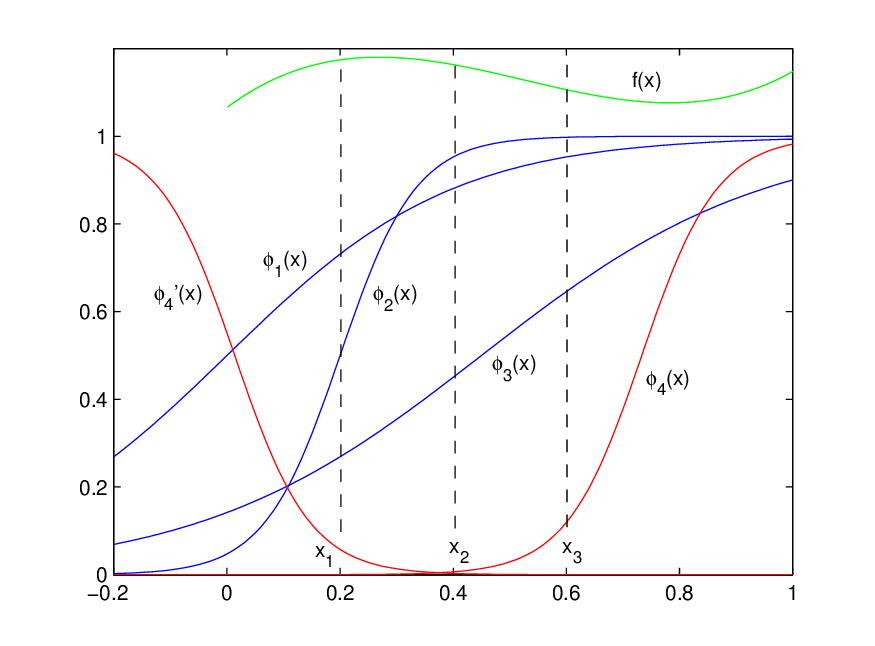}
\caption{A two-sided solution.}
\label{Fig.3}
\end{figure}

\begin{dfn}[Negative unit]
Let $\phi(x) = \sigma(wx + b)$ be the activation function of a generalized sigmoidal unit $\mathcal{U}$. If $w > 0$, $\mathcal{U}$ is called a positive unit and otherwise a negative unit.
\end{dfn}

\begin{lem}
Under theorem 6, to a local unit $u_k$ whose zero-error point is knot $x_k$, suppose that a negative unit $u'_k$ is added whose activation function is $\phi_k'(x) = \sigma(w_k'x + b_k')$ with $w_k' < 0$ and whose zero-error point is also $x_k$. Given any output weight $\lambda'_{k}$ of $u_k'$, theorem 6 still holds by readjusting the remaining parameters of network $\mathfrak{N}$.
\end{lem}
\begin{proof}
Figure \ref{Fig.3} gives an example of this lemma, in which $\phi'_4(x)$ is the activation function of new added unit $u'_{4}$ whose zero-error point is $x_2=0.4$. After $u'_{4}$ having been introduced, the output weights $\lambda_i$'s for $i = 1, 2, 3$ of $\phi_i(x)$'s should be altered to compensate the impact of $\phi'_4(x)$ on $[0, x_2]$; this is not difficult by lemmas 2 and 4. The parameter resetting could influence the polynomial on $(x_2, x_3]$, but the output weight $\lambda_4$ of $\phi_4(x)$ can be adjusted to compensate the disturbance (lemma 4), such that the original polynomial can still be generated. This example contains the general principle of this lemma.
\end{proof}

\begin{prp}
Any number of negative units whose zero-error points are knots could be added, without influencing the conclusion of theorem 6.
\end{prp}
\begin{proof}
This is the repeated application of lemma 5.
\end{proof}

\section{General Solution Framework}
The solutions constructed in section 3 may not include all those obtained by the back-propagation algorithm. So we develop a general framework by reducing the spline implementation to a system of linear equations, through which the diversity of solutions can be attributed to the nonsingular property of matrices. We first show two examples derived from section 3 and then give the general form.

\subsection{One-Sided Example}
This section describes theorem 6 through matrix operations. Let $g(x) = \sum_{i = 1}^{\Theta}\lambda_i\phi_i(x)$ be the output of a two-layer network $\mathfrak{N}$ with generalized sigmoidal units, where $\phi_i(x) = \sigma_i(w_ix + b_i) \in C^{(m)}(\mathbb{R})$ is the activation function of the $i$th unit of the hidden layer. The set $\Delta = \{x_{j}: j = 1, 2, \dots, \zeta-1\}$ includes all the knots. Denote by $u_{k}$'s for $k = 1, 2, \dots, \mu$ the global units, with $\mu \ge m+1$, and by $u_{j}$'s for $j = \mu+1, \mu+2, \dots, \mu + \zeta-1$ the local units.

By theorem 4, each activation function $\phi_i(x)$ has an approximation via a spline $s_i(x) \in \mathfrak{S}^m_1(\Delta)$ whose polynomial on subinterval $I_{\nu}$ is denoted by $p_{i\nu}(x)$ for $\nu = 1, 2, \dots, \zeta$. Then
\begin{equation}
p_{ij}(x) = p_{i, j-1}(x) + \beta_{i,j-1}(x - x_{j-1})_+^{m}
\end{equation}
for $j = 2, 3, \dots, \zeta$. Let $\mathcal{S}(x) \in \mathfrak{S}^m_1(\Delta)$ be the spline constructed from $f(x)$, which is to be implemented by network $\mathfrak{N}$, and let $\mathcal{P}_{\nu}(x)$ be the polynomial of $\mathcal{S}(x)$ on $I_{\nu}$. We also have
\begin{equation}
\mathcal{P}_{j}(x) = \mathcal{P}_{j-1}(x) + c_{j-1}(x - x_{j-1})_+^{m}.
\end{equation}
The goal is to realize $\mathcal{P}_{j}(x)$'s by the weighted sum of  $p_{ij}(x)$'s.

According to section 1, the way of producing $\mathcal{P}_1(x)$ on $I_1 = [0, x_1]$ is by
\begin{equation}
\mathcal{P}_1(x) = \sum_{k \in G \ \text{or} \  k\le\mu}\lambda_{k}p_{k1}(x),
\end{equation}
where $G$ is the set of the subscripts of the global units, leading to
\begin{equation}
A_1\boldsymbol{\lambda}_1 = \boldsymbol{b}_1
\end{equation}
as equation 2.9, where $A_1$ is a $(m+1) \times \mu$ matrix, $\boldsymbol{\alpha}_1 = [\lambda_{1}, \lambda_{2}, \dots, \lambda_{\mu}]^T$ and $\boldsymbol{b}_1$ is a $(m+1) \times 1$ vector whose entries are the coefficients of the terms of $\mathcal{P}_1(x)$.

Polynomial $\mathcal{P}_2(x)$ is obtained by
\begin{equation}
\mathcal{P}_{2}(x) = \mathcal{P}_1(x) + \big(\sum_{k \in G}\lambda_{k}\beta_{k1} + \lambda_{\mu+1}\beta_{\mu+1, 1}\big)(x - x_{1})_+^{m}.
\end{equation}
We make
\begin{equation}
\sum_{k \in G}\lambda_{k}\beta_{k1} + \lambda_{\mu+1}\beta_{\mu+1, 1} = c_1
\end{equation}
of equation 4.2 by adjusting the free parameter $\lambda_{\mu+1}$. The remaining $\mathcal{P}_{j}(x)$'s are similar. Through
\begin{equation}
\mathcal{P}_{j}(x) = \mathcal{P}_{j-1}(x) + \big(\sum_{k < \mu + j-1}\lambda_{k}\beta_{k,j-1} + \lambda_{\mu + j-1}\beta_{\mu + j-1, j-1}\big)(x - x_{j-1})_+^{m},
\end{equation}
we obtain
\begin{equation}
\sum_{k < \mu + j-1}\lambda_{k}\beta_{k,j-1} + \lambda_{\mu + j-1}\beta_{\mu + j-1, j-1} = c_{j-1}
\end{equation}
and set $\lambda_{\mu + j-1}$ to generate $c_{j-1}$.

Equations 4.4 and 4.8 can be combined to a integrated system of linear equations
\begin{equation}
\mathcal{A}\boldsymbol{\lambda} = \boldsymbol{b},
\end{equation}
where


\begin{equation}
\mathcal{A} = \begin{pmat}({ccc|})
&A_1  & & & &\boldsymbol{0} \cr\-
\beta_{11} & \beta_{21} & \cdots & \beta_{\mu,1} &\beta_{\mu+1,1} &0 &\cdots &0\cr
\beta_{12} & \beta_{22} & \cdots & \beta_{\mu,2} &\beta_{\mu+1, 2} &\beta_{\mu+2, 2} &\cdots &0\cr
\vdots & \vdots & \ddots & \vdots &\vdots &\vdots &\ddots &\vdots \cr
\beta_{1, \zeta-1} & \beta_{2, \zeta-1} & \cdots & \beta_{\mu, \zeta-1} & \beta_{\mu + 1,\zeta-1} & \beta_{\mu + 2, \zeta-1} &\cdots &\beta_{\mu + \zeta-1,\zeta-1}\cr
\end{pmat}
\end{equation}
is a $(\zeta + m) \times (\mu + \zeta - 1)$ matrix and $\boldsymbol{b}$ is a $(\zeta + m) \times 1$ vector including the entries of $\boldsymbol{b}_1$ in equation 4.4 as well as $c_{j-1}$'s of equation 4.8. Equation 4.10 can be written as
\begin{equation}
\mathcal{A} = \begin{pmat}({})
A_1 & & & B_1 \cr
C_1 & & & D_1 \cr
\end{pmat} =
\begin{pmat}({})
A_1 & & & \boldsymbol{0} \cr
C_1 & & & D_1 \cr
\end{pmat},
\end{equation}
where $C_1$ and $D_1$ equals the corresponding blocks of $\mathcal{A}$, respectively.


\begin{prp}
In equation 4.10 or 4.11, if $\mu = m + 1$ and $\det A_1 \ne 0$, $\mathcal{A}$ is nonsingular.
\end{prp}
\begin{proof}
Since $\det{A} = \det{A_1}\det{D_1} = \det{A_1}\prod_{\nu=1}^{\zeta-1}\beta_{\mu+\nu, \nu}$, the conclusion follows.
\end{proof}

\begin{rmk}
This proposition describes the mechanism of theorem 6 by matrix operations.
\end{rmk}

\subsection{A Two-Sided Example}
This section corresponds to the solution of section 3.6. Under the notations of section 3.6, to a unit $u_{k}$ , add a negative unit $u_{k}'$  whose zero-error point is knot $x_k$. Let $\phi_{k}(x) = \sigma(w'_{k}x + b_{k}')$ be the activation function of $u_{k}'$, and $p_{k\tau}'(x)$ for $\tau = 1, 2, \dots, \zeta$ be the $\tau$th polynomial of a spline approximating $\phi_{k}'(x)$ on subinterval $I_{\tau}$, satisfying
\begin{equation}
p_{k j}'(x) = p_{k,j-1}'(x) + \beta_{k,j-1}'(x - x_{j-1})_+^{m}.
\end{equation}
for $2 \le j \le k$. After introducing $u_k'$, matrix $\mathcal{A}$ of equation 4.11 is changed into
\begin{equation}
\mathcal{A}' = \begin{pmat}({})
A_1' & & & B_1' \cr
C_1' & & & D_1' \cr
\end{pmat},
\end{equation}
whose blocks correspond to those of equation 4.11, respectively, among which $A_1'=A_1$ and $C_1'=C_1$. Blocks $B_1'$ and $D'_1$ is obtained as follows. Equation 4.3 becomes
\begin{equation}
\mathcal{P}_1(x) = \sum_{\nu \in G}\lambda_{\nu}p_{\nu1}(x) + \lambda_{k}'p_{k1}'(x)
\end{equation}
and correspondingly, the submatrix $B_1$ of equation 4.11 should be modified to
\begin{equation}
B_1' = \begin{pmat}({cc||})
0 &\cdots &0 & \gamma_{k1} &0 &\cdots &0 \cr\
0 &\cdots &0 &\gamma_{k2} &0 &\cdots &0 \cr\
\vdots &\ddots &\vdots &\vdots &\vdots &\ddots &\vdots \cr\
0 &\cdots &0 &\gamma_{k,m+1} &0 &\cdots &0 \cr
\end{pmat},
\end{equation}
by adding a column between the original $k-1$th and $k$th columns, whose entries $\gamma_{ki}$'s for $1 \le i \le m+1$ are derived from the coefficients of the polynomial $p_{k1}'(x)$ of equation 4.14 weighted by $\lambda_{k}'$. The submatrix
\begin{equation}
\begin{aligned}
&\quad \quad\quad\quad\quad\quad\quad\quad\quad\quad\quad\quad\quad\quad\quad D_1' = \\
&\begin{pmat}({ccc||})
\beta_{\mu+1,1} &0 &0 &\cdots &\beta_{\mu+k,1}'&0 &0 &\cdots &0 \cr
\beta_{\mu+1,2} & \beta_{\mu+2,2} &0 &\cdots &\beta_{\mu+k,2}'&0 &0 &\cdots &0\cr
\vdots &\vdots &\vdots &\ddots &\vdots &\vdots &\ddots &\vdots \cr
\beta_{\mu+1,k} & \beta_{\mu+2,k} & \cdots &\beta_{\mu+k-1,k} &\beta'_{\mu+k,k} &0 &0 &\cdots &0\cr
\beta_{\mu+1,k+1} & \beta_{\mu+2,k+1} & \cdots &\beta_{\mu+k-1,k+1} &0 &\beta_{\mu+k,k+1} &0 &\cdots &0\cr
\vdots &\vdots &\vdots &\ddots &\vdots &\vdots &\ddots &\vdots \cr
\beta_{\mu + 1,\zeta-1} & \beta_{\mu + 2,\zeta-1} & \cdots &\beta_{\mu+k-1,\zeta-1} &0 &\beta_{\mu+k,\zeta-1} &\beta_{\mu+k+1,\zeta-1} &\cdots &\beta_{\psi,\zeta-1}\cr
\end{pmat},
\end{aligned}
\end{equation}
where $\psi = \mu+ \zeta-1$, which  is different from $D_1$ of equation 4.10 in adding a column
\begin{equation}
[\beta_{\mu+k,1}', \beta_{\mu+k,2}', \dots, \beta'_{\mu+k,k}, \dots,0,0]^T.
\end{equation}

\begin{prp}
In equation 4.13, suppose that $\mu = m+1$. Then $\mathcal{A}'$ is of size $(\zeta+m)\times(\zeta+m+1)$ and its block $A_1'$ is a square matrix. If block $A_1'$ is nonsingular, the rank of $\mathcal{A}'$ is $\zeta+m$.
\end{prp}
\begin{proof}
The reason is that it has a nonsingular square submatrix $\mathcal{A}$ of equation 4.11 obtained by deleting the added column in $B_1'$ and $D_1'$.
\end{proof}

\begin{rmk}
This proposition indicates the existence of the solution of lemma 5 or proposition 2 as well.
\end{rmk}

\subsection{General Solution Form}
\begin{dfn}
Given a generalized sigmoidal unit $\mathscr{U}$ whose activation function is $\phi(x) = \sigma(wx+b)$, its indicator function $\mathbb{I}: N \to \{0, 1\}$ with respect to $\varepsilon$, where $N = \{1, 2, \dots, \zeta\}$, is defined as
\begin{equation}
\mathbb{I}(i) =
\begin{cases}
\ \ 0, \ \text{if} \ \{\int_{x_{i-1}}^{x_{i}}\phi(x)^2dx\}^{1/2} < \varepsilon \\
\ \ 1, \ \ \ \ \ \ \ \ \ \ \ \text{otherwise}
\end{cases}
\end{equation}
for $i = 1, 2, \dots, \zeta$, where $x_0 = 0$, $x_{\zeta} = 1$ and $\varepsilon$ is a sufficiently small positive real number as required.
\end{dfn}

\begin{thm}
Notations being from section 4.1, to each local unit $u_j$'s for $j = \mu+1, \mu+2, \dots, \psi$, let $\mathbb{I}_{j}$ be its indicator function. Write
\begin{equation}
\mathfrak{A}\boldsymbol{\lambda} = \boldsymbol{b},
\end{equation}
where
\begin{equation}
\mathfrak{A} = \begin{pmat}({})
A & & & B \cr
C & & & D \cr
\end{pmat},
\end{equation}
in which $A = A_1$, $C = C_1$, with $A_1$ and $C_1$ from equation 4.11,
\begin{equation}
B = \begin{pmat}({})
\mathbb{I}_{\mu+1}(1)\gamma_{\mu+1, 1} &\mathbb{I}_{\mu+2}(1)\gamma_{\mu+2, 1} &\cdots  &\mathbb{I}_{\psi}(1)\gamma_{\psi, 1} \cr
\mathbb{I}_{\mu+1}(1)\gamma_{\mu+1, 2} &\mathbb{I}_{\mu+2}(1)\gamma_{\mu+2, 2} &\cdots  &\mathbb{I}_{\psi}(1)\gamma_{\psi, 2} \cr
\vdots &\vdots  &\ddots &\vdots \cr\
\mathbb{I}_{\mu+1}(1)\gamma_{\mu+1, m+1} &\mathbb{I}_{\mu+2}(1)\gamma_{\mu+2, m+1} &\cdots &\mathbb{I}_{\psi}(1)\gamma_{\psi, m+1}  \cr
\end{pmat},
\end{equation}
where $\gamma_{j, \nu}$ for $j=\mu+1,\mu+2,\dots,\psi$ and $\nu=1,2,\dots,m+1$ is the weighted coefficient of the corresponding term of the polynomial on $[0, x_1]$ approximating $\phi_j(x)$ as in equation 4.15, and
\begin{equation}
D = \begin{pmat}({})
\mathbb{I}_{\mu+1}(2)\beta_{\mu+1,1} &\mathbb{I}_{\mu+2}(2)\beta_{\mu+2, 1} & \cdots &\mathbb{I}_{\psi}(2)\beta_{\psi,1} \cr
\mathbb{I}_{\mu+1}(3)\beta_{\mu+1,2} &\mathbb{I}_{\mu+2}(3)\beta_{\mu+2, 2} & \cdots &\mathbb{I}_{\psi}(3)\beta_{\psi,2} \cr
\vdots & \vdots & \ddots & \vdots  \cr\
\mathbb{I}_{\mu+1}(\zeta)\beta_{\mu+1,\zeta-1} &\mathbb{I}_{\mu+2}(\zeta)\beta_{\mu+2,\zeta-1} & \cdots &\mathbb{I}_{\psi}(\zeta)\beta_{\psi,\zeta-1} \cr
\end{pmat},
\end{equation}
where $\beta_{ji}$ for $i = 1, 2, \dots, \zeta-1$ is defined as the nonzero entry of equation 4.10. Then any spline $s(x) \in \mathfrak{S}_1^{m}( \Delta)$ can be realized by a two-layer network $\mathfrak{N}$ in terms of equation 4.19, provided that the rank of $\mathfrak{A}$ is $\zeta + m$.
\end{thm}
\begin{proof}
The proof is the generalization of sections 4.1 and 4.2.
\end{proof}

\begin{dfn}[Spline matrix]
Matrix $\mathfrak{A}$ of equation 4.20 is called the \textsl{spline matrix} of a two-layer neural network $\mathfrak{N}$ with generalized sigmoidal units.
\end{dfn}

\section{Multivariate Local Approximation}
In lemma 2 of section 1, Taylor series expansions of a univariate function were realized by a two-layer neural network with smooth activation functions, and this section generalizes that idea to the multivariate case. The obtained result is the basis of multivariate-spline construction of section 6.

\subsection{Theoretical Framework}
We require the number of the coefficients of a multivariate polynomial. To an $n$-variate homogeneous polynomial of degree $m$, the number of its coefficients  is $\mathscr{N}(n, m) = \binom{n+m-1}{m}$, whose calculation can be reduced to the combination problem of multisets (\citet*{Brualdi2004}'s theorem 3.5.1). The number of the coefficients of an $n$-variate polynomial $p(\boldsymbol{x})$ of degree $m$ is
\begin{equation}
\mathcal{N}(n, m) = \binom{n+m}{m},
\end{equation}
because $p(\boldsymbol{x})$ comprises $m+1$ sets of homogeneous polynomials whose degrees are $m, m-1, \dots, 0$, respectively, such that
\begin{equation*}
\mathcal{N}(n, m) = \binom{n+m-1}{m} + \binom{n+m-2}{m-1} + \dots + \binom{n}{1} + \binom{n-1}{0} = \binom{n+m}{m},
\end{equation*}
by the identity $\binom{n-1}{0} = \binom{n}{0}$ and by repeatedly applying Pascal's formula from the right side of the middle term.

\begin{dfn}[Multivariate smooth function]
A function $f: [0, 1]^n \to \mathbb{R}$ is said to be smooth with order $m$, if to each $\boldsymbol{x} \in [0, 1]^n$,
\begin{equation}
\frac{\partial^{k} f(\boldsymbol{x})}{\partial x_1^{\alpha_1}\partial x_2^{\alpha_2}\dots\partial x_n^{\alpha_n}}
\end{equation}
is continuous, where $k = 0, 1, \dots, m$, $\boldsymbol{x} = [x_1, x_2, \dots, x_n]^T$ and
\begin{equation}
\alpha_1 + \alpha_2 + \dots + \alpha_n = k
\end{equation}
with each $\alpha_i$ for $i = 1, 2, \dots, n$ being a nonnegative integer. The set of this type of function is denoted by $C^m([0,1]^n)$.
\end{dfn}

\begin{prp}
To $f(\boldsymbol{x}) \in C^{m}([0, 1]^n)$, the number of the elements of
\begin{equation}
D_k = \{\frac{\partial^{k} f(\boldsymbol{x})}{\partial x_1^{\alpha_1}\partial x_2^{\alpha_2}\dots\partial x_n^{\alpha_n}}: \sum_{i=1}^{n}\alpha_i = k, \alpha_i \ge 0, \alpha_i \in \mathbb{Z}\}
\end{equation}
is $|D_k| = \binom{n+k-1}{k}$; and the total number of the terms of the partial derivatives of $f(\boldsymbol{x})$ up to order $m$ is $\sum_{k = 0}^{m} |D_k| = \binom{n+m}{m}$.
\end{prp}
\begin{proof}
Note that in equation 5.4, the term $\partial x_1^{\alpha_1}\partial x_2^{\alpha_2}\dots\partial x_n^{\alpha_n}$ has an one-to-one correspondence with $x_1^{\alpha_1}x_2^{\alpha_2}\dots x_n^{\alpha_n}$. So the conclusion is obvious by equation 5.1.
\end{proof}

\begin{dfn}[Order of polynomial terms]
Write
\begin{equation}
p(\boldsymbol{x}) = \sum_{0\le s(\boldsymbol{\alpha}) \le m} c_{\boldsymbol{\alpha}}x_1^{\alpha_1}x_2^{\alpha_2}\dots x_n^{\alpha_n},
\end{equation}
where $s(\boldsymbol{\alpha}) = \alpha_1 + \alpha_2 + \dots + \alpha_n$, an $n$-variate polynomial of degree $m$. We say
\begin{equation}
x_1^{\alpha_1}x_2^{\alpha_2}\dots x_n^{\alpha_n} \prec x_1^{\alpha_1'}x_2^{\alpha_2'}\dots x_n^{\alpha_n'},
\end{equation}
provided that: (a) $\sum_{i=1}^n\alpha_i < \sum_{i=1}^n\alpha_i'$; (b) $\sum_{i=1}^n\alpha_i = \sum_{i=1}^n\alpha_i'$ but $\alpha_{1} > \alpha'_{1}$; (c) $\sum_{i=1}^n\alpha_i = \sum_{i=1}^n\alpha_i'$, $\alpha_{j} = \alpha_{j}'$ for $j = 1, 2, \dots, \beta$ but $\alpha_{j+1} > \alpha'_{j+1}$, with $1 \le \beta \le n-1$, which is read as ``$x_1^{\alpha_1}x_2^{\alpha_2}\dots x_n^{\alpha_n}$ is \textsl{smaller} than $x_1^{\alpha_1'}x_2^{\alpha_2'}\dots x_n^{\alpha_n'}$, or the latter is \textsl{bigger} than the former''.
\end{dfn}

\begin{rmk}
The reason for introducing this definition will be explained in the proof of lemma 6.
\end{rmk}

It's easy to verify:
\begin{prp}
The relation ``$\prec$'' of definition 10 is a strict partial order.
\end{prp}

\begin{dfn}[Multivariate generalized Wronskian matrix]
Denote by $u_i(\boldsymbol{x}) \in C^{m}([0, 1]^n)$ for $i = 1, 2, \dots, \gamma$ a set of functions. Write

\begin{equation}
\begin{aligned}
D(u_i(\boldsymbol{x})) := [u_i(\boldsymbol{x}), \frac{\partial u_i(\boldsymbol{x})}{\partial x_1},&\dots, \frac{\partial u_i(\boldsymbol{x})}{\partial x_n},  \frac{\partial^2 u_i(\boldsymbol{x})}{\partial x_1^2}, \frac{\partial^2 u_i(\boldsymbol{x})}{\partial x_1x_2}, \dots, \frac{\partial^2 u_i(\boldsymbol{x})}{\partial x_n^2}, \dots, \\
&\frac{\partial^m u_i(\boldsymbol{x})}{\partial x_1^m}, \frac{\partial^m u_i(\boldsymbol{x})}{\partial x_1^{m-1}x_2}, \dots, \frac{\partial^m u_i(\boldsymbol{x})}{\partial x_n^m}]
\end{aligned},
\end{equation}
a vector of length
\begin{equation}
\tau = \binom{n+m}{m}
\end{equation}
(proposition 5) whose entries are composed of the $\nu$th partial derivative for all $\nu = 0, 1, \dots, m$ and are arranged in the strict partial order of equation 5.6 because there exists a bijection map between ${\partial x_1^{\alpha_1}\partial x_2^{\alpha_2}\dots\partial x_n^{\alpha_n}}$ and $x_1^{\alpha_1}x_2^{\alpha_2}\dots x_n^{\alpha_n}$. Then
\begin{equation}
\mathcal{W}(u_i(\boldsymbol{x}_0)) =
\begin{pmat}({})
D(u_1(\boldsymbol{x}_0)) \cr
\vdots \cr
D(u_{\gamma}(\boldsymbol{x}_0))\cr
\end{pmat}
\end{equation}
of size $\gamma \times \tau$ for $\gamma \ge \tau$ is called the \textsl{generalized Wronskian matrix} of $u_i(\boldsymbol{x}_0)$'s with respect to order $m$.
\end{dfn}

\begin{thm}
If the rank of generalized Wronskian matrix of equation 5.9 is $\tau$ of equation 5.8, the linear combination of $u_i(\boldsymbol{x}_0)$'s, denoted by $\sum_i\lambda_iu_i(\boldsymbol{x}_0)$, can approximate an arbitrary $n$-variate polynomial $p(\boldsymbol{x})$ of degree $m$ to any desired accuracy, within a sufficiently mall neighbourhood $\delta(\boldsymbol{x}_0)$ of $\boldsymbol{x}_0$. Especially, when $\gamma = \tau$, the condition is that the generalized Wronskian matrix is nonsingular.
\end{thm}
\begin{proof}
The proof is similar to that of theorem 1 and the difference lies in multivariate Taylor series expansions. So we have
\begin{equation}
\mathcal{W}(u_i(\boldsymbol{x}_0))^T\boldsymbol{\lambda} = \boldsymbol{a},
\end{equation}
corresponding to equation 2.13 of the univariate case.
\end{proof}

\subsection{Construction of Local Approximation}
\begin{lem}
Let $\phi_i(\boldsymbol{x}) = \sigma(\boldsymbol{w}_i^T\boldsymbol{x} + b_i) \in C^{m}([0, 1]^n)$ for $i = 1, 2, \dots, \tau$ be the activation function of the $i$th unit of the hidden layer of a two-layer neural network $\mathfrak{N}$, where $\tau$ is from equation 5.8. Then, at any point $\boldsymbol{x}_0 \in [0,1]^n$, a nonsingular generalized Wronskian matrix $\mathcal{W}(\phi_i(\boldsymbol{x}_0))$ can be constructed with arbitrary precision.
\end{lem}
\begin{proof}
The proof is composed of three parts. Part 1 illustrates the main idea by an example. Part 2 deals with the order of the terms of a multivariate polynomial. Part 3 proposes the general construction method.

\noindent
\textbf{Part 1}. For simplicity, we use the following notations in the next equation 5.11: $\sigma'_{x_{\nu}}(y):=\partial\sigma(y)/\partial x_{\nu}$ for $\nu = 1, 2$, $\sigma''_{x_{\nu}x_{\mu}}(y):=\partial^2\sigma(y)/\partial x_{\nu}^{\alpha_1}\partial x_{\mu}^{\alpha_2}$ for $\nu, \mu = 1, 2$, with $\alpha_1 + \alpha_2 = 2$ and $0 \le \alpha_1, \alpha_2 \le 2$. Analogous to lemma 2, write
\begin{equation}
\begin{aligned}
&\quad \quad \quad \quad \quad \quad \quad \quad \quad \quad \quad \mathcal{W}(\sigma(y_i)) = \\
&\begin{pmat}({})
\sigma(y_1) &\sigma'_{x_1}(y_1)w_{11}  &\sigma'_{x_2}(y_1)w_{12} &\sigma''_{x_1^2}(y_1)w_{11}^2 &\sigma''_{x_1x_2}(y_1)w_{11}w_{12} & \sigma''_{x_2^2}(y_1)w_{12}^2 \cr
\sigma(y_2) &\sigma'_{x_1}(y_2)w_{21}  &\sigma'_{x_2}(y_2)w_{22} &\sigma''_{x_1^2}(y_2)w_{21}^2 &\sigma''_{x_1x_2}(y_2)w_{21}w_{22} & \sigma''_{x_2^2}(y_2)w_{22}^2 \cr
\sigma(y_3) &\sigma'_{x_1}(y_3)w_{31}  &\sigma'_{x_2}(y_3)w_{32} &\sigma''_{x_1^2}(y_3)w_{31}^2 &\sigma''_{x_1x_2}(y_3)w_{31}w_{32} & \sigma''_{x_2^2}(y_3)w_{32}^2 \cr
\sigma(y_4) &\sigma'_{x_1}(y_4)w_{41}  &\sigma'_{x_2}(y_4)w_{42} &\sigma''_{x_1^2}(y_4)w_{41}^2 &\sigma''_{x_1x_2}(y_4)w_{41}w_{42} & \sigma''_{x_2^2}(y_4)w_{42}^2 \cr
\sigma(y_5) &\sigma'_{x_1}(y_5)w_{51}  &\sigma'_{x_2}(y_5)w_{52} &\sigma''_{x_1^2}(y_5)w_{51}^2 &\sigma''_{x_1x_2}(y_5)w_{51}w_{52} & \sigma''_{x_2^2}(y_5)w_{52}^2 \cr
\sigma(y_6) &\sigma'_{x_1}(y_6)w_{61}  &\sigma'_{x_2}(y_6)w_{62} &\sigma''_{x_1^2}(y_6)w_{61}^2 &\sigma''_{x_1x_2}(y_6)w_{61}w_{62} & \sigma''_{x_2^2}(y_6)w_{62}^2 \cr
\end{pmat},
\end{aligned}
\end{equation}
where $y_i = \boldsymbol{w}_i^T\boldsymbol{x}_0 + b_i$ for $1 \le i \le 6$ and $w_{ij}$ for $j = 1, 2$ is the entry of $\boldsymbol{w}_i = [w_{i1}, w_{i2}]^T$, which is the generalized Wronskian matrix of $\phi_i(x)$'s with respect to order $2$. To the diagonal entries of $\mathcal{W}(\sigma(y_i))$, the partial derivatives of $\sigma(y)$ on a fixed point could be nonzero by adjusting the bias parameter; thus we only consider the construction of $w_{ij}$'s to make $\mathcal{W}(\sigma(y_i))$ approximately nonsingular.

A key point is that only the diagonal and upper-triangular part of $\mathcal{W}(\sigma(y_i))$ should be taken into consideration, while the lower-triangular part can be ignored. The ultimate goal is to set each diagonal (except for the first and last ones) entry to be $\Delta t$, and to set each upper-triangular entry to be of $o(\Delta t)$. Another point is that, with $w_{ij}$'s as variables, the polynomial degrees of the diagonal entries are the same as those of the entries of each row, respectively, as shown in an example of Table 1.
\begin{table}[htbp]
\centering
\begin{tabular}{|r|r|r|r|r|r|r|}
\hline
Diagonal terms & $1$ & $w_{21}$ & $w_{32}$ & $w_{41}^2$ & $w_{51}w_{52}$ & $w_{62}^2$\\
\hline
First-row terms & $1$ & $w_{11}$ & $w_{12}$ & $w_{11}^2$ & $w_{11}w_{12}$ & $w_{12}^2$\\
\hline
\end{tabular}
\caption{Polynomial degrees of a generalized Wronskian matrix}
\end{table}
The two points above play an important role in the proof of this lemma.

For simplicity of descriptions, let
\begin{equation}
W =
\begin{pmat}({})
1 & w_{11} & w_{12} & w_{11}^2 & w_{11}w_{12} & w_{12}^2 \cr
1 & w_{21} & w_{22} & w_{21}^2 & w_{21}w_{22} & w_{22}^2 \cr
1 & w_{31} & w_{32} & w_{31}^2 & w_{31}w_{32} & w_{32}^2 \cr
1 & w_{41} & w_{42} & w_{41}^2 & w_{41}w_{42} & w_{42}^2 \cr
1 & w_{51} & w_{52} & w_{51}^2 & w_{51}w_{52} & w_{52}^2 \cr
1 & w_{61} & w_{62} & w_{61}^2 & w_{61}w_{62} & w_{62}^2 \cr
\end{pmat},
\end{equation}
regardless of the coefficients of variable $w_{ij}$'s in equation 5.11. We deal with $W$ from the first row; just letting $w_{11} = w_{12} = \Delta t^{1 + c}$ with $c > 0$, all the entries except for the first one is $o(\Delta t)$. To the second row, set the diagonal entry $w_{21} = \Delta t$; if $w_{22} = \Delta t^{1 + c}$, $W(2, j)$'s for $j \ge 3$ would be $o(\Delta t)$. The third row is similar to the second one. In the fourth row, let $w_{41} = \Delta t^{1 / 2}$ and $w_{42} = \Delta t^{\beta}$ , where $\beta$ can be arbitrarily set as long as $W(4, 5)$ and $W(4, 6)$ are $o(\Delta t)$. The above cases are all trivial, because the diagonal entry has only one variable such that other variables can be freely adjusted without affecting it.

The more difficult case is the fifth row when $w_{52}$ appears in both diagonal and upper-triangle positions. We should set $w_{52}$ properly to make $W(5, 5) = \Delta t$ and $W(5, 6) = o(\Delta t)$ simultaneously. Let $w_{51} = \Delta t^{c}$ for $c > 0$ and $w_{52} = \Delta t^{1 - c}$; then $w_{51}w_{52} = \Delta t$ and if $c < 1/2$, $w_{52}^2 = \Delta t^{2(1 - c)}$ whose exponential $2(1 - c) > 1$ , that is, $w_{52}^2 = o(\Delta t)$. The last row is trivial by setting $w_{62} = 1/\Delta t^4$ or $w_{62} = \Delta t^{-4}$. So we finally get
\begin{equation}
W =
\begin{pmat}({})
1 & o(\Delta t) & o(\Delta t) & o(\Delta t) & o(\Delta t) & o(\Delta t) \cr
* & \Delta t & o(\Delta t) & o(\Delta t) & o(\Delta t) & o(\Delta t) \cr
* & * & \Delta t & o(\Delta t) & o(\Delta t) & o(\Delta t) \cr
* & * & * & \Delta t & o(\Delta t) & o(\Delta t) \cr
* & * & * & * & \Delta t & o(\Delta t) \cr
* & * & * & * & * & 1/\Delta t^4 \cr
\end{pmat},
\end{equation}
where the ``$*$'' entries mean that they cannot influence the determinant of the matrix when $\Delta t \to 0$. So $\lim_{\Delta t \to 0}\det W(\Delta t) = 1$. By this result, we have
\begin{equation}
\lim_{\Delta t \to 0}\det \mathcal{W}(\sigma(y_i), \Delta t) = \alpha \ne 0,
\end{equation}
where $\alpha$ is the product of the coefficients of the diagonal entries of $\mathcal{W}(\sigma(y_i))$.

\vspace{3.0mm}
\noindent
\textbf{Part 2}. Write a polynomial of degree $2$ with three variables in this form
\begin{equation}
p(\boldsymbol{w}) = p_0 + p_1 + p_2 + p_3,
\end{equation}
where $\boldsymbol{w} = [w_1, w_2, w_3]^T$, $p_0 = 1$, $p_1 = w_1 + w_2 + w_3$,
\begin{equation}
p_2 = \big\{w_1^2 + w_1(w_2 + w_3)\big\} + \big\{w_2^2 + w_2w_3\big\} + w_3^2,
\end{equation}
and
\begin{equation}
p_3 = \big\{w_1^3 + w_1^2(w_2 + w_3) + w_1(w_2^2 + w_2w_3 + w_3^2)\big\} + \big\{w_2^3 + w_2^2w_3 + w_2w_3^2\big\} + w_3^3.
\end{equation}

The order of the terms of $p(\boldsymbol{w})$ in equation 5.15 is important to this lemma. Each $p_i$ for $i = 0, 1, 2, 3$ is a homogeneous polynomial of degree $i$ and $p_i$'s are arranged in ascending order of their degrees. The terms of each $p_i$ are also arranged in a certain order. For example, to $p_2$ of equation 5.16, the first term $w_1^2$ is a univariate-$w_1$ polynomial of degree 2. Then reduce the degree of $w_1^2$ and simultaneously increase that of other variables with the constraint that the total degree remains invariant, resulting in two terms $w_1(w_2 + w_3)$. The order of introducing new variables in $p_{2}$ is by the ascending order of the subscripts of $w_i$'s for $i > 1$. When the degree of $w_1$ is reduced to zero, variable $w_2$ plays the role of $w_1$ and the similar process of creating new terms is repeated as above. A difference is that when introducing new variables for $w_2$, only those whose subscripts are greater than 2 should be considered. The last case $w_3^2$ has only one term because there's no variable left with subscript greater than 3.

A more complex example is $p_3$ of equation 5.17. First dealt with variable $w_1$ by repeatedly reducing its degree from 3 to 1 and simultaneously increasing the degree of other variables. Note that in $w_1(w_2^2 + w_2w_3 + w_3^2)$, the terms of $w_2^2 + w_2w_3 + w_3^2$ with degree 2 should be arranged by the rule of $p_2$ discussed above; so this is a recursive procedure obeying a consistent principle. When the degree of $w_1$ is zero, $w_2$ follows and would be processed analogous to $w_1$. And also, a difference from $w_1$ is that only variable $w_3$ whose subscript is greater than $w_2$ should be introduced when decreasing the degree of $w_2^3$. The last step is for $w_3$, which is trivial as the case of $p_2$.

We summarize a rule for the order of the terms of $p(\boldsymbol{w})$ above, which is in fact the strict partial order ``$\prec$'' of definition 10. So to any $n$-variate polynomial $p(\boldsymbol{w})$ of degree $m$, we can arrange its terms in the form $\{q_k: 1 \le k \le \tau\}$ such that
\begin{equation}
(q_i = w_1^{k_1}w_2^{k_2}\dots w_n^{k_n}) \prec (q_{i+1} = w_1^{k_1'}w_2^{k_2'}\dots w_n^{k_n'})
\end{equation}
for $i= 1,2,\dots, \tau-1$, where $q_{i}$ and $q_{i+1}$ are arbitrary two terms of $p(\boldsymbol{w})$ without considering the coefficients.

\vspace{3.0mm}
\noindent
\textbf{Part 3}. On the basis of the preceding two parts, we give the general construction method. Given a generalized Wronskian matrix $\mathcal{W}$ with respect to order $m$, the associated matrix $W$ is obtained by ignoring the coefficients of the terms, analogous to equation 5.12. To each row of $W$, the entries are regard as the terms of a polynomial and must be arranged by the order ``$\prec$'' of definition 10 as in equation 5.18.

The main procedure is as follows. We process each row of $W$ one by one according to the ascending order of $i = 1, 2, \dots, \tau = \binom{n+m}{m}$. The case of $i = 1$ is trivial as the example of part 1. In the $\nu$th step for $\nu = 2, \dots, \tau-1$, let $W(\nu, \nu)=\Delta t$; and the terms bigger than $W(\nu, \nu)$ are set to be $o(\Delta t)$. Finally, let $W(\tau, \tau) =  1/\Delta t^{\tau-2}$. The remaining proof is mainly for the second case.

Let $p(\boldsymbol{w})$ be the polynomial of the $\nu$th row. For simplicity, denote the term $W(\nu, \nu)$ of $p(\boldsymbol{w})$ by $w_1^{k_1}w_2^{k_2}\dots w_n^{k_n}$ or $\prod_{j=1}^nw_j^{k_j}$, disregarding the variables' row index (or the first subscript); similarly, each term of $p(\boldsymbol{w})$ bigger than $W(\nu, \nu)$ is denoted by $w_1^{k_1'}w_2^{k_2'}\dots w_n^{k_n'}$.

When $w_1^{k_1}w_2^{k_2}\dots w_n^{k_n}$ is a univariate polynomial with only one of $k_j$'s being nonzero, this case is trivial. For example, suppose that $w_1^{k_1}w_2^{k_2}\dots w_n^{k_n} = w_{\mu}^k$ for $1\le \mu \le n$. Just let $w_{\mu} = \Delta t^{1/k}$. The terms of this row bigger than $w_{\mu}^k$ have three possibilities: (a) it's a univariate polynomial with variable $w_{\mu}$ whose exponential is greater than $k$ and thus is $o(\Delta t)$; (b) it's a multivariate polynomial having other variables; in this case we can set the exponential of any other variable to be $\Delta t^{\alpha}$, where $\alpha$ can be arbitrarily large such that the polynomial becomes $o(\Delta t)$; (c) it's a polynomial that does not contain $w_{\mu}$ and can also be freely adjusted without any restriction.

The difficult case is when
\begin{equation}
W(\nu, \nu)= w_1^{k_1}w_2^{k_2}\dots w_n^{k_n}
\end{equation}
is a multivariate polynomial. Denoted by $S$ the set of the terms of $p(\boldsymbol{w})$ bigger than $w_1^{k_1}w_2^{k_2}\dots w_n^{k_n}$. Suppose that
\begin{equation}
\sum_{i=1}^nk_i = k.
\end{equation}
The set $S$ can be classified into three categories, including
\begin{equation}
S = S_1 \cup S_2 \cup S_3,
\end{equation}
in which
\begin{equation}
S_1 = \{w_1^{k_1'}w_2^{k_2'}\dots w_n^{k_n'}: 0<k_1'<k_1, \ \sum_{i=1}^nk_i' = k \},
\end{equation}
\begin{equation}
S_2 = \{w_1^{k_1'}w_2^{k_2'}\dots w_n^{k_n'}: 0<k_1' \le k_1, \ \sum_{i=1}^nk_i' > k \},
\end{equation}
and
\begin{equation}
S_3 = S-(S_1 \cup S_2),
\end{equation}
where $k_1$ is from equation 5.19.

Let
\begin{equation}
w_1 = \Delta t^{\frac{1+c_1}{k+1}},
\end{equation}
and
\begin{equation}
w_j = \Delta t^{\frac{1+c_j}{k}},
\end{equation}
for $j = 2, 3, \dots, n$, where
\begin{equation}
c_j > c_1 > 0.
\end{equation}
We should adjust the parameters of $c_i$'s for $i=1,2,\dots,n$ such that
\begin{equation}
w_1^{k_1}w_2^{k_2}\dots w_n^{k_n} = \Delta t
\end{equation}
and each element of $S$ becomes $o(\Delta t)$. To satisfy equation 5.28, we have
\begin{equation}
\frac{1+c_1}{k+1}k_1 + \frac{1+c_2}{k}k_2 + \dots + \frac{1+c_n}{k}k_n = 1,
\end{equation}
which can be simplified into
\begin{equation}
c_1\frac{k_1}{k+1} + \sum_{j = 2}^nc_j\frac{k_j}{k} = \frac{k_1}{k(k+1)}
\end{equation}
by equation 5.20.

We now prove that if equation 5.30 has a solution of $c_i$'s, all the elements of $S$ would be $o(\Delta t)$. To the case of $S_1$ of equation 5.22, when $k_1$ is reduced, some of $k_j$'s would be increased, due to $\sum_ik_i' = \sum_{i}k_i = k$. Then equation 5.29 is changed into
\begin{equation}
\frac{1+c_1}{k+1}k_1' + \frac{1+c_2}{k}k_2' + \dots + \frac{1+c_n}{k}k_n' > 1,
\end{equation}
because inequality 5.27 leads to $(1+c_j)/k > (1+c_1)/(k+1)$; thus, each element of $S_1$ is $o(\Delta t)$.

For $S_2$, let $\gamma = \min\{w_i: 1\le i\le n\} = \Delta t^{\frac{1+c_1}{k+1}}$. So
\begin{equation}
w_1^{k_1'}w_2^{k_2'}\dots w_n^{k_n'} \ge \gamma^{k+1} = \Delta t^{1+c_1},
\end{equation}
which is also $o(\Delta t)$. The elements of $S_3$ can be dealt with similarly to $S_2$ and the difference only lies in the value of $\gamma$.

Therefore, the key point is the existence of $c_i$'s of equation 5.30 under the constraint of inequality 5.27. We can make $c_1$ as small as possible until all the $c_j$'s are greater than it; noting that the extreme case is that $c_1 = 0$ and $c_j >0$, before reaching this limit, there's a large space to tune them to reach a solution. For example, we can first find an initial solution with $c_i>0$ for all $i$ but cannot satisfy inequality 5.27 for all $j$'s. Choose one of $c_j$'s smaller than $c_1$, say $c_{\mu}$. Decrease $c_1$ and simultaneously increase $c_{\mu}$ to compensate the reduction of $c_1k_1/(k+1)$ until $c_{\mu} > c_1$. Each $c_j$ that doesn't fulfil inequality 5.27 should be dealt with analogously to $c_{\mu}$ and a solution of $c_i$'s can be finally obtained. This completes the proof.
\end{proof}

\begin{rmk}
This lemma is the generalization of the univariate case of lemma 2. Its main difficulty is the disturbance between different dimensions, such that some of the weights cannot be freely adjusted. This problem was solved by introducing an order of polynomial terms and carefully controlling the decreasing rate of the weights.
\end{rmk}

\begin{thm}[Construction of multivariate Taylor series expansions]
Denote by $\phi_i(\boldsymbol{x}) = \sigma(\boldsymbol{w}_i^T\boldsymbol{x} + b_i) \in C^{m}([0, 1]^n)$ for $i = 1, 2, \dots, \binom{n+m}{m}$ the activation function of the $i$th unit of a two-layer neural network $\mathfrak{N}$. Let $\delta(\boldsymbol{x}_0)$ be a neighbourhood of an arbitrary point $\boldsymbol{x}_0 \in [0, 1]^n$. They $\mathfrak{N}$ can realize any $n$-variate polynomial of degree $m$ defined on $\delta(\boldsymbol{x}_0)$ to arbitrary accuracy, provided that $\delta(\boldsymbol{x}_0)$ is sufficiently small.
\end{thm}
\begin{proof}
Lemma 6 and theorem 9 imply this conclusion.
\end{proof}

\section{Multivariate Global Approximation}
This section turns to global approximation for multivariate functions, also in terms of splines as the univariate case of section 3. \textbf{Section 6.1} constructs a multivariate spline approximating a given function to be realized by two-layer neural networks. \textbf{Section 6.2} gives a fundamental result, a recurrence formula of polynomial pieces of a multivariate spline analogous to equation 3.4. \textbf{Sections 6.3} deals with the zero-error part of a unit for higher-dimensional input to control approximation error.

\textbf{Section 6.4} presents a property of solutions that is useful in locally determining the output weight of units. \textbf{Section 6.5} proposes a most distinguished character of two-layer neural networks---smooth-continuity restriction. \textbf{Section 6.6} discusses spline construction over a single strict partial order, the simplest multivariate solution leading to more complex ones. \textbf{Section 6.7} proves universal approximation. \textbf{Section 6.8} investigates the two-sided solutions that are ubiquitous in engineering. \textbf{Section 6.9} is a short discussion about generalized tanh units.

\subsection{Multivariate Splines}
Two main techniques had been developed for multivariate splines. One is by tensor product \citep*{de Boor2001} and the other is by a relationship between adjacent polynomial pieces \citep*{Chui1983}. In order to implement a spline via neural networks, we propose a new method based on the integral of linear hyperplanes, with an advantage that an approximation to a given function can be simultaneously obtained.

In the rest of the paper, some definitions are not formally presented because they can be found in \citet*{Huang2024}, such as regions, adjacent regions, negative units, one-sided bases and boundaries. But for a self-contained purpose, they would be shortly described when first used.

A region is a connected part of $\mathbb{R}^n$ divided by a set of $n-1$-dimensional hyperplane, including its boundary. A knot $l$ in $\mathbb{R}^n$ is an $n-1$-dimensional hyperplane $\boldsymbol{w}^T\boldsymbol{x} + b = 0$ and can be regarded as being from the activation function $\sigma(\boldsymbol{w}^T\boldsymbol{x} + b )$ of a unit $\mathcal{U}$. Denote by $l^+$ the region whose output of $\mathcal{U}$ is positive, and by $l^0$ the region $\mathbb{R}^n-l^+$, both of which should include their boundary via a closure operation of sets.
\begin{dfn}[Smooth polynomial pieces]
Let $f: R_1 \cup R_2 \to \mathbb{R}$ be a piecewise polynomial of degree $m$, with a knot $\mathcal{L} = R_1 \cap R_2$. Denote by $p_1(\boldsymbol{x})$ and $p_2(\boldsymbol{x})$ the two polynomials on $R_1$ and $R_2$, respectively. If we say that $p_1(\boldsymbol{x})$ is smoothly continuous with $p_2(\boldsymbol{x})$ at $\mathcal{L}$ by order $m-1$, it means that
\begin{equation}
D_l^{k}p_2(\boldsymbol{x}_0) = D_l^{k}p_1(\boldsymbol{x}_0)
\end{equation}
for $k = 0, 1, \dots, m-1$ but $D_l^{m}p_2(\boldsymbol{x}_0) \ne D_l^{m}p_1(\boldsymbol{x}_0)$, where $l$ is an arbitrary one-dimensional line embedded in $\mathbb{R}^n$ satisfying $l \cap \mathcal{L} \ne \emptyset$ and $l \not\subset \mathcal{L}$, $\boldsymbol{x}_0 = l \cap \mathcal{L}$ and $D_l^{k}$ represents the $k$th directional-derivative operator along $l$. We sometimes write $p_1(\boldsymbol{x}) \overset{m-1}{\frown} p_2(\boldsymbol{x})$.
\end{dfn}

Two regions of $R_1$ and $R_2$ are said to be \textsl{adjacent} if $R_1 \cap R_2$ is $n-1$-dimensional as a part of the $n-1$-dimensional hyperplane separating them.
\begin{dfn}[Space of multivariate smooth splines]
Denote by $H$ a set of $n-1$-dimensional hyperplanes of $\mathbb{R}^n$, and by $\mathcal{R} = \bigcup_{i = 1}^{\zeta}R_i$ the set of the regions of $U = [0, 1]^n$ formed by $H$. Suppose that $U = \mathcal{R}$. Let
\begin{equation}
\begin{aligned}
\mathfrak{S}^{m}_n(H, \mathcal{R})  &:=  \{s: s(\boldsymbol{x})=s_i(\boldsymbol{x}) \in \mathcal{P}_{m} \ \text{for} \ x \in R_i, \\& s_i \overset{m-1}{\frown} \mathscr{N}_i, \ i = 1, 2, \dots, \zeta\},
\end{aligned}
\end{equation}
where $\mathcal{P}_{m}$ is the set of $n$-variate polynomials of degree $m$, $\mathscr{N}_i = \{s_{n_{i\kappa}}(\boldsymbol{x}): \kappa = 1, 2, \dots, \phi_i, 1 \le n_{i\kappa} \le \zeta\}$ whose element $s_{n_{i\kappa}}(\boldsymbol{x})$ is defined on the region $R_{n_{i\kappa}}$ that is adjacent to $R_i$, and $s_i \overset{m-1}{\frown} \mathscr{N}_i$ means that each $s_{n_{i\kappa}}\in \mathscr{N}_i$ is smoothly continuous with $s_i$ by order $m-1$ at the knot $k_i = R_i \cap R_{n_{i\kappa}}$. We call $\mathfrak{S}^{m}_n(H, \mathcal{R})$ the space of smooth splines of order $m$ with respect to $H$ and $\mathcal{R}$.
\end{dfn}

A multivariate smooth function was defined in definition 9 and the following proposition gives an equivalent definition that is more suitable for use in this paper.
\begin{prp}
A function $f(\boldsymbol{x}) \in C^m([0, 1]^n)$, if and only if the corresponding univariate case
\begin{equation}
f(t) =  f(\boldsymbol{x}_0 + t\boldsymbol{d}) \in C^m(\mathbb{R})
\end{equation}
holds for arbitrary point $\boldsymbol{x}_0 \in [0, 1]^n$ and for arbitrary direction $\boldsymbol{d}$.
\end{prp}
\begin{proof}
Let
\begin{equation}
D^{(k, \boldsymbol{d})} f(\boldsymbol{x}_0)= \sum_{\{\alpha_i\} \in A} \frac{\partial^{k} f(\boldsymbol{x}_0)}{\partial x_1^{\alpha_1}\partial x_2^{\alpha_2}\dots\partial x_n^{\alpha_n}}d_1^{\alpha_1}d_2^{\alpha_2} \dots d_n^{\alpha_n} = D^k f(\boldsymbol{x}_0+\boldsymbol{d}t) \Big |_{t=0},
\end{equation}
be the $k$th derivative of $f(t)$ of equation 6.3 with respect to direction $\boldsymbol{d}$, where $k = 0, 1, \dots, m$, $\boldsymbol{d} = [d_1, d_2, \dots, d_n]^T$, $0\le \alpha_i\le k$ for $i = 1, 2, \dots, n$ and each element of $A$ is a set of $\alpha_i$'s satisfying $\alpha_1+\alpha_2+\dots+\alpha_n=k$. By equation 6.4, it is obvious that definition $9$ implies equation 6.3. Conversely, for example, to dimension $x_1$, letting $\boldsymbol{d}=[1,0,0,\dots,0]^T$, then equation 6.3 results in the continuity of $\partial^{k} f(\boldsymbol{x})/\partial x_1^{\alpha_1}\partial x_2^{\alpha_2}\dots\partial x_n^{\alpha_n}$ with respect to variable $x_1$; and other variables are similar. Thus, if and only if the condition of definition 9 is fulfilled, equation 6.3 holds:
\end{proof}

\begin{dfn}[Directional-derivative hypersurface (hyperplane)]
Given a smooth function $f(\boldsymbol{x}) \in C^m([0, 1]^n)$ and some direction $\boldsymbol{d}$, there exists a $k$th directional derivative at each point $\boldsymbol{x} \in [0, 1]^n$ for some $1\le k \le m$ and all of the derivatives comprise a hypersurface $\mathcal{F}^{(k, f, \boldsymbol{d})}(\boldsymbol{x})$, which is called a $k$th directional-derivative hypersurface of $f(\boldsymbol{x})$ with respect to $\boldsymbol{d}$. If $\mathcal{F}^{(k, f, \boldsymbol{d})}(\boldsymbol{x})$ is an $n-1$-dimensional hyperplane, it is also called a directional-derivative hyperplane.
\end{dfn}

\begin{lem}
Each $\mathcal{F}^{(k, f, \boldsymbol{d})}(\boldsymbol{x})$ for $k = 1, 2, \dots, m$ of definition 14 is smooth with order $m-k$.
\end{lem}
\begin{proof}
By equation 6.4, write
\begin{equation}
D^k(f(t), \boldsymbol{d}, \boldsymbol{x}_0) = \sum_{\{\alpha_i\} \in A} g_{\nu}(\boldsymbol{x})d_1^{\alpha_1}d_2^{\alpha_2} \dots d_n^{\alpha_n}\Big |_{\boldsymbol{x}=\boldsymbol{x}_0+\boldsymbol{d}t},
\end{equation}
where
\begin{equation}
g_{\nu}(\boldsymbol{x}) = \frac{\partial^{k} f(\boldsymbol{x})}{\partial x_1^{\alpha_1}\partial x_2^{\alpha_2}\dots\partial x_n^{\alpha_n}}
\end{equation}
for $\nu = 1, 2, \dots, |A|$. The hypersurface $\mathcal{F}^{(k, f, \boldsymbol{d})}(\boldsymbol{x})$ is made up of all the values of $D^k(f(t), \boldsymbol{d}, \boldsymbol{x}_0)$ for different $t$ and $\boldsymbol{x}_0$, or
\begin{equation}
\mathcal{F}^{(k, f, \boldsymbol{d})}(\boldsymbol{x}) = \sum_{\{\alpha_i\} \in A} g_{\nu}(\boldsymbol{x})d_1^{\alpha_1}d_2^{\alpha_2} \dots d_n^{\alpha_n}.
\end{equation}
Since $f(\boldsymbol{x}) \in C^m([0, 1]^n)$, $g_{\nu}(\boldsymbol{x})$'s is smooth with orders $m-k$ and the conclusion follows by equation 6.7.
\end{proof}

\begin{lem}
Given an $n$-variate polynomial $p(\boldsymbol{x})$ of degree $m$ and arbitrary direction $\boldsymbol{d} \in \mathbb{R}^n$, its $m-1$th directional-derivative hypersurface $\mathcal{F}^{(m-1, p, \boldsymbol{d})}(\boldsymbol{x})$ is an $n-1$-dimensional hyperplane.
\end{lem}
\begin{proof}
This is a special case of equation 6.7 when $f(\boldsymbol{x}) = p(\boldsymbol{x})$ is a multivariate polynomial such that $g_{\nu}(\boldsymbol{x})$ is a linear function of the entries of $\boldsymbol{x} = [x_1, x_2, \dots, x_n]^n$.
\end{proof}

\begin{lem}
Let $R_1$ and $R_2$ be two regions of $\mathbb{R}^n$ and $l = R_1 \cap R_2$ be a knot. Suppose that $p_1(\boldsymbol{x})$ and $p_2(\boldsymbol{x})$ are two $n$-variate polynomials of degree $m$ on $R_1$ and $R_2$, respectively. If to some direction $\boldsymbol{d} = [d_1, d_2, \dots, d_n]^T$, their directional-derivative hypersurfaces are smoothly continuous at $l$ with order $k$ for $0\le k \le m-1$, then this property holds for arbitrary other direction as well.
\end{lem}
\begin{proof}
We use equation 6.7 to express the condition of this lemma as
\begin{equation}
\sum_{\{\alpha_i\} \in A} g^{(1)}_{\nu}(\boldsymbol{x})d_1^{\alpha_1}d_2^{\alpha_2} \dots d_n^{\alpha_n} = \sum_{\{\alpha_i\} \in A} g^{(2)}_{\nu}(\boldsymbol{x})d_1^{\alpha_1}d_2^{\alpha_2} \dots d_n^{\alpha_n}
\end{equation}
for $\boldsymbol{x} \in l$, where $g^{(i)}_{\nu}(\boldsymbol{x}) = \partial^{k} p_i(\boldsymbol{x})/\partial x_1^{\alpha_1}\partial x_2^{\alpha_2}\dots\partial x_n^{\alpha_n}$ for $i = 1, 2$ is also a polynomial. Because equation 6.8 holds for infinitely many points on knot $l$, we have $g^{(1)}_{\nu}(\boldsymbol{x}) = g^{(2)}_{\nu}(\boldsymbol{x})$, regardless the entries $d_i$'s of $\boldsymbol{d}$, which implies this lemma.
\end{proof}

\begin{thm}[Multivariate-spline construction]
Let $H$ be a set of $n-1$-dimensional hyperplanes partitioning $U = [0, 1]^n$ into a set $\mathcal{R}$ of regions. To any function $f(\boldsymbol{x}) \in C^{m}([0,1]^n)$, a spline $s(\boldsymbol{x}) \in \mathfrak{S}_n^m(H; \mathcal{R})$ can be constructed from a piecewise linear approximation to a directional-derivative hypersurface of $f(\boldsymbol{x})$, and $s(\boldsymbol{x})$ can approximate $f(\boldsymbol{x})$ with arbitrary precision, provided that the volume of each region of $\mathcal{R}$ is sufficiently small.
\end{thm}
\begin{proof}
The proof is based on the univariate case of theorem 4. By lemma 7, the directional-derivative hypersurface $\mathcal{F}^{(m-1, f, \boldsymbol{d})}(\boldsymbol{x})$  of equation 6.7 with respect to any direction $\boldsymbol{d}$ is smooth. Then a continuous linear spline $\mathscr{S}(\boldsymbol{x}) \in \mathfrak{S}_n^1(H, \mathcal{R})$ approximating $\mathcal{F}^{(m-1, f, \boldsymbol{d})}(\boldsymbol{x})$ can be constructed (lemma 6 of \citet*{Huang2024}) and the approximation error could be arbitrarily small, provided that the volume of each region of $\mathcal{R}$ is small enough; we use the formula
\begin{equation}
\mathscr{S}(\boldsymbol{x}) \approx \mathcal{F}^{(m-1, f, \boldsymbol{d})}(\boldsymbol{x})
\end{equation}
to represent this meaning.

We can find an $n-1$-dimensional hyperplane $l$ with $U \subset l^+$. To arbitrary point $\boldsymbol{x}_0 \in l$, the ray $\boldsymbol{x}_0 + t\boldsymbol{d}$ can run through $l^+$ as $\boldsymbol{x}_0$ and $t > 0$ both arbitrarily change, that is,
\begin{equation*}
l^+ = \{\boldsymbol{x}_0 + t\boldsymbol{d}: \boldsymbol{x}_0 \in l, t>0\},
\end{equation*}
which contains $U$ as a subset. To a fixed $\boldsymbol{x}_0$, write
\begin{equation*}
S_t := \{t: \big(U \cap (\boldsymbol{x}_0 + t\boldsymbol{d})\big) \ne \emptyset\},
\end{equation*}
whose minimum and maximum elements are denoted by $t_i$ and $t_a$, respectively. Let $\mathcal{L} = U \cap (\boldsymbol{x}_0 + t\boldsymbol{d})$ , which is a line segment; $\mathcal{L}$ passes though some regions of $U$ and a univariate linear spline $g(t)$ defined on $[t_i, t_a]$ can be derived from $\mathscr{S}(\boldsymbol{x})$ of equation 6.9, whose values are approximately the $m-1$th directional derivative of $f(\boldsymbol{x})$ with respect to $\boldsymbol{d}$.

Then we use a method analogous to theorem 4 to integrate $g(t)$ to obtain a smooth spline $G(t)$ approximating $f(\boldsymbol{x})$ on $\mathcal{L}$; and as $\boldsymbol{x}_0$ changes in $l$, $G(t)$ can form a multivariate smooth spline $s(\boldsymbol{x}) \in \mathfrak{S}_n^m(H; \mathcal{R})$ approximating $f(\boldsymbol{x})$. To this purpose, we first need a result that an integral of a directional-derivative hyperplane could be an algebraic hypersurface or a polynomial. We show an example of one region. On a certain region $r$ of $\mathcal{R}$, $g(t)$ could form an $n$-dimensional hyperplane $h(\boldsymbol{x})$ as $\boldsymbol{x}_0$ varies within the associated domain and is the linear function on $r$ approximating the directional-derivative of $f(\boldsymbol{x})$. Let $\mathscr{L}$ be one of the boundaries of $r$, on which the function values and the directional-derivative values of $f(\boldsymbol{x})$ or their approximations via polynomials are given; the two types of values are collectively called ``boundary condition'' in this proof. We want to obtain the expression of the integral of $h(\boldsymbol{x})$ along direction $\boldsymbol{d}$ under a boundary condition.

First see the case of $m=2$ and $n=2$ when $h(\boldsymbol{x})$ is a plane and $\mathscr{L}$ is a line segment. Let $\boldsymbol{y}_0 \in \mathscr{L}$ and write $\boldsymbol{y}_0=\boldsymbol{O}+t_1\boldsymbol{d}_1$, where $\boldsymbol{O}$ is an arbitrary fixed point of $\mathscr{L}$ and $\boldsymbol{d}_1$ is the direction of $\mathscr{L}$. Given a $\boldsymbol{y}_0$, we have
\begin{equation*}
h(\boldsymbol{x})=h(t)=at+b(\boldsymbol{y}_0),
\end{equation*}
where $b(\boldsymbol{y}_0)$ is the function value of $h(\boldsymbol{x})$ at $\boldsymbol{y}_0$. We can write $h(\boldsymbol{x})=at+b(t_1)=h(t,t_1)$, where $b(t_1)$ is a linear function of $t_1$. Thus, $h(t,t_1)=h(\boldsymbol{x})$ is a linear function of $t$ and $t_1$ and $\boldsymbol{t}=[t,t_1]^T$ together with origin point $\boldsymbol{O}$ can be regarded as a new coordinate system derived from an affine transformation of $\boldsymbol{x}$, namely $\boldsymbol{t}=A\boldsymbol{x}+\boldsymbol{O}$. Then we have
\begin{equation}
f(\boldsymbol{x}) = f(t, t_1) \approx p(t,t_1) = \int atdt + b(t_1)t + c(t_1),
\end{equation}
in which
\begin{equation}
c(t_1)= \int a_1t_1dt_1 + b_1t_1 + c_1,
\end{equation}
where the parameters $a_1$ and $b_1$ are from $h(\boldsymbol{O}+t_1\boldsymbol{d}_1)=a_1t_1+b_1$ and the constant $c_1=f(\boldsymbol{O})$ or $c_1 \approx f(\boldsymbol{O})$ when $c_1$ is of a polynomial approximation to $f(\boldsymbol{x})$; thus, $p(t, t_1)$ is a polynomial of variables $t$ and $t_1$. The algebraic expression $p(\boldsymbol{x})$ with respect to variable $\boldsymbol{x}$ can be obtained by the affine transformation from $\boldsymbol{x}$ to $\boldsymbol{t}$ as mentioned above.

The case of arbitrary $n \ge 3$ when $m=2$ is only different in the expressions $b(t_1)$ and $c(t_1)$ and they would be changed into $b(\boldsymbol{t}_r)$ and $c(\boldsymbol{t}_r)$, respectively, where $\boldsymbol{t}_r=[t_1, t_2, \dots, t_{n-1}]^T$. The term $b(\boldsymbol{t}_r)$ is a linear function of the variables in $\boldsymbol{t}_r$, while $c(\boldsymbol{t}_r)$ is a polynomial of degree 2 that can be obtained similarly to equation 6.10. The polynomial $p(\boldsymbol{x})$ can be constructed iteratively on the base of $n=2$ of equation 6.10. For instance, to $n=3$, we first use equation 6.10 to construct $c(\boldsymbol{t}_r)$, then integrate each $g(t)$ to form $p(\boldsymbol{t})$. Analogously, when $n=4$, $c(\boldsymbol{t}_r)$ is obtained by the method of $n=3$, after which the integral of $g(t)$ is operated on the remaining dimension. This process can be repeatedly done until the desired dimensionality is reached.

To $m \ge 3$, the difference lies in the number of integral operations. Based on the discussion above, iterative integrals as equation 6.10 always yield a polynomial, since the integrand is a polynomial of the previous step and the boundary condition newly introduced is also of a polynomial. Note that to a new integral or when $m$ is increased by one, a derivative value or function value of $f(\boldsymbol{x})$ should be used to determine the constant term of the associated polynomial, if necessary.

We now go back to the case of $m=2$ and $n=2$ to complete the spline construction. To produce a smooth spline $s(\boldsymbol{x}) \in \mathfrak{S}_2^2(H; \mathcal{R})$ on all the regions of $U$, use theorem 4 to integrate $g(t)$'s to yield smooth $G(t)$'s for all $\boldsymbol{x}_0 \in l$; to ensure the smooth continuity in the whole input space, the constant term of the first polynomial of each $G(t)$ (as $c_1$ of equation 6.11) should be determined by a polynomial approximating $f(\boldsymbol{x})$, except for one case that an initial value must be introduced from $f(\boldsymbol{x})$.

We can also regard the construction process as being composed of the steps, with each corresponding to a certain $h(\boldsymbol{x})$ that leads to different $h(t)$ and $c(t_1)$. In this view, the expression of each polynomial can be simultaneously obtained. The function value of $f(\boldsymbol{x})$ is used only once for determining the constant term of a polynomial, and other cases are through the polynomials approximating $f(\boldsymbol{x})$.

Generally, a spline $s(\boldsymbol{x})$ for $n \ge 3$ or $m \ge 3$ can also be iteratively constructed on the basis of $n=2$ and $m=2$, as the polynomial construction above. The constructed $s(\boldsymbol{x})$ is smooth with degree $m-1$ with respect to direction $\boldsymbol{d}$; and by lemma 9, $s(\boldsymbol{x})$ is also smooth at any other direction and so $s(\boldsymbol{x}) \in \mathfrak{S}_n^m(H; \mathcal{R})$.

To the approximation error, first note that by the construction method of $\mathscr{S}(\boldsymbol{x})$, $\varepsilon_1 = \max_{\boldsymbol{x}\in U}{|\mathscr{S}(\boldsymbol{x})-\mathcal{F}^{(m-1, f, \boldsymbol{d})}(\boldsymbol{x})|}$ can be arbitrarily small, and so is $\varepsilon_2=\max_{\boldsymbol{x}\in U}|s(\boldsymbol{x})-f(\boldsymbol{x})|$ by the proof of theorem 4. The final approximation error $\varepsilon=(\int_U(s(\boldsymbol{x})-f(\boldsymbol{x}))^2d\boldsymbol{x})^{1/2} < \varepsilon_2$. This completes the proof.
\end{proof}

\subsection{Relations between Adjacent Polynomial Pieces}
\begin{lem}
Let $p_1(\boldsymbol{x})$ and $p_2(\boldsymbol{x})$ be two $n$-variable polynomials of degree $m$ defined on two adjacent regions $R_1$ and $R_2$ of $\mathbb{R}^n$, respectively, with $\dim{\mathcal{L}} = n-1$, where $\mathcal{L} = R_1 \cap R_2$ is part of the $n-1$-dimensional hyperplane separating the two regions. Suppose that $p_1(\boldsymbol{x})$ and $p_2(\boldsymbol{x})$ are smoothly continuous at $\mathcal{L}$ with order $m-1$. Denote by $\boldsymbol{d}$ a direction that is not parallel to $\mathcal{L}$. Then the directional-derivative hyperplanes $h_1(\boldsymbol{x})$ and $h_2(\boldsymbol{x})$, also the $m-1$th directional-derivative hypersurfaces, of $p_1(\boldsymbol{x})$ and $p_2(\boldsymbol{x})$ with respect to $\boldsymbol{d}$ , respectively, are continuous at $\mathcal{L}$.
\end{lem}
\begin{proof}
This conclusion is by lemma 8 and definition 12.
\end{proof}

\begin{thm}[Relations between polynomial pieces]
Given two $n$-variate polynomials $p_1(\boldsymbol{x})$ and $p_2(\boldsymbol{x})$ of order $m$, $p_2(\boldsymbol{x})$ is smoothly continuous with $p_1(\boldsymbol{x})$ at knot $\mathcal{L}$ by order $m-1$, if and only if
\begin{equation}
p_2(\boldsymbol{x}) = p_1(\boldsymbol{x}) + \lambda(\boldsymbol{w}^T\boldsymbol{x} + b)_+^m
\end{equation}
holds for $\boldsymbol{x} \in  \mathcal{L}^+$, where $\lambda$ is a constant and  $\boldsymbol{w}^T\boldsymbol{x} + b = 0$ is the equation of $\mathcal{L}$.
\end{thm}
\begin{proof}
We first prove that equation 6.12 implies the smoothness between $p_1(\boldsymbol{x})$ and $p_2(\boldsymbol{x})$. By equation 6.12, each line $l = \boldsymbol{x}_0 + t\boldsymbol{d}$ penetrating $\mathcal{L}$ yields two univariate polynomials $p_1(l)$ and $p_2(l)$ satisfying $p_2(l) = p_1(l) + \lambda(\boldsymbol{w}^Tl + b)_+^m$. If $\boldsymbol{x}_0 = l \cap \mathcal{L}$, the previous formula can be simplified into $p_2(t) = p_1(t) + \lambda't^m_+$, where $\lambda' = \lambda \boldsymbol{w}^T\boldsymbol{d}$, which means that $p_2(t)$ and $p_2(t)$ are smoothly continuous at $\boldsymbol{x}_0$ by order $m-1$, regardless of direction $\boldsymbol{d}$. Note that different direction $\boldsymbol{d}$ would result in distinct $\lambda'$, but the set of all $\boldsymbol{d}$ shares the same parameter $\lambda$; this fact would be helpful to understand the converse conclusion of this theorem.

The converse is also true. By \citet*{Huang2024}'s theorem 2, if $\mathscr{P}_2(\boldsymbol{x})$ and $\mathscr{P}_1(\boldsymbol{x})$ are two linear functions and continuous at $\mathcal{L}$, then
\begin{equation}
\mathscr{P}_2(\boldsymbol{x}) = \mathscr{P}_1(\boldsymbol{x}) + \alpha(\boldsymbol{w}^T\boldsymbol{x} + b)_+
\end{equation}
for $\boldsymbol{x} \in  \mathcal{L}^+$, where $\alpha$ is a constant. To equation 6.12, let $q_1(\boldsymbol{x})$ and $q_2(\boldsymbol{x})$ be the $m-1$th directional-derivative hyperplanes of $p_1(\boldsymbol{x})$ and $p_2(\boldsymbol{x})$ for some direction $\boldsymbol{d}$, respectively. By lemmas 8 and 10, both $q_1(\boldsymbol{x})$ and $q_2(\boldsymbol{x})$ are $n-1$-dimensional hyperplanes and are continuous at $\mathcal{L}$. Then by equation 6.13, we have
\begin{equation}
q_2(\boldsymbol{x}) = q_1(\boldsymbol{x}) + \beta(\boldsymbol{w}^T\boldsymbol{x} + b)_+
\end{equation}
for $\boldsymbol{x} \in  \mathcal{L}^+$. When $\boldsymbol{x} =  \boldsymbol{x}_0 + t\boldsymbol{d}$ and $\boldsymbol{x}_0 \in \mathcal{L}$, equation 6.14 becomes
\begin{equation}
q_2(t) = q_1(t) + \gamma t_+,
\end{equation}
where $\gamma=\beta \boldsymbol{w}^T\boldsymbol{d}$, in which $q_2(t)$ and $q_1(t)$ comprise a continuous linear spline.

Then by the method of theorem 4, integrate equation 6.15 repeatedly by $m-1$ times and obtain
\begin{equation}
p_2(\boldsymbol{x})=p_2(t) = p_1(t) + \frac{\gamma}{m!}t^{m}_+ = p_1(\boldsymbol{x}) + \frac{\beta}{m!(\boldsymbol{w}^T\boldsymbol{d})^{m-1}}(\boldsymbol{w}^T\boldsymbol{x}+b)^{m}_+
\end{equation}
for $\boldsymbol{x} = \boldsymbol{x}_0 + t\boldsymbol{d}$, which is equivalent to equation 6.12 on line $l=\boldsymbol{x}_0 + t\boldsymbol{d}$, with
\begin{equation}
\lambda = \frac{\beta}{m!(\boldsymbol{w}^T\boldsymbol{d})^{m-1}}.
\end{equation}

Note that equation 6.16 holds for arbitrary $\boldsymbol{x}_0 \in \mathcal{L}^+$ and thus we have proved equation 6.12 for a certain direction $\boldsymbol{d}$. Next step is to prove that the obtained parameter $\lambda$ is invariant with respect to $\boldsymbol{d}$. The parameter $\lambda$ is unique, because different $\lambda$ leads to different $p_2(\boldsymbol{x})$ and it's an one-to-one correspondence when $\boldsymbol{w}^T\boldsymbol{x} + b$ is fixed. This completes the proof.
\end{proof}

\begin{rmk-6}
The geometric meaning of $\lambda$ in equation 6.12 can be derived from $\beta$ of equation 6.14, which was explained in theorem 2 of \citet*{Huang2024}.
\end{rmk-6}

\begin{rmk-6}
In the two-dimensional case, \citet*{Chui1983} gave a more general result than equation 6.12, using Bezout's result (cf. \citet*{Walker1950}) of algebraic geometry. Our proof is different in two aspects. First, the result is regardless of input dimensionality and includes the two-dimensional input as a special case. Second, parameter $\lambda$ has a distinct geometric meaning derived from the construction process.
\end{rmk-6}

\subsection{Zero-Part Error of Units}
\begin{lem}
Denote by $\mathcal{U}$ a generalized sigmoid unit whose activation function is $\phi(\boldsymbol{x}) = \sigma(\boldsymbol{w}^T\boldsymbol{x} + b)$, and by $l_0$ the hyperplane $\boldsymbol{w}^T\boldsymbol{x} + b = 0$. Let $\mathcal{L}$ be a knot $\boldsymbol{w}^T\boldsymbol{x} + b_k = 0$ with $b_k < b$, which divides $U = [0, 1]^n$ into two regions $R_1 = \mathcal{L}^0 \cap U$ and $R_1'=\mathcal{L}^+ \cap U$. Suppose that $\phi(\boldsymbol{x}) < \varepsilon$ for $\boldsymbol{x} \in \mathcal{L}^0$, where $\varepsilon \le \sigma(0)$. Then by scaling $\boldsymbol{w}$ and adjusting $b$, $\phi(\boldsymbol{x})$ can be modified to
\begin{equation}
\phi_{a}(\boldsymbol{x}) = \sigma({\boldsymbol{w}'}^T\boldsymbol{x} + b')=\sigma(\rho(\boldsymbol{w}^T\boldsymbol{x} + b) + \gamma)
\end{equation}
such that the following conclusions hold.
\begin{itemize}
\item[(a)] To any point $\boldsymbol{x} \in \mathcal{L}^0$
\begin{equation}
\lim_{\rho(\gamma) \to +\infty}\phi_{a}(\boldsymbol{x}) = 0,
\end{equation}
where $\rho(\gamma)$ means the change of parameter $\rho$ followed by bias alteration $\gamma$ of equation 6.18, like equation 3.18 of the univariate case.

\item[(b)] Let $l_{k+1}$ be a knot whose equation is $\boldsymbol{w}^T\boldsymbol{x} + b_{k+1} = 0$, with $b_{k+1} > b_k$. If $b_{k+1} - b_k$ is sufficiently small,
\begin{equation}
    \phi_{a}(\boldsymbol{x}) \approx c(\boldsymbol{w}^T\boldsymbol{x} + b_k)_+^m
\end{equation}
for $\boldsymbol{x} \in R_2 = \mathcal{L}^+ \cap \l_{k+1}^0$, where ``$\approx$'' means that $c(\boldsymbol{w}^T\boldsymbol{x} + b_k)_+^m$ could approximate $\phi_a(\boldsymbol{x})$ with arbitrary precision on $R_2$.
\item[(c)] To the parameter $c$ of equation 6.20,
\begin{equation}
\lim_{\rho(\gamma) \to +\infty}c = +\infty,
\end{equation}
and $c$ monotonically increases with $\rho(\gamma)$.
\end{itemize}
\end{lem}
\begin{proof}
The proof is by reducing the multivariate input to the univariate case of lemma 3. Given a line $h = \boldsymbol{x}_0 + t\boldsymbol{d}$ with $\boldsymbol{x}_0 \in l_0$, the activation function $\phi(\boldsymbol{x})$ is reduced to
\begin{equation}
\phi(t) = \sigma(\boldsymbol{w}^T(\boldsymbol{x}_0 + t\boldsymbol{d}) + b) = \sigma(wt),
\end{equation}
where $w = \boldsymbol{w}^T\boldsymbol{d}$, which can be regarded as a univariate activation function of lemma 3. Correspondingly, $h \cap \mathcal{L}$ is a point knot
\begin{equation}
t_k = (b-b_k)/w
\end{equation}
obtained by $\boldsymbol{w}^T(\boldsymbol{x}_0+t\boldsymbol{d})+b_k = 0$, as the counterpart $x_k$ of lemma 3. Similarly, $t_{k+1} = (b-b_{k+1})/w$ determines the point $h \cap l_{k+1}$. Then we use the method of lemma 3 to change $\sigma(wt)$ into
\begin{equation}
\phi_{a}(t) = \sigma(\rho wt + \gamma),
\end{equation}
such that
\begin{equation}
\lim_{\rho(\gamma) \to +\infty}\phi_{a}(t) = 0
\end{equation}
for $t \le t_k$,
\begin{equation}
    \phi_{a}(t) \approx c'(t-t_k)_+^m
\end{equation}
for $t \in (t_k, t_{k+1}]$, and
\begin{equation}
\lim_{\rho(\gamma) \to +\infty}c' = +\infty.
\end{equation}

When $\boldsymbol{x} = \boldsymbol{x}_0 + t\boldsymbol{d}$, equations 6.24 and 6.25 can be expressed as
equations 6.18 and 6.19, respectively. Since $\boldsymbol{w}^T(\boldsymbol{x}_0 + t\boldsymbol{d}) + b_k = w(t-t_k)$, equation 6.26 is equivalent to
\begin{equation}
    \phi_{a}(\boldsymbol{x}) \approx c(\boldsymbol{w}^T\boldsymbol{x} + b_k)_+^m
\end{equation}
for $\boldsymbol{x} = \boldsymbol{x}_0 + t\boldsymbol{d}$ with $t \in (t_k, t_{k+1}]$, where
\begin{equation}
c = c' / w;
\end{equation}
thus,
\begin{equation}
\lim_{\rho(\gamma) \to +\infty}c = +\infty
\end{equation}
by equation 6.27. So we proved this lemma over the line $\boldsymbol{x} = \boldsymbol{x}_0 + t\boldsymbol{d}$.

Note that equations from 6.24 to 6.27 as well as equation 6.29 are all independent of the selection of $\boldsymbol{x}_0$. When $\boldsymbol{x}_0$ runs through $l_0$ and $t$ arbitrarily changes, $p_{k}(\boldsymbol{x}) = c(\boldsymbol{w}^T\boldsymbol{x} + b_k)_+^m$ of equation 6.28 forms an $n$-variate polynomial of degree $m$, which is composed of identical $q_{k}(t) = c'(t-t_k)^m_+$ of equation 6.26 at different $\boldsymbol{x}_0$. The formed $p_{k}(\boldsymbol{x})$ can approximate $\phi_{a}(\boldsymbol{x})$ as precisely as possible on $R_2$ according to equation 6.26, and the approximation-error analysis is similar to the proof of theorem 11. To arbitrary another direction $\boldsymbol{d}'$, the above process can also yield a multivariate polynomial $p_{k}'(\boldsymbol{x})$ and it would be equal to $p_{k}(\boldsymbol{x})$, due to the uniqueness of the polynomial approximating $\phi_{a}(\boldsymbol{x})$ on $R_2$ under arbitrary precision; so the parameter $c$ is unique with respect to different $\boldsymbol{d}$. This completes the proof.

\end{proof}

\begin{thm}
Notations being from lemma 11, suppose that on the two regions $R_1$ and $R_2$, a smooth spline $g(\boldsymbol{x})$ is defined as: $g(\boldsymbol{x}) = 0$ for $\boldsymbol{x} \in R_1$ and $g(\boldsymbol{x}) = a\sigma(\boldsymbol{w}^T\boldsymbol{x} + b_k)^m_+$. Then a two-layer neural network $\mathfrak{N}$ with only one generalized sigmoidal unit can approximate $g(\boldsymbol{x})$ with arbitrary precision.
\end{thm}
\begin{proof}
The proof is similar to that of theorem 5.
\end{proof}

\begin{dfn}[Multivariate global (local) unit]
Given $f(\boldsymbol{x}) \in C^m([0, 1]^n)$, suppose that it is approximated by a two-layer neural network $\mathfrak{N}$ with generalized sigmoidal units, that is,
\begin{equation}
\lVert f(\boldsymbol{x}) - \sum_i\alpha_i\phi_i(\boldsymbol{x}) \rVert_2 < \varepsilon,
\end{equation}
where $\phi_i(\boldsymbol{x}) = \sigma(\boldsymbol{w}^T_i\boldsymbol{x} + b_i)$ is the activation function of the $i$th unit $u_i$ of the hidden layer of $\mathfrak{N}$.

Each unit $u_i$ determines an $n-1$-dimensional hyperplane $\boldsymbol{w}^T_i\boldsymbol{x} + b_i = 0$, denoted by $l_i$. To $u_i$, let $l_i'$ be a hyperplane obtained by reducing the parameter $b_i$ of $l_i$, satisfying $l_i' \parallel  l_i$ and $l_i' \subset l_i^0$. Suppose that the activation function $\phi(\boldsymbol{x})$ of $u_i$ is truncated at $l_i'$ in terms of $\phi(\boldsymbol{x}) = 0$ on $U \cap l_i'^0$; after that, if inequality 6.31 still holds, we call $l_i'$ a potential zero-error hyperplane with respect to $\varepsilon$. Under the truncation operation, to satisfy inequality 6.31, the supremum of $b_i$ leads to a hyperplane $z_i$, which is called a \textsl{zero-error hyperplane} of $u_i$ (or $\phi_i(\boldsymbol{x})$) with respect to $\varepsilon$.

To make the approximation error $\varepsilon$ of inequality 6.31 smaller by adjusting the parameters of the units, if
\begin{equation}
\lim_{\varepsilon \to 0}\dim \{z_i(\varepsilon)^0 \cap U\} = n,
\end{equation}
we say that $u_i$ is a \textsl{local unit}; otherwise, it is a global unit.
\end{dfn}

\subsection{Multiple Expressions of Polynomial Pieces}
Hereafter, to simplify the descriptions, given a partition of $U = [0, 1]^n$ by a set $H = \{l_i : 1\le i\le\Theta\}$ of $n-1$-dimensional hyperplanes, if we say that the units of the hidden layer of a two-layer neural network $\mathfrak{N}$ with generalized sigmoidal units are derived from $H$, it means that to each local unit $u_{\nu}$, $l_{\nu}$ is the zero-error hyperplane of $u_{\nu}$; and based on the results of sections 6.3, the activation function of $u_{\nu}$ on $l_{\nu}^0$ can be regarded as zero. To global unit $u_{\mu}$, $l_{\mu}$ is the original meaning.

The following results not only introduce a parameter-setting technique, but also involve a local property of parameters that is useful in understanding training solutions. In this section, ``$f(\boldsymbol{x})\approx g(\boldsymbol{x})$'' means that $f(\boldsymbol{x})$ can approximate $g(\boldsymbol{x})$ with arbitrary precision.

\begin{figure}[t!] 
\captionsetup{justification=centering}
\centering
\includegraphics[width=3.0in, trim = {3.0cm 2.5cm 2.0cm 1.5cm}, clip]{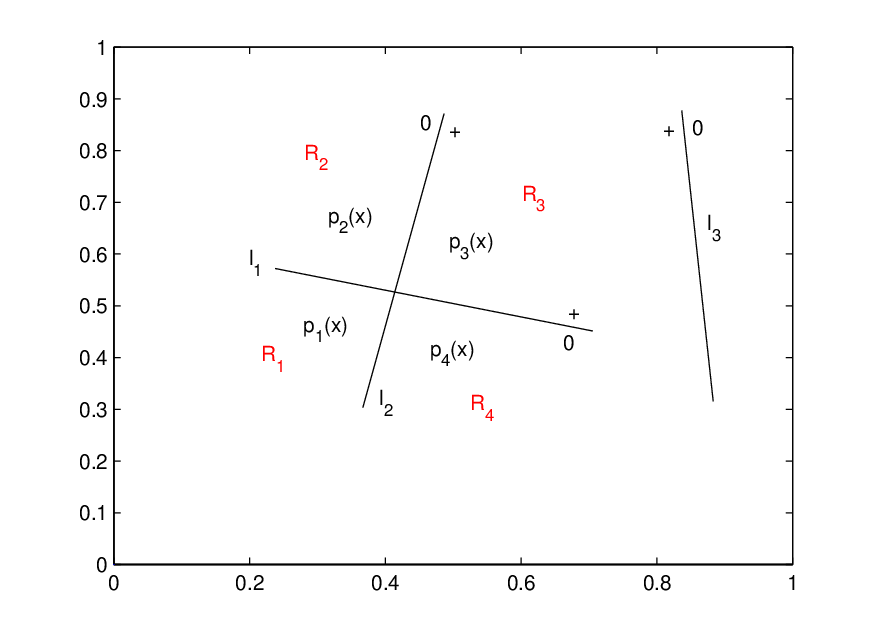}
\caption{Multiple expressions of a polynomial piece.}
\label{Fig.4}
\end{figure}

\begin{lem}
Let $H \subset \mathbb{R}^n$ be a set of $n-1$-dimensional hyperplanes and $\mathcal{R}$ be the regions formed by $H$. Denote by $R_{\nu}$'s for $1 \le \nu \le 4$ the four regions of $\mathcal{R}$, which are separated by $l_1$ and $l_2$ with $l_1 \cap l_2 \ne \emptyset$ and $l_1, l_2 \in H$.
Suppose that $R_i$ is adjacent to $R_{i+1}$ with $R_5 := R_1$, and that $R_1 \subset l_1^0 \cap l_2^0$, $R_2 \subset l_1^+ \cap l_2^0$, $R_3 \subset l_1^+ \cap l_2^+$, $R_4 \subset l_1^0 \cap l_2^+$ (e.g., Figure \ref{Fig.4}).

Suppose that a spline $s(\boldsymbol{x}) \in \mathfrak{S}_n^m(H, \mathcal{R})$ is approximately realized by a two-layer neural network $\mathfrak{N}$ with generalized sigmoidal units, whose hidden-layer units $u_j$'s for $j = 1,2,\dots,\Theta$ are from the hyperplanes of $H$, and that $u_1$ and $u_2$ are local unit, corresponding to $l_1$ and $l_2$ above, respectively. Let $\phi_j(\boldsymbol{x})$ be the activation function of $u_j$, $\lambda_j$ its output weight and $\boldsymbol{w}_j^T\boldsymbol{x} + b_j = 0$ the equation of $l_i$. Denote by $p_i(\boldsymbol{x})$ the polynomial of $s(\boldsymbol{x})$ on $R_i$. Each activation function  $\phi_j(\boldsymbol{x})$ is approximated by a spline $s_{j}(\boldsymbol{x}) \in \mathfrak{S}_n^m(H, \mathcal{R})$ and let $p_{ji}(\boldsymbol{x})$ be the polynomial of $s_j(\boldsymbol{x})$ on $R_i$.

Then $p_3(\boldsymbol{x})$ has two different expressions including
\begin{equation}
p_3(\boldsymbol{x}) \approx p_2(\boldsymbol{x}) + (\lambda_1\alpha_{13}+\lambda_3\alpha_{33})(\boldsymbol{w}_2^T\boldsymbol{x} + b_2)^m_+ + \lambda_2\phi_2(\boldsymbol{x})
\end{equation}
and
\begin{equation}
p_3(\boldsymbol{x}) \approx p_4(\boldsymbol{x}) + (\lambda_2\alpha_{23}+\lambda_3\alpha_{33})(\boldsymbol{w}_1^T\boldsymbol{x} + b_1)^m_+ + \lambda_1\phi_1(\boldsymbol{x}),
\end{equation}
where parameters $\alpha_{\nu\mu}$'s for $1\le \nu,\mu \le 3$ are derived from splines $s_j(\boldsymbol{x})$'s and are the parameter $\lambda$ of equation 6.12 (see the proof below). Equation 6.33 can also be written as
\begin{equation}
p_2(\boldsymbol{x}) \approx p_3(\boldsymbol{x}) - (\lambda_1\alpha_{13}+\lambda_3\alpha_{33})(\boldsymbol{w}_2^T\boldsymbol{x} + b_2)^m_+ - \lambda_2\phi_2(\boldsymbol{x})
\end{equation}
\end{lem}
\begin{proof}
We first see an example of Figure \ref{Fig.4}. To the polynomials on region $R_3$, we have
\begin{equation}
p_{j3}(\boldsymbol{x}) = p_{j2}(\boldsymbol{x}) + \alpha_{j3}(\boldsymbol{w}_2^T\boldsymbol{x} + b_2)^m_+
\end{equation}
for $j = 1, 2, 3$. Since $R_2 \subset l_1^+ \cap l_3^+$,
\begin{equation}
p_2(\boldsymbol{x}) = \lambda_1p_{12}(\boldsymbol{x}) + \lambda_3p_{32}(\boldsymbol{x}).
\end{equation}
Similarly,
\begin{equation}
p_3(\boldsymbol{x}) = \lambda_1p_{13}(\boldsymbol{x}) + \lambda_3p_{33}(\boldsymbol{x}) + \lambda_2p_{23}(\boldsymbol{x})
\end{equation}
due to $R_3 \subset l_1^+ \cap l_3^+ \cap l_2^+$. Because $R_2 \subset l_2^0$ and $R_3 \subset l_2^+$, in equation 6.36,
\begin{equation}
p_{23}(\boldsymbol{x}) = \alpha_{23}(\boldsymbol{w}_2^T\boldsymbol{x}+b_2)^m_+ \approx \phi_2(\boldsymbol{x})
\end{equation}
since $p_{22}(\boldsymbol{x})\approx0$. Equations from 6.36 to 6.39 yield
\begin{equation}
\begin{aligned}
p_3(\boldsymbol{x}) &= p_2(\boldsymbol{x}) + (\lambda_1\alpha_{13}+\lambda_3\alpha_{33}+\lambda_2\alpha_{23})(\boldsymbol{w}_2^T\boldsymbol{x} + b_2)^m_+\\
&\approx p_2(\boldsymbol{x}) + (\lambda_1\alpha_{13}+\lambda_3\alpha_{33})(\boldsymbol{w}_2^T\boldsymbol{x} + b_2)^m_+ + \lambda_2\phi_2(\boldsymbol{x}).
\end{aligned}
\end{equation}

The case of equation 6.34 is similar. This example contains the general principle of this lemma.

\end{proof}

\begin{thm}[Multiple Expressions of a polynomial piece]
Let $g(\boldsymbol{x}) = \sum_{i = 1}^{\Theta}\lambda_i\phi_i(\boldsymbol{x})$ be the output of a two-layer neural network $\mathfrak{N}$ with generalized sigmoidal units, where $\phi_i(\boldsymbol{x})$ is the activation function of unit $u_i$ of the hidden layer; unit $u_i$'s are derived from the set $H = \{l_i: 1\le i\le\Theta\}$ of $n-1$-dimensional hyperplanes, with the equation of $l_i$ bing $\boldsymbol{w}_i^T\boldsymbol{x} + b_i=0$. Suppose that $u_{\tau}$'s for $\tau = 1, 2, \dots, \theta$ with $\theta\le \Theta$ are local unit and the remaining ones are global unit, and that $H'=\{l_{\tau}: 1\le \tau\le \theta\}$. Denote by $\mathcal{R} = \{r_j: 1\le j\le\zeta\}$ the set of the regions formed by $H'$. Each $\phi_i(\boldsymbol{x})$ is approximated by a spline $s_i(\boldsymbol{x}) \in \mathfrak{S}_n^m(H, \mathcal{R})$ with arbitrary precision. Let
\begin{equation}
\mathcal{S}(\boldsymbol{x}) = \sum_{i=1}^{\Theta}\lambda_is_i(\boldsymbol{x}) \approx g(\boldsymbol{x})
\end{equation}
and $p_j(\boldsymbol{x})$ be the polynomial of $\mathcal{S}(\boldsymbol{x})$ on region $r_j$.

Suppose that $r_j$ has $\rho_j$ adjacent regions including $r_{n_{j1}}, r_{n_{j2}}, \dots, r_{n_{j\rho_j}}$ with $\nu=1,2,\dots,\rho_j$ and $1 \le n_{j\nu} \le \zeta$. Each $r_{n_{j\nu}}$ shares a knot $l_{n_{j\nu}'}$ with $r_j$, where $1 \le n_{j\nu}' \le \theta$. Suppose that $r_j \subset \bigcap_{\mu=1}^{\theta_j}l_{k_{j\mu}}^+$, where $1\le k_{j\mu}, \theta_j\le \Theta$. Then based on equation 6.41, $p_j(\boldsymbol{x})$ of $\mathcal{S}(\boldsymbol{x})$ has $\rho_j$ different expressions, each of which is either of the following two possible forms: the first is
\begin{equation}
p_j(\boldsymbol{x})\approx p_{n_{j\nu}}(\boldsymbol{x})+ \sum_{\mu = 1}^{\theta_j}\lambda_{k_{j\mu}}\alpha_{k_{j\mu},j}(\boldsymbol{w}_{n_{j\nu}'}^T\boldsymbol{x}+b_{n_{j\nu}'})^m_+ + \lambda_{n_{j\nu}'}\phi_{n_{j\nu}'}(\boldsymbol{x}),
\end{equation}
where $\alpha_{k_{j\mu},j}$ is the parameter of $s_{k_{j\mu}}(\boldsymbol{x})$ as $\alpha_{13}$ and $\alpha_{33}$ of equation 6.33, if
\begin{equation}
r_j \subset l_{n_{j\nu}'}^+ \ \text{and} \ r_{n_{j\nu}} \subset l_{n_{j\nu}'}^0;
\end{equation}
the second is
\begin{equation}
p_j(\boldsymbol{x}) \approx p_{n_{j\nu}}(\boldsymbol{x}) - \sum_{\mu = 1}^{\theta_j}\lambda_{k_{j\mu}}\alpha_{k_{j\mu},j}(\boldsymbol{w}_{n_{j\nu}'}^T\boldsymbol{x}+b_{n_{j\nu}'})^m_+ - \lambda_{n_{j\nu}'}\phi_{n_{j\nu}'}(\boldsymbol{x}),
\end{equation}
provided that
\begin{equation}
r_j \subset l_{n_{j\nu}'}^0 \ \text{and} \ r_{n_{j\nu}} \subset l_{n_{j\nu}'}^+.
\end{equation}
\end{thm}
\begin{proof}
This theorem is a more general description of lemma 12.
\end{proof}

\subsection{Smooth-Continuity Restriction}
\begin{lem}
Let $H$ be a set of $n-1$-dimensional hyperplanes and $\mathcal{R}$ be the set of the regions of $U = [0, 1]^n$ divided by $H$. To any spline $s(\boldsymbol{x}) \in \mathfrak{S}_n^m(H, \mathcal{R})$, its $m-1$th directional-derivative hypersurface for any direction $\boldsymbol{d}$ is a continuous linear spline $\hat{s}(\boldsymbol{x}) \in \mathfrak{S}_n^1(H, \mathcal{R})$.
\end{lem}
\begin{proof}
Due to lemma 8 and $s(\boldsymbol{x}) \in C^{m-1}([0,1]^n)$, the conclusion follows.
\end{proof}

\begin{lem}
Let $H$ be a set of $n-1$-dimensional hyperplanes partitioning $U=[0,1]^n$ into a set $\mathcal{R} = \{R_j: 1\le j \le \zeta\}$ of regions, and let $\mathfrak{N}$ be a two-layer neural network composed of generalized sigmoidal units that are derived from $H$. Suppose that $R_{\nu}$'s for $1\le \nu\le 4$ satisfy the condition of the four regions of lemma 12. Let $\mathcal{S}(\boldsymbol{x}) \in \mathfrak{S}_n^m(H, \mathcal{R})$ be a smooth spline defined on $\mathcal{R}$, with $p_{\nu}(\boldsymbol{x})$ being the $n$-variate polynomial of degree $m$ on $R_{\nu}$. If $p_2(\boldsymbol{x})$ on $R_2$ and $p_4(\boldsymbol{x})$ on $R_4$ (e.g., Figure \ref{Fig.4}) have been realized by $\mathfrak{N}$, then $p_3(\boldsymbol{x})$ on $R_3$ adjacent to both $R_2$ and $R_4$ is simultaneously implemented.
\end{lem}
\begin{proof}
Denote by $\phi_i(\boldsymbol{x}) \in C^{m}([0, 1]^n)$ the activation function of the $i$th hidden-layer unit $u_i$ of $\mathfrak{N}$. By theorem 11, each $\phi_i(\boldsymbol{x})$ can be approximated by a spline $s_i(\boldsymbol{x}) \ \in \mathfrak{S}_n^{m}(H, \mathcal{R})$ to any desired accuracy. Let $g(\boldsymbol{x}) = \sum_{i = 1}^{\Theta}\lambda_i\phi_i(\boldsymbol{x})$ be the function output by $\mathfrak{N}$, which can be approximated by
\begin{equation}
\mathcal{S}(\boldsymbol{x}) = \sum_{i=1}^{\Theta}\lambda_is_i(\boldsymbol{x})
\end{equation}
as precisely as possible, with $\mathcal{S}(\boldsymbol{x}) \in \mathfrak{S}_n^{m}(H, \mathcal{R})$ as well. Write $\hat{s}_i(\boldsymbol{x})$, which belongs to $\mathfrak{S}_n^1(H, \mathcal{R})$ and is composed of the directional-derivative hyperplanes of $s_i(\boldsymbol{x})$ with respect to some direction $\boldsymbol{d}$. Letting
\begin{equation}
\hat{\mathcal{S}}(\boldsymbol{x}) = \sum_{i=1}^{\Theta}\lambda_i\hat{s_i}(\boldsymbol{x}),
\end{equation}
we have $\hat{\mathcal{S}}(\boldsymbol{x}) \in \mathfrak{S}_n^1(H, \mathcal{R})$. Since the integral of the right side of equation 6.47 with respect to $\boldsymbol{d}$ (as in theorem 12) gives $\mathcal{S}(\boldsymbol{x})$ of equation 6.46, the directional-derivative hyperplanes of $\mathcal{S}(\boldsymbol{x})$ (lemma 13) comprise $\hat{\mathcal{S}}(\boldsymbol{x})$ of equation 6.47.

Let $q_{i\nu}(\boldsymbol{x})$ be the linear function on $R_{\nu}$ of $\hat{s_i}(\boldsymbol{x})$, and $q_{\nu}(\boldsymbol{x})$ be the one of $\hat{\mathcal{S}}(\boldsymbol{x})$. According to the ``boundary-determination principle'' of splines in $\mathfrak{S}_n^1(H, \mathcal{R})$ (see the proof of \citet*{Huang2024}'s lemma 6), $q_{i2}(\boldsymbol{x})$ and $q_{i4}(\boldsymbol{x})$ for $\hat{s_i}(\boldsymbol{x})$ can uniquely determine $q_{i3}(\boldsymbol{x})$ due to the property of continuous piecewise linear functions, and so is the unique determination of $q_{3}(\boldsymbol{x})$ through $q_2(\boldsymbol{x})$ and $q_4(\boldsymbol{x})$ for $\hat{\mathcal{S}}(\boldsymbol{x})$. Thus, if the linear combination of equation 6.47 via weight parameters $\lambda_i$'s are set to produce $q_2(\boldsymbol{x})$ and $q_4(\boldsymbol{x})$, $q_3(\boldsymbol{x})$ is simultaneously implemented. This principle remains invariant under the integral operation of the two-sides of equation 6.47 and this lemma follows.
\end{proof}

The boundary of a region $R$ is a set whose each element is part of a hyperplane forming $R$. For example, in Figure \ref{Fig.4}, the boundary of $R_3$ is $\{l_1\cap R_3, l_2\cap R_3\}$ with two elements.
\begin{thm}[Smooth-continuity restriction]
Denote by $R_i$'s for $i = 1, 2, \dots, \zeta$ the regions of $U = [0, 1]^n$ formed by a set $H = \{l_j: j = 1, 2, \dots, \psi\}$ of $n-1$-dimensional hyperplanes, where $n \ge 2$. Let $\mathcal{R} = \bigcup_iR_i$ and let $\mathcal{S}(\boldsymbol{x}) \in \mathfrak{S}_n^m(H, \mathfrak{N})$ be a spline to be approximated by a two-layer neural network $\mathfrak{N}$ whose hidden-layer units are of generalized sigmoidal type and derived from $H$. To arbitrary region $R_{k} \in \mathcal{R}$ for $1 \le k \le \zeta$, let $\mathcal{L}_{\nu}$ and $\mathcal{L}_{\mu}$ for $1 \le \nu, \mu \le \psi$ be two elements of its boundary satisfying $l_{\nu} \cap l_{\mu} \ne \emptyset$, where $l_{\nu}$ and $l_{\mu}$ are the hyperplanes that $\mathcal{L}_{\nu}$ and $\mathcal{L}_{\mu}$ belong to, respectively. Then, if the part of $\mathcal{S}(\boldsymbol{x})$ on $\mathcal{L}_{\nu} \cup \mathcal{L}_{\mu}$ has been implemented by $\mathfrak{N}$, the polynomial on $R_{k}$ would be simultaneously produced.
\end{thm}
\begin{proof}
This conclusion is due to the proof of lemma 14. To the example of Figure \ref{Fig.4}, this theorem means that if the polynomials on the boundary of $R_3$ formed by $l_1$ and $l_2$ have been realized by $\mathfrak{N}$, the one on $R_3$ is simultaneously implemented.
\end{proof}

\begin{rmk}
Let us compare the smooth-continuity restriction of this theorem with the continuity restriction of a two-layer ReLU network $\mathscr{N}$ ((\citet*{Huang2024}'s theorem 9).  The latter is due to two facts: one is the boundary-determination principle (BDP) mentioned in the proof of lemma 14 and the other is that the output of $\mathscr{N}$ is a continuous piecewise linear function. The former is established by BDP as well as the fact that the approximation of network $\mathfrak{N}$ with generalized sigmoidal units can be reduced to that of continuous piecewise linear functions. Both of the two cases are the combination of geometric property of linear splines with the property of functions output by neural networks.
\end{rmk}

\subsection{Single Strict Partial Order}
Let $H = \{l_1, l_2, \dots, l_\zeta\}$ be a set of $n-1$-dimensional hyperplane satisfying
\begin{equation}
l_1 \prec l_2 \prec \dots \prec l_{\zeta},
\end{equation}
where ``$\prec$'' represents a strict partial order \citep*{Huang2024} of higher-dimensional knots analogous to the ``$<$'' (less than) relation of real numbers. The region $R_i$'s for $i = 1, 2, \dots, \zeta$ are the corresponding ordered regions and form a set
\begin{equation}
\mathcal{R}:=\{R_i: 1\le i\le \zeta\}.
\end{equation}
To each $s(\boldsymbol{x}) \in \mathfrak{S}^m_n(H, \mathcal{R})$ over a single strict partial order, according to theorem 12. we get the recurrence formula
\begin{equation}
s_{\nu}(\boldsymbol{x}) = s_{\nu-1}(\boldsymbol{x}) + \alpha_{\nu-1}(\boldsymbol{w}_{\nu-1}^T\boldsymbol{x}+b_{\nu-1})^m_+
\end{equation}
for $\nu = 2, 3, \dots, \zeta-1$ analogously to the univariate case of equation 3.4, where $s_{\nu}(\boldsymbol{x})$ is the polynomial on $R_{\nu}$. Equation 6.50 is a foundation of solution construction for higher-dimensional input.

The definition of one-sided bases of $\mathfrak{S}^m_n(H, \mathcal{R})$ is similar to that of $\mathfrak{S}^1_n(H, \mathcal{R})$ for linear splines (\citet*{Huang2024}'s definition 10). Let
\begin{equation}
\mathfrak{B}_1 = \{\rho_{j}(\boldsymbol{x}) = \sigma(\boldsymbol{w}_{j}^T\boldsymbol{x} + b_{j}): -\xi \le j \le \zeta, \ \binom{n+m}{m} \le \xi+1\}
\end{equation}
be a set of one-sided bases of $\mathfrak{S}_n^m(H, \mathcal{R})$, in which $\rho_{\nu}(\boldsymbol{x})$'s for $\nu = -\xi, -\xi+1, \dots, 0$ are global bases and the remaining ones are local bases. Correspondingly, the associated units $u_j$'s are also classified into global and local units that are consistent with definition 15.

\begin{thm}[Splines over a single strict partial order]
Any spline $\mathcal{S}(\boldsymbol{x}) \in \mathfrak{S}_n^m(H, \mathcal{R})$ with $\zeta$ linear pieces, where $H$ is from equation 6.48 and $\mathcal{R}$ from equation 6.49, can be approximated as precisely as possible by a two-layer neural network $\mathfrak{N}$ composed of generalized sigmoidal units, in terms of the one-sided bases of equation 6.51, whose hidden layer has
\begin{equation}
\Theta \ge \zeta -1 + \binom{n+m}{m}
\end{equation}
units.
\end{thm}
\begin{proof}
On the basis of the strict partial order ``$\prec$'' of equation 6.48 that is equivalent to ``$<$'', theorems 10, 12 and 13 convert this problem to the univariate case of theorem 6.
\end{proof}

The solution of a spline $\mathcal{S}(\boldsymbol{x}) \in \mathfrak{S}_n^m(H, \mathcal{R})$ over a single strict partial order, including the one of theorem 16 as well as the two-sided case, also has a general framework analogous to the univariate case of theorem 8. To a generalized sigmoidal unit $\mathfrak{U}$ with activation function $\phi(\boldsymbol{x}) = \sigma(\boldsymbol{w}^T\boldsymbol{x}+b) \in C^{m}([0, 1]^n)$, the indicator function $\mathbb{I}: N \to \{0, 1\}$ for $\mathfrak{U}$, where $N = \{1, 2, \dots, \zeta\}$, is defined as
\begin{equation}
\mathbb{I}(i) =
\begin{cases}
\ \ 0, \ \text{if} \ \big(\int_{R_i}\phi(\boldsymbol{x})^2d\boldsymbol{x}\big)^{1/2} < \varepsilon \\
\ \ 1, \ \ \ \ \ \ \ \ \ \ \ \text{otherwise}
\end{cases}
\end{equation}
for $i = 1, 2, \dots, \zeta$. To each local unit $u_j$'s of equation 6.51, let $\mathbb{I}_{j}$ be its indicator function. As the counterpart of equation 4.20, write
\begin{equation}
\mathscr{A} = \begin{pmat}({})
\mathcal{A} & & & \mathcal{B} \cr
\mathcal{C} & & & \mathcal{D} \cr
\end{pmat},
\end{equation}
in which $\mathscr{A} = \mathcal{W}(u_i(\boldsymbol{x}_0))^T$ of equation 5.10, while $\mathcal{B}$, $\mathcal{C}$ and $\mathcal{D}$ can be obtained similarly to those of equation 4.20, respectively. A two-layer neural network with generalized sigmoidal units have a solution of $\mathcal{S}(\boldsymbol{x})$ if and only if the rank of
\begin{equation}
\mathscr{A}\boldsymbol{\lambda} = \boldsymbol{b}
\end{equation}
is
\begin{equation}
\gamma = \zeta-1+\binom{n+m}{m}.
\end{equation}

\subsection{Multivariate Universal Approximation}
The concepts of ``standard partition'' and ``universal global hyperplane'' were defined in \cite*{Huang2024}. Intuitively, when the input is two-dimensional, a standard partition is a set of regions that appear to be formed by horizontal and vertical lines as in Figure \ref{Fig.5}; a universal global hyperplane $l$ of $U=[0,1]^n$ means that $U \subset l^+$.

\begin{thm}[Multivariate-spline implementation]
Let $H$ be a set of $n-1$-dimensional hyperplanes of $\mathbb{R}^n$ for $n \ge 2$ and $H$ forms a standard partition of $U=[0,1]^n$ (e.g., Figure \ref{Fig.5}), with $\mathcal{R}$ the set of the divided regions. Denote by $\mathfrak{N}$ a two-layer neural network composed of generalized sigmoidal units whose activation function belongs to $C^{m}([0,1]^n)$. The local units of the hidden layer of $\mathfrak{N}$ are from $H$ and the global units are from at least $\binom{n+m}{m}$ universal global hyperplanes of $U$. Then any spline $\mathcal{S}(\boldsymbol{x}) \in \mathfrak{S}_n^m(H, \mathcal{R})$ can be approximated by $\mathfrak{N}$ with arbitrary precision. If $\mathcal{S}(\boldsymbol{x})$ has $\zeta$ polynomial pieces, the number of the units of the hidden layer of $\mathfrak{N}$ satisfies
\begin{equation}
\Theta \ge ({\zeta}^{1/n}-1)n + \binom{n+m}{m}.
\end{equation}
\end{thm}
\begin{proof}
In combination with smooth-continuity restriction of theorem 15, the proof is analogous to \citet*{Huang2024}'s lemma 10.
\end{proof}

\begin{rmk}
When the input is two-dimensional, under a standard partition used in \citet*{Huang2024}'s theorem 10 or later proposition 8 of this paper, the polynomial $p_{ij}(x, y)$ on region $R_{ij}$ of $\mathcal{S}(\boldsymbol{x})$ can be expressed as
\begin{equation}
p_{ij}(x, y) = p_{11}(x, y) + \sum_{\nu = 1}^{i-1}\alpha_{\nu}(x - x_{\nu})^m_+ + \sum_{\mu = 1}^{j-1}\lambda_{\mu}(y - y_{\mu})^m_+,
\end{equation}
which is similar to \citet*{Chui1983}'s equation 3.1. Our difference is twofold. First, the construction method is applicable to arbitrary input dimensionality $n \ge 2$ and not restricted to the two-dimensional case of equation 6.58. Second, we revealed a fundamental principle (smooth-continuity restriction of theorem 15) underlying equation 6.58, also regardless of input dimensionality, which is useful in both finding a solution and interpreting its meaning.
\end{rmk}

The boundary of a partition can be found in \citet*{Huang2024}'s definition 23. As an example in Figure \ref{Fig.5}, the set of the red rectangles is the boundary of the partition of $U$.
\begin{cl}[Boundary-determination principle]
Notations from theorem 17, suppose that $H$ forms $n$ strict partial orders $\mathscr{P}_i = l_{i1} \prec l_{i2} \prec \dots \prec l_{i,M_i-1}$. Then, if the polynomial pieces of $\mathcal{S}(\boldsymbol{x})$ on the boundary $\mathcal{B}$ of $\mathcal{R}$ have been realized by the two-layer network $\mathfrak{N}$, the remaining ones on $\mathcal{R} - \mathcal{B}$ could be simultaneously produced by $\mathfrak{N}$. When $M_i = M$, $|\mathcal{R}| = M^n$ and $|\mathcal{B}| = (M-1)n + 1$, the solution for $(M-1)n + 1$ regions could completely determine $\mathcal{S}(\boldsymbol{x})$ for all of the $M^n$ regions.
\end{cl}
\begin{proof}
This corollary is a direct consequence of the proof of theorem 17.
\end{proof}

\begin{thm}[Multivariate universal approximation]
Notations being from theorem 17, any function $f(\boldsymbol{x}) \in C^{m}([0,1]^n)$ can be approximated by $\mathfrak{N}$ to arbitrarily desired accuracy, in terms of the implementation of a spline $\mathcal{S}(\boldsymbol{x}) \in \mathfrak{S}_n^m(H, \mathcal{R})$ approximating $f(\boldsymbol{x})$. To achieve a certain approximation error, if $\zeta$ polynomial pieces of $\mathcal{S}(\boldsymbol{x})$ are required, the number of the units of the hidden layer of $\mathfrak{N}$ must fulfil inequality 6.57.
\end{thm}
\begin{proof}
The proof is formally the same as that of \citet*{Huang2024}'s theorem 10. The preceding results of section 6 provided all the prerequisites to prove this theorem.
\end{proof}

\subsection{Two-Sided Bases}
To two-sided solutions over a single strict partial order, equation 6.55 of section 6.6 has provided a general principe; this section investigates the more complicated case for multiple strict partial orders. Given a unit $\mathcal{U}$ with activation function $\phi(\boldsymbol{x})=\sigma(\boldsymbol{w}^T\boldsymbol{x}+b)$, if $\phi(\boldsymbol{x})$ is changed into $\phi(\boldsymbol{x})=\sigma(-\boldsymbol{w}^T\boldsymbol{x}-b)$, we call the associated unit a negative form of $\mathcal{U}$, denoted by $-\mathcal{U}$. Correspondingly, the hyperplane $-l$ with equation $-\boldsymbol{w}^T\boldsymbol{x}-b=0$ is the negative form of $l$ whose equation is $\boldsymbol{w}^T\boldsymbol{x}+b=0$.

\begin{dfn}[Two-Sided Bases]
Under the notations of section 6.6, to arbitrary local unit $u_{j}$ for $j = 1, 2, \dots, \zeta$ from the one-sided bases $\mathfrak{B}_1$ of equation 6.51, add a negative unit $-u_i$ of $u_i$ or the original $u_i$ is substituted by $-u_i$; we call this process a negative-unit operation. After at least one negative-unit operation, if the rank of the corresponding matrix $\mathscr{A}$ of equation 6.54 is $\gamma$ of equation 6.56, all the activation functions are called a set of the two-sided bases of $\mathfrak{S}_n^m(H, \mathcal{R})$.
\end{dfn}

\begin{figure}[t!]
\captionsetup{justification=centering}
\centering
\includegraphics[width=2.3in, trim = {3.7cm 1.9cm 2.8cm 1.0cm}, clip]{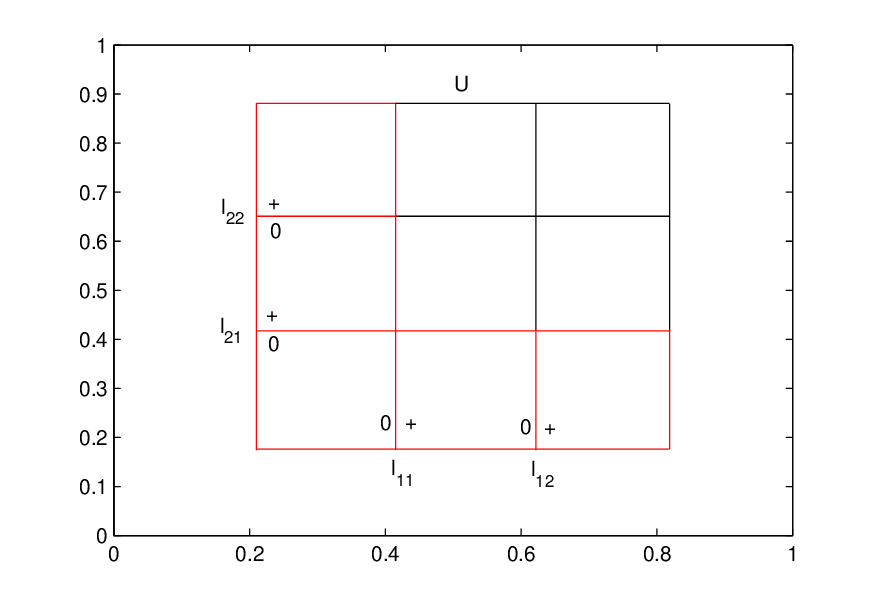}
\caption{A standard partition.}
\label{Fig.5}
\end{figure}

\begin{prp}
Denote by $H = \{l_{ij}: 1 \le i \le n, 1 \le  j \le M_i-1\}$ a set of $n-1$-dimensional hyperplanes of $\mathbb{R}^n$, with the equation of $l_{ij}$ being $x_i = j/M_i$. The set $H$ forms a standard partition of $U = [0, 1]^n$ and $n$ strict partial orders including $\mathscr{P}_i = l_{i1} \prec l_{i2} \prec \dots \prec l_{i,M_i-1}$ for all $i$. Let $\mathcal{R}$ be the set of all the ordered regions of $\mathscr{P}_i$'s, and $\mathfrak{N}$ be a two-layer neural network with generalized sigmoidal units, whose local and global units are from $H$ and at least $\binom{n+m}{m}$ universal global hyperplanes, respectively. Then by theorem 17, any $\mathcal{S}(\boldsymbol{x}) \in \mathfrak{S}_n^m(H, \mathcal{R})$ can be implemented by $\mathfrak{N}$ with arbitrary precision. If at least one of the units of some $\mathscr{P}_i$ is changed into its negative form and if each alteration results in two-sided bases of $\mathscr{P}_i$, a solution of $\mathcal{S}(\boldsymbol{x})$ via network $\mathfrak{N}$ still exists by resetting the output weights.
\end{prp}
\begin{proof}
Before modification, $\mathcal{S}(\boldsymbol{x})$ is constructed by theorem 17. The principle of this theorem can be illustrated by the example of Figure \ref{Fig.5}. In Figure \ref{Fig.5}, there are two strict partial orders $\mathscr{P}_1 = l_{11} \prec l_{12}$ and $\mathscr{P}_2 = l_{21} \prec l_{22}$; lines $l_{ij}$ for $i, j = 1, 2$ form a standard partition of $U = [0, 1]^2$.

A typical feature of $\mathscr{P}_1$ and $\mathscr{P}_2$ is that any line of $\mathscr{P}_1$ cannot influence the ordered regions of $\mathscr{P}_2$ in terms of positive outputs, and vice verse; on the other hand, this feature implies that any negative line of $\mathscr{P}_1$ (or $\mathscr{P}_2$) would contribute to a global influence on $\mathscr{P}_2$ (or $\mathscr{P}_1$), such that $\mathscr{P}_2$ (or $\mathscr{P}_1$) can reformulate its parameters to produce the original function without disturbing $\mathscr{P}_1$ (or $\mathscr{P}_2$).

Thus, if we change one of the lines, $l_{21}$, say, of $\mathscr{P}_2$ into its negative form, the parameter readjusting of $\mathscr{P}_1$ can compensate this disturbance and restore the original output, without influencing $\mathscr{P}_2$ simultaneously. So if the altered units of $\mathscr{P}_2$ forms a set of two-sided bases, the output of $\mathscr{P}_2$ can also remain the same. By smooth-continuity restriction, a solution of $\mathcal{S}(\boldsymbol{x})$ still exists. The principle underlying the above example is related to an order of multiple strict partial orders and can be founded in \citet*{Huang2024}'s theorem 7.
\end{proof}

\begin{thm}[Local solution for two-sided bases]
Let $H$ be a set of $n-1$-dimensional hyperplanes of $\mathbb{R}^n$ and $\mathfrak{N}$ a two-layer neural network composed of generalized sigmoidal units whose hidden-layer units are from $H$. Suppose that $H' \subset H$ doesn't form a strict partial order but can be modified to one (denoted by $\mathscr{P}$) by changing some of its hyperplanes into their negative form. Denote by $\mathcal{R}'$ the set of ordered regions of $\mathscr{P}$. Suppose that $\mathfrak{N}$ generates a spline $\mathcal{S}'(\boldsymbol{x}) \in \mathfrak{S}_n^m(H', \mathcal{R}')$. Then the output weights of the units of $H'$ for the production of $\mathcal{S}'(\boldsymbol{x})$ must obey the rule of two-sided bases over $\mathscr{P}$ in terms of equation 6.55.
\end{thm}
\begin{proof}
The influence (if any) of other hyperplanes on $\mathcal{R}'$ can be regarded as that of global ones for $\mathcal{R}'$; and thus when only considering $\mathscr{P}$, the principle of two-sided bases must be satisfied.
\end{proof}

\subsection{Generalized Tanh-Unit Case}
Under the principle similar to corollary 1, when the input-dimensionality $n\ge2$, the preceding results of section 6 for two-layer neural network with generalized sigmoid units are all applicable to the generalized tanh-unit case, except that the number of the units required should be increased by at least 1.

\section{Explanation of Training Solutions}
This section uses experiments to demonstrate that our theory grasps the main principle of solutions of two-layer neural networks obtained by the back-propagation algorithm. The examples to be given can also be regarded as the phenomena predicted by the theory.

\subsection{Preliminaries}
Since the training process deals with discrete data points, we use the following method to check whether a unit is local or global, which can be regarded as a discrete version of definition 5. Suppose that
\begin{equation}
\big\{\sum_{i}\big(f(x_i) - \sum_{j}\lambda_j\phi_j(x_i)\big)^2\big\}^{1/2} = \varepsilon.
\end{equation}
To each unit $u_{\nu}$ whose activation function is $\phi_{\nu}(x)=\sigma(w_{\nu}x+b_{\nu})$, when $w_{\nu} > 0$, by truncation operation at $0< x_k<1$, equation 7.1 becomes
\begin{equation}
\big\{\sum_{i:x_i\le x_k}\big(f(x_i) - \sum_{j \ne \nu}\lambda_j\phi_j(x_i)\big)^2 + \sum_{i:x_i>x_k}\big(f(x_i) - \sum_{j}\lambda_j\phi_j(x_i)\big)^2\big\}^{1/2} = \varepsilon'.
\end{equation}
The zero-error point of $u_{\nu}$ is determined by 
\begin{equation}
z_{\nu0} = \sup\{x_k: |\varepsilon' - \varepsilon| < \gamma_1\varepsilon\},
\end{equation}
where $0 < \gamma_1 < 1$ is a threshold.

Corresponding to equation 3.57, we check if
\begin{equation}
|\alpha_{\nu}\phi_{\nu}(0)| > \gamma_2\varepsilon,
\end{equation}
where $0 < \gamma_2 < 1$, to exclude the case that as $\varepsilon \to 0$, a unit reduces its truncation error by decreasing $x_k$ until $[0, x_k]$ shrinks to a point. If inequality 7.4 is satisfied, it is not a local unit. Thus, if $0 < z_{\nu0} < 1$ and $|\alpha_{\nu}\phi_{\nu}(0)| \le \gamma_2\varepsilon$, $u_{\nu}$ is a local unit; if $z_{\nu0}\le 0$ or $z_{\nu0}\ge 1$ or $|\alpha_{\nu}\phi_{\nu}(0)|>\gamma_2\varepsilon$, it is a global unit.

If $w_{\nu} < 0$, equations from 7.2 to 7.4 are modified to
\begin{equation}
\big\{\sum_{i:x_i\ge x_k}\big(f(x_i) - \sum_{j \ne \nu}\lambda_j\phi_j(x_i)\big)^2 + \sum_{i:x_i< x_k}\big(f(x_i) - \sum_{j}\lambda_j\phi_j(x_i)\big)^2\big\}^{1/2} = \varepsilon',
\end{equation}
\begin{equation}
z_{\nu0} = \inf\{x_k: |\varepsilon' - \varepsilon| < \gamma_1\varepsilon\},
\end{equation}
and
\begin{equation}
|\alpha_{\nu}\phi_{\nu}(1)| > \gamma_2\varepsilon,
\end{equation}
respectively. Equation 7.1 is a local approximation provided that all $u_{j}$'s are global unit; otherwise, it is a global approximation with local units.

We give a criterion to determine whether a unit is inactivated. Let
\begin{equation}
\big\{\big(\sum_{i}\big(f(x_i)-\sum_{j \ne \nu}\lambda_j\phi_j(x_i)\big)^2\big\}^{1/2} = \varepsilon''.
\end{equation}
If
\begin{equation}
|\varepsilon'' - \varepsilon| < \gamma_3\varepsilon,
\end{equation}
where $0 < \gamma_3 < 1$, $u_{\nu}$ is inactivated.

When the input-dimensionality $n\ge2$, equation 7.1 becomes
\begin{equation}
\big\{\sum_{i}\big(f(\boldsymbol{x}_i)-\sum_{j}\lambda_j\phi_j(\boldsymbol{x}_i)\big)^2\big\}^{1/2} = \varepsilon.
\end{equation}
Each unit $u_j$ leads to an $n-1$-dimensional hyperplane $\boldsymbol{w}_j^Tx + b_j = 0$. To the truncation operation of $u_{\nu}$ on $L_{\nu}^0$ whose equation is $\boldsymbol{w}_{\nu}^Tx + b_{\nu}' = 0$, where $b_{\nu}'<b_{\nu}$, the higher-dimensional counterpart of equation 7.2 is
\begin{equation}
\big\{\sum_{i:\boldsymbol{x}_i \in L_{\nu}^0}\big(f(\boldsymbol{x}_i)-\sum_{j \ne \nu}\lambda_j\phi_j(\boldsymbol{x}_i)\big)^2 + \sum_{i:\boldsymbol{x}_i \in L_{\nu}^+}\big(f(\boldsymbol{x}_i)-\sum_{j}\lambda_j\phi_j(\boldsymbol{x}_i)\big)^2\big\}^{1/2} = \varepsilon'.
\end{equation}
The parameter $b_{\nu0}$ of the zero-error hyperplane is determined by
\begin{equation}
b_{0\nu} = \sup\{b_{\nu}': |\varepsilon' - \varepsilon| < \gamma_1\varepsilon\}.
\end{equation}

To the condition of equation 6.32, we make an analogy with inequality 7.4 as follows. Let
\begin{equation}
\mathcal{B}_{\nu} = \inf\{b_{\nu}': U \cap L_{\nu}^0 \ne \emptyset\}
\end{equation}
and $\mathscr{L}_{\nu}$ be the hyperplane of equation $\boldsymbol{w}_{\nu}^Tx + \mathcal{B}_{\nu} = 0$. Then use
\begin{equation}
\sum_{x_i \in S}|\alpha_{\nu}\phi_{\nu}(x_i)| > \gamma_2\varepsilon,
\end{equation}
where $S = U \cap \mathscr{L}_{\nu}$, to play the role as inequality 7.4.

The difference between
\begin{equation}
\big\{\sum_{i}\big(f(\boldsymbol{x}_i) - \sum_{j \ne \nu}\alpha_j\phi_j(\boldsymbol{x}_i)\big)^2\big\}^{1/2} = \varepsilon''.
\end{equation}
and $\varepsilon$ of equation 7.10 can determine whether a unit is inactivated as inequality 7.9.

\subsection{One-Dimensional Input}
\begin{figure}[htp]
\captionsetup{justification=centering}
\centering
\subfloat[Example 1]{\includegraphics[width=2.9in, trim = {1.1cm 0.7cm 0.9cm 0.1cm}, clip]{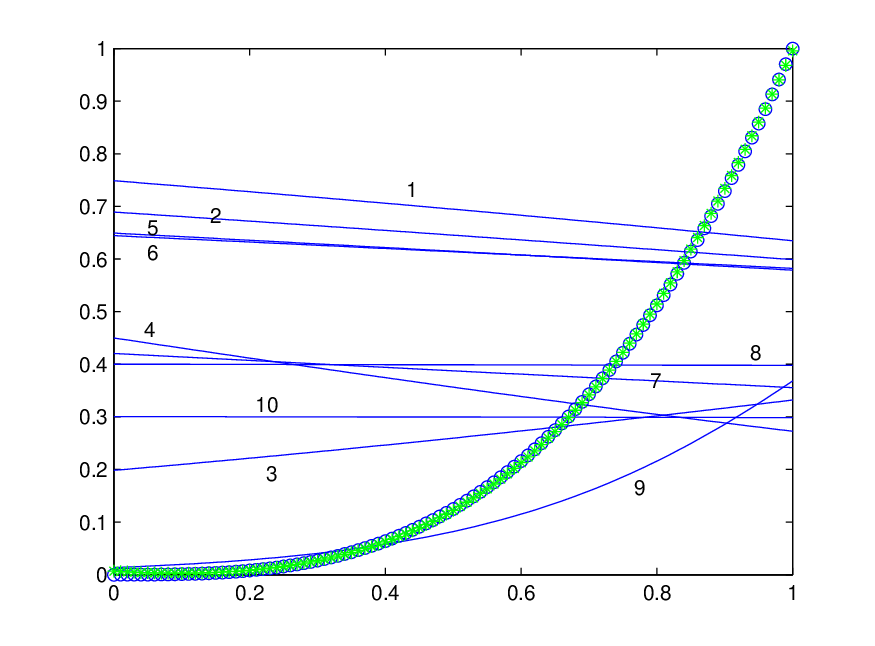}}
\subfloat[Example 2]{\includegraphics[width=2.9in, trim = {1.1cm 0.7cm 0.9cm 0.1cm}, clip]{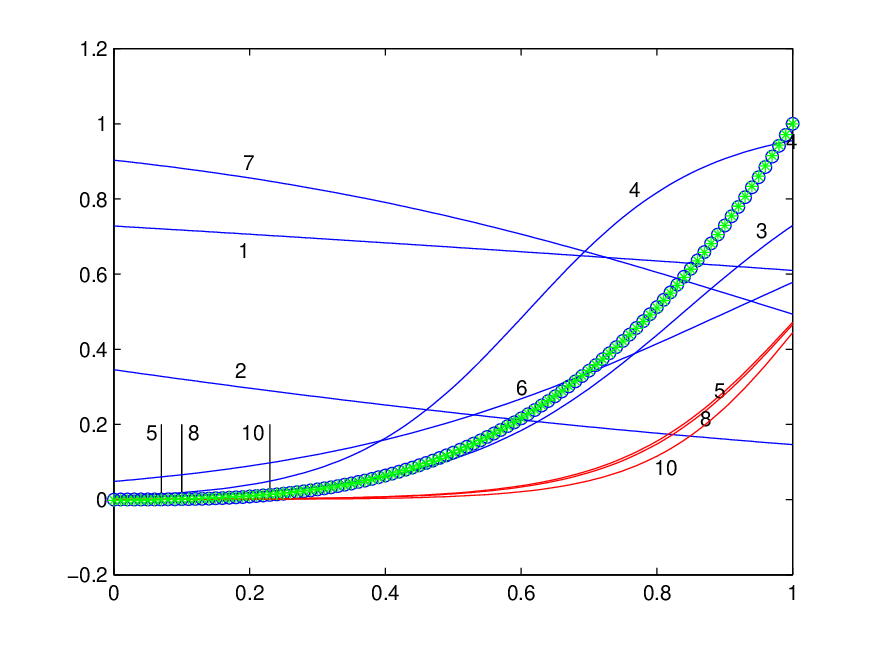}} \\
\subfloat[Example 3]{\includegraphics[width=2.9in, trim = {1.1cm 0.7cm 0.9cm 0.1cm}, clip]{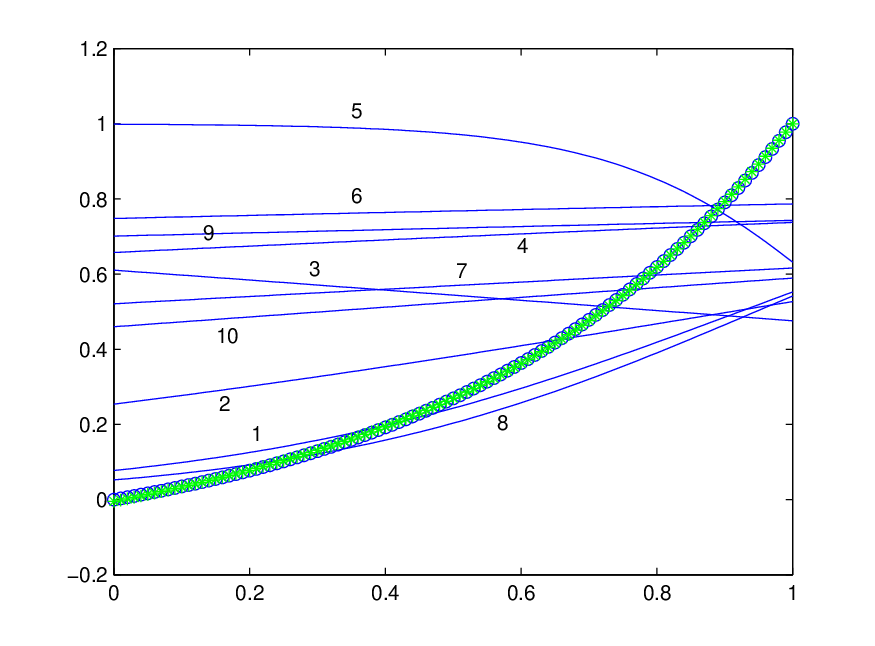}}
\subfloat[Example 4]{\includegraphics[width=2.9in, trim = {1.1cm 0.7cm 0.9cm 0.1cm}, clip]{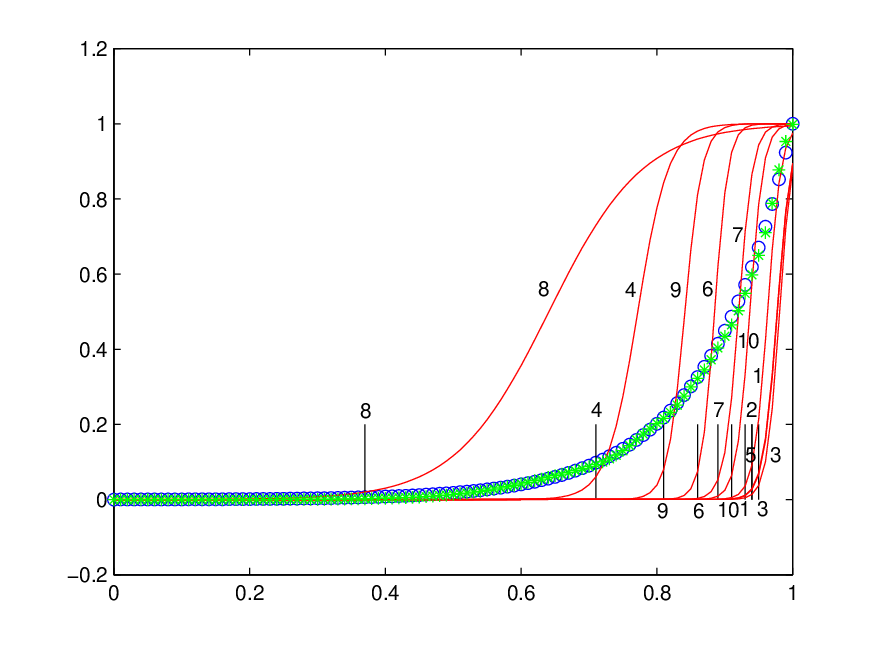}} \\
\subfloat[Example 5]{\includegraphics[width=2.9in, trim = {1.1cm 0.7cm 0.9cm 0.1cm}, clip]{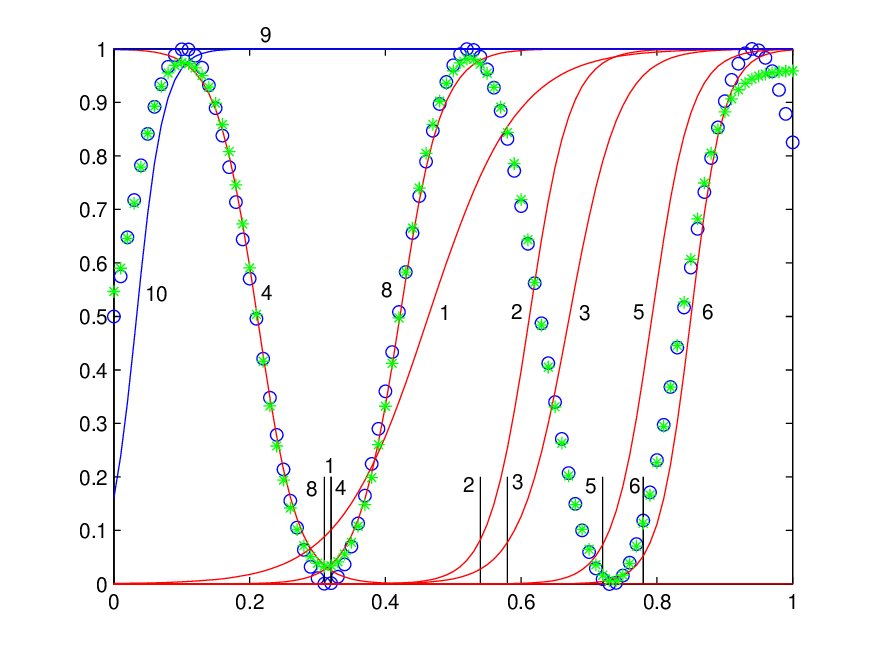}}
\subfloat[Example 6]{\includegraphics[width=2.9in, trim = {1.1cm 0.7cm 0.9cm 0.1cm}, clip]{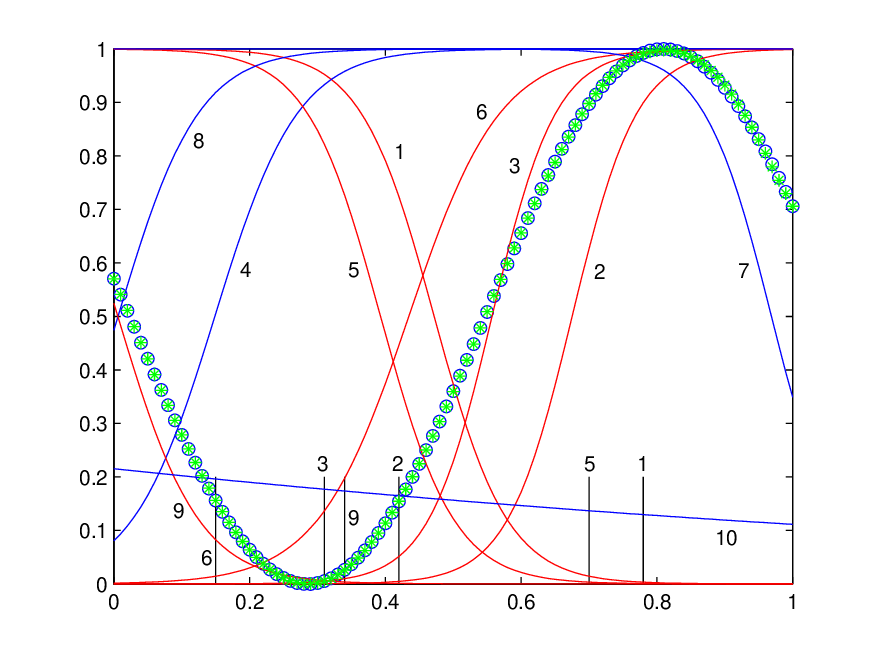}}
\caption{Solution explanation for one-dimensional input.}
\label{Fig.6}
\end{figure}

In practice, equation 7.3 is used in terms of $z_{\nu0} = \max\{x_k: |\varepsilon'-\varepsilon| < \gamma_1\varepsilon\}$ by discretizing $x_k \in [0, 1]$ in step length 0.01, and equation 7.6 is similar. The activation function of hidden-layer units of a two-layer neural network $\mathfrak{N}$ to be trained is the logistic sigmoid type $\sigma(x) = 1 / (1 + e^{-x})$; $\mathfrak{N}$ is sometimes called a logistical sigmoidal neural network in the rest of section 7.

In Figure \ref{Fig.6}a, the circle points are from the discretization of a continuous function $y = x^3 + 3$ on $[0, 1]$ with step length $0.01$, and are fitted by $\mathfrak{N}$ with $\Theta = 10$ units in the hidden layer. In order to show $y(x)$ and the activation functions on the same graph, the former is normalized according to the latter by scaling its function values to have maximum 1. The green asterisks are produced by $\mathfrak{N}$. The weights and biases of $\mathfrak{N}$ are randomly initialized by the uniform distribution $U(-1, 1)$ and are optimized by the back propagation algorithm with the learning rate $c = 0.05$ and the number of training iterations $N = 5000$. Set $\gamma_1 = \gamma_2 = \gamma_3 = 0.01$.

If a unit is a local one, a black line segment will be placed at its zero-error point (such as the examples of Figure \ref{Fig.6}b) and the corresponding activation-function curve is depicted by red colors; otherwise it is a global unit represented by blue curves. The number nearby a line or curve is the index of the unit. In Figure \ref{Fig.6}d, the nearly identical units $u_2$ and $u_5$ are too close to $u_3$ such that only $u_3$ is labeled; the zero-error points of $u_2$ and $u_5$ are also almost the same. Except for the fitted functions and possibly different parameter settings, the above descriptions are applicable to the remaining subfigures.

Note that there exist almost identical activation functions but with different zero-error points; the reason is that their output weights are different, while the zero-error point is determined by both the activation function and its output weight. The following items are the explanation of experimental results.

\begin{itemize}
\item[\rm{\Romannum{1}}.] \textbf{Local approximation}. Since there's no local unit in Figure \ref{Fig.6}a, it is a local approximation, and similarly for Figure \ref{Fig.6}c for function $y = e^{2x}$.

\item[\rm{\Romannum{2}}.] \textbf{Global approximation}. In Figure \ref{Fig.6}b the function is $y = 32x^3 + 3$ and the parameter settings are the same as those of Figure \ref{Fig.6}a except for $\gamma_3 = 0.05$. There are 6 global units and 3 local units, while $u_9$ is inactivated. The function of Figure \ref{Fig.6}d is $y = e^{8x}$ with parameters $c = 0.0001$,$N = 5000$, $\gamma_1 = \gamma_2 = 0.0001$ and $\gamma_3 = 0.01$. There's no global unit and the function values on the leftmost part of $[0, 1]$ are approximately zero compared to the larger ones on the right. Figure \ref{Fig.6}d clearly demonstrates the pattern of global approximations.

\item[\rm{\Romannum{3}}.] \textbf{Two-sided bases}. Figure \ref{Fig.6}e is of a two-sided solution that can be explained by proposition 2, in which $y = 30(\sin15x + 1)$, $\Theta = 10$, $c = 0.001$, $N = 10000$, $\gamma_1 = 0.01, \gamma_2 = 0.001$, and $\gamma_3 = 0.05$. There exist 7 local units, 2 global units and 1 inactivated unit. Compared with Figure \ref{Fig.3} of proposition 2, Figure \ref{Fig.6}e has the similar pattern that positive unit $u_1$ and negative unit $u_4$ share the same zero-error point $x = 0.32$. In Figure \ref{Fig.6}f, $y = 30\sin(6x + 3) + 3$, $c = 0.01$, $N = 5000$ and other parameters are set as Figure \ref{Fig.6}e. We didn't find a theoretical solution for the pattern of Figure \ref{Fig.6}f, but it can be included in the framework of theorem 8 and would have a solution provided the condition of theorem 8 is satisfied.

\end{itemize}

\subsection{Two-Dimensional Input}
\begin{figure}[htp]
\captionsetup{justification=centering}
\centering
\subfloat[Example 1: data-fitting effect.]{\includegraphics[width=2.8in, trim = {1.2cm 0.7cm 0.9cm 0.1cm}, clip]{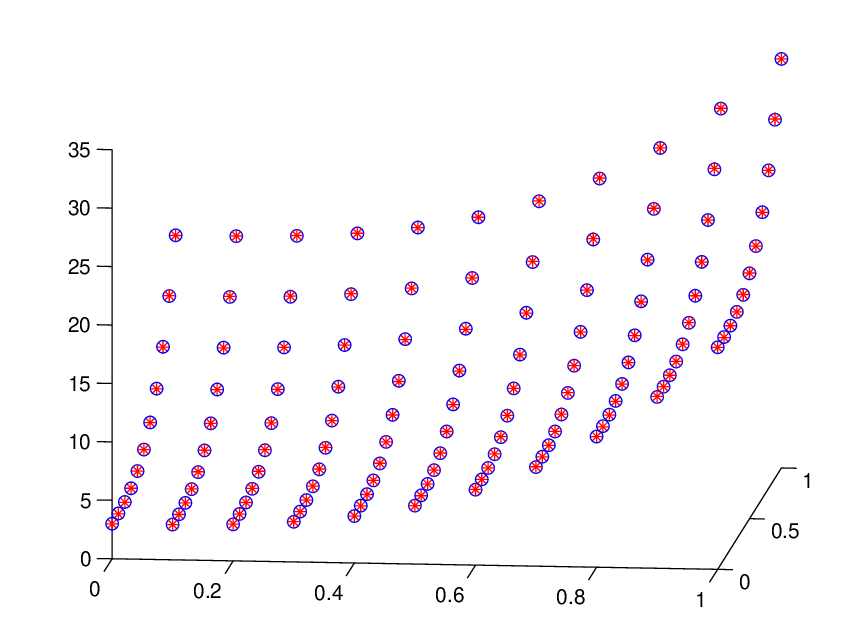}} \ \
\subfloat[Example 1: solution of (a).]{\includegraphics[width=2.8in, trim = {1.2cm 0.7cm 1.0cm 0.1cm}, clip]{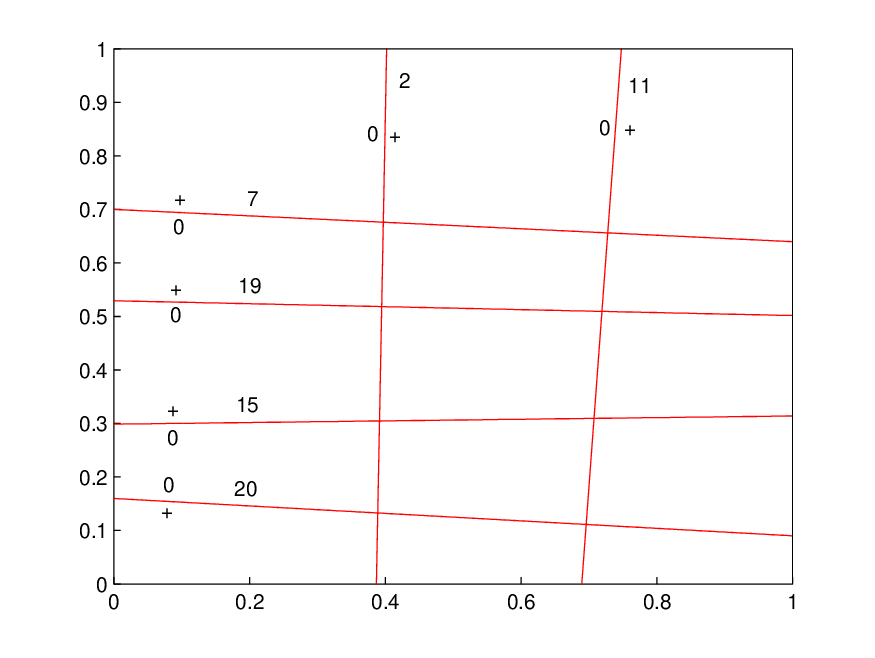}} \\
\subfloat[Example 2: data-fitting effect.]{\includegraphics[width=2.8in, trim = {1.1cm 0.7cm 0.9cm 0.1cm}, clip]{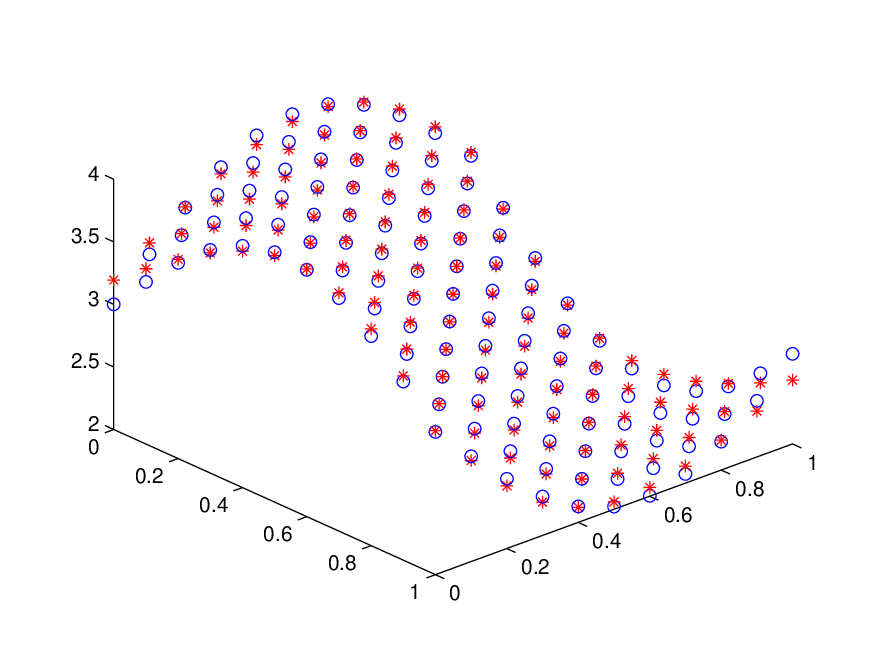}}  \ \
\subfloat[Example 2: solution of (c).]{\includegraphics[width=2.8in, trim = {1.2cm 0.7cm 1.0cm 0.1cm}, clip]{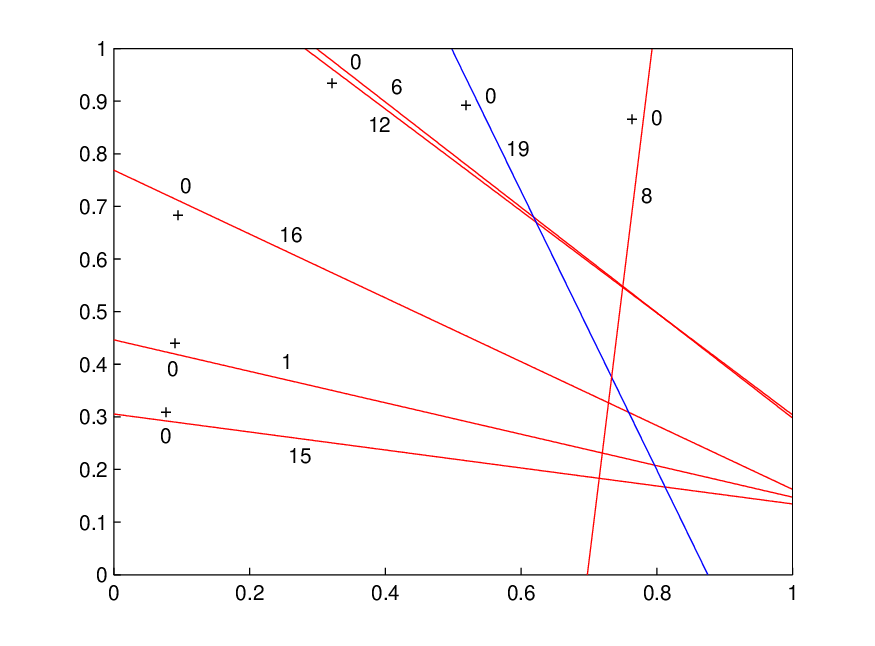}} \\
\subfloat[Example 3: data-fitting effect.]{\includegraphics[width=2.8in, trim = {1.1cm 0.7cm 0.9cm 0.1cm}, clip]{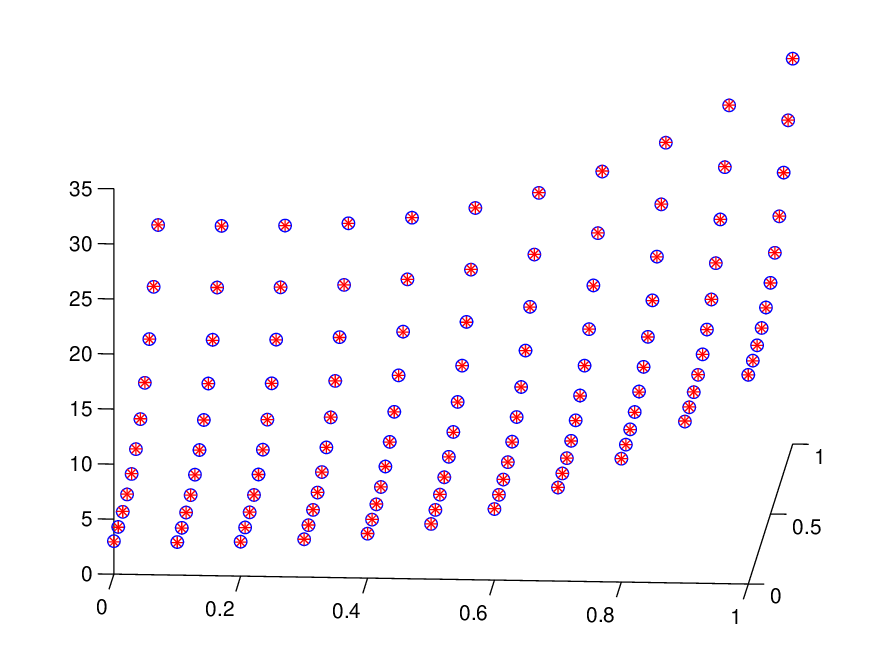}}  \ \
\subfloat[Example 3: solution of (e).]{\includegraphics[width=2.8in, trim = {1.2cm 0.7cm 1.0cm 0.1cm}, clip]{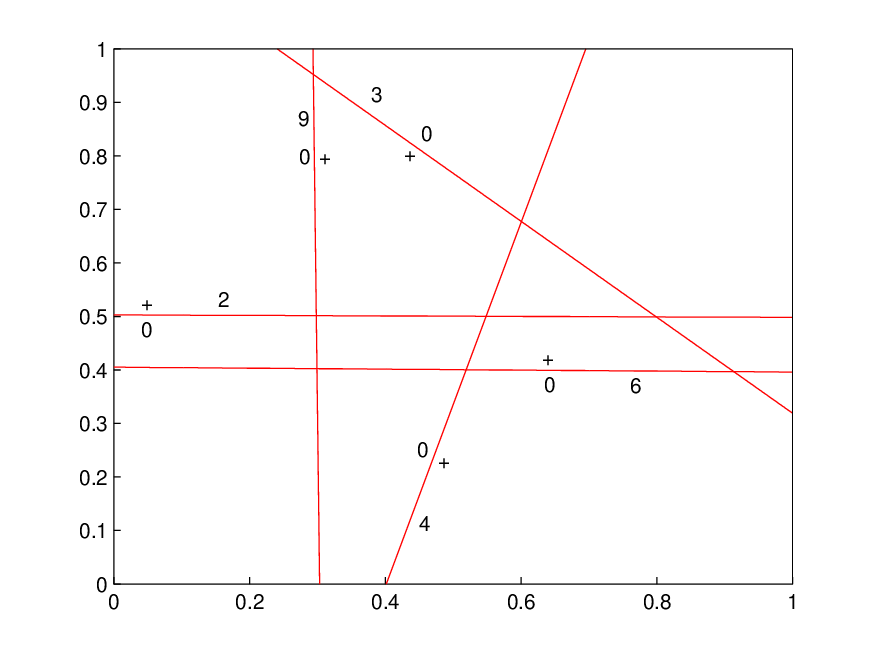}}
\caption{Solution explanation for two-dimensional input.}
\label{Fig.7}
\end{figure}

Similarly to the one-dimensional case, by discretizing $b_{\nu}'$, equation 7.12 becomes $b_{0\nu} = \max\{b_{\nu}': |\varepsilon' - \varepsilon| < \gamma_1\varepsilon\}$. The circle points of Figure \ref{Fig.7}a are from the discretization of $z = 16(x^3 + y^3)+3$ on $U = [0, 1]^2$ with step length 0.1 in both $x$ and $y$ dimensions. The asterisks are the outputs of a two-layer logistical sigmoidal neural network $\mathfrak{N}_2$ with $\Theta = 20$ units in the hidden layer. The parameters of $\mathfrak{N}_2$ are initialized by uniform distribution $U(-1, 1)$ and are trained by the back propagation algorithm with learning rate $c = 0.01$ and learning step $N = 5000$. Set $\gamma_1 = \gamma_2 = \gamma_3 = 0.01$.

Figure \ref{Fig.7}b is the solution of Figure \ref{Fig.7}a expressed by the zero-error lines (red ones in the figure) of local units; global units whose corresponding line is out of $[0,1]^2$ as well as inactivated units are not depicted in the figure. There are 6 local units, 10 global units and 4 inactivated units. The number of the global units is $\binom{5}{3} = 10$, equal to the number of the coefficients of a 2-variate polynomial of degree 3; thus any initial polynomial piece of this type could be produced, if the rank of the associated generalized Wronskian matrix satisfies the condition of theorem 9.

We see that the partition of Figure \ref{Fig.7}b resembles a standard partition of theorem 17 (e.g., Figure \ref{Fig.5}); for simplicity, we also call Figure \ref{Fig.7}b a standard partition in this section. Notice that both sigmoidal and ReLU (Figure 11b of \citet*{Huang2024}) two-layer neural networks have a standard-partition solution, providing an experimental evidence that the two types of neural networks obey a similar rule related to continuity restriction, as indicated in this paper's theorem 15 and \citet*{Huang2024}'s theorem 9.

To the two-sided bases, a strict partial order $-l_{20} \prec l_{15} \prec l_{9} \prec l_{7}$ can be obtained by changing $l_{20}$ into its negative form $-l_{20}$; so the original four lines yield two-sided bases if the associated matrix $\mathscr{A}$ of 6.55 satisfy the condition of equation 6.56. The change from $-l_{20}$ to $l_{20}$ does not influence the existence of solutions over the other strict partial order $l_2 \prec l_{11}$, as explained in proposition 8.

In Figure \ref{Fig.7}c, the function is $z = \sin{3(x + y + 1)} + 3$ and the parameter setting is the same as that of Figure \ref{Fig.7}a. Figure \ref{Fig.7}d is the solution of Figure \ref{Fig.7}c and can be nearly regarded as a standard partition. The numbers of local, global and inactivated units are 6, 7 and 7, respectively. We depict one global unit $u_{19}$ (blue line) because its line lies in $U=[0,1]^2$. The seven global units can form any linear function (the least number required is 4), provided the condition of theorem 9 is fulfilled; and they can also produce some polynomials of degree $m\ge2$.

We use the example of Figures \ref{Fig.7}e and \ref{Fig.7}f to show how the principle of smooth continuity restriction of theorem 15 applies to non-standard partition case. The data points of Figure \ref{Fig.7}e are from Figure \ref{Fig.7}a and the parameters are also identical except for the number of units $\Theta = 10$. The solution of Figure \ref{Fig.7}f includes 5 local units, 4 global units and 1 inactivated unit and is a combination of a standard partition $\mathcal{P}$ on $U \cap l_3^+$ with five regions in $U \cap l_3^0$. The polynomials on $\mathcal{P}$ can be obtained by the method of Figure \ref{Fig.7}b or \ref{Fig.7}d; after that, suppose that the output weight of $u_3$ is set to yield the polynomial on $l_9^+l_3^0l_4^0 \cap U$. Then the polynomial on $l_4^+l_3^0l_2^+ \cap U$ is determined by smooth-continuity restriction, because all the polynomials on the adjacent regions have already been previously set; the remaining three regions can be similarly dealt with.

\subsection{Tanh-Unit Case}
\begin{figure}[htp]
\captionsetup{justification=centering}
\centering
\subfloat[Example 1]{\includegraphics[width=2.9in, trim = {1.1cm 0.7cm 0.9cm 0.1cm}, clip]{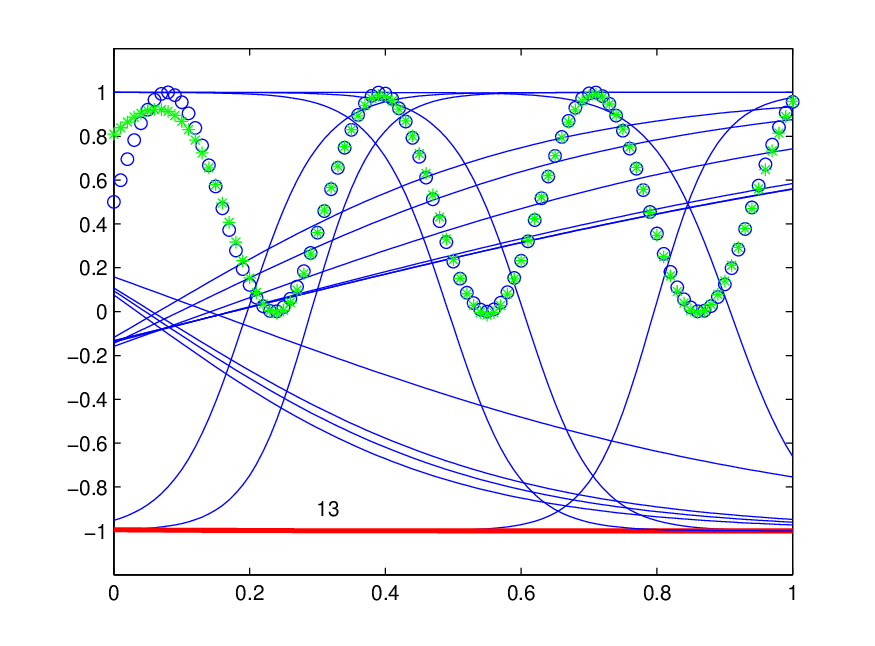}}
\subfloat[Example 2]{\includegraphics[width=2.9in, trim = {1.1cm 0.7cm 0.9cm 0.1cm}, clip]{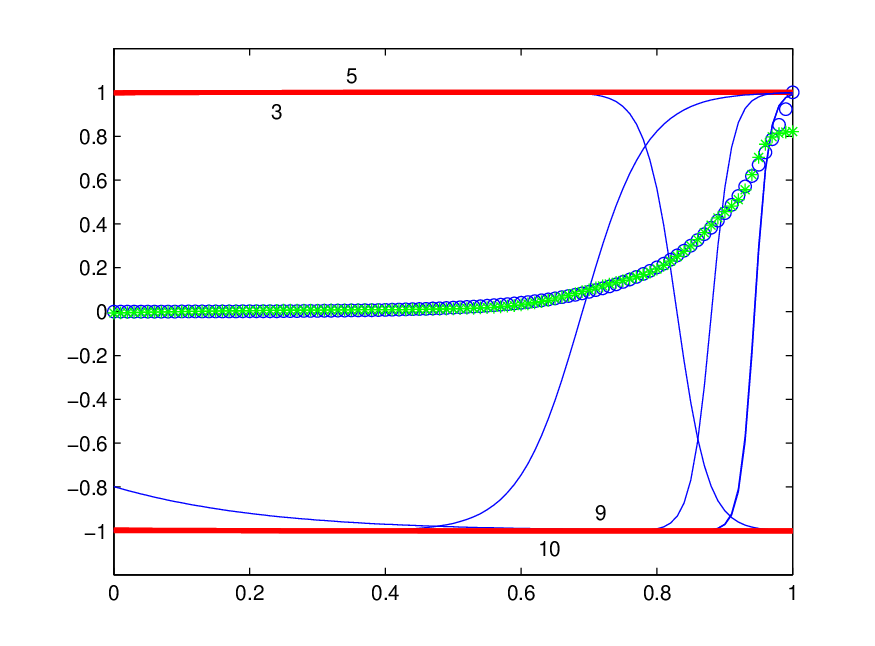}}
\caption{Units for constant term predicted by corollary 1.}
\label{Fig.8}
\end{figure}

To the solution of two-layer neural networks composed of generalized tanh units, corollary 1 indicates that there exist units exclusively for producing constant term $\mathfrak{C}$ of equation 3.84. We now use experiments to verify that.

In Figure \ref{Fig.8}a, similarly to the figures of section 7.2, the blue circles are from function $y = \sin{20x}$; the function values shown in the figure are normalized into $[0, 1]$ for visualization convenience. The hidden-layer units of network $\mathfrak{N}$ are of the tanh type whose activation function is
\begin{equation}
\sigma(x)=(e^{x}-e^{-x})/(e^{x}+e^{-x});
\end{equation}
the parameter settings of $\mathfrak{N}$ include unit number $\Theta = 20$, learning rate $c = 0.05$ and learning steps $N = 5000$; the green asterisks are produced by $\mathfrak{N}$. Use the method of section 7.1 to find out the inactivated units and do not depict them.

To judge whether a unit $u_i$ is for constant-term production, the criterion is
\begin{equation}
|(\phi_i(B) - M_i)/M_i)|<\gamma_4,
\end{equation}
where $\phi_i(x)=\sigma(w_ix+b_i)$ is the activation function of $u_i$, $M_i=\max_j{\phi_i(x_j)}$, $\gamma_4$ is the threshold and $B=0, 1$ are the two endpoints of $[0, 1]$. If equation 7.17 holds and if the unit is activated, $u_i$ is for producing a constant term. We set $\gamma_4 = 0.05$. Under this criterion, in Figure \ref{Fig.8}a, $u_{13}$ is the unit predicted by corollary 1, whose activation function is a bold red line. The blue curves correspond to other units.

Note that from equation 7.16, $\lim_{x\to +\infty}\sigma(x)=1$ is also a constant. Thus, a tanh unit can generate a constant in two ways including $\sigma(-\infty)=-1$ and $\sigma(+\infty)=1$; the associated experimental phenomenon is shown in Figure \ref{Fig.8}b. The function for Figure \ref{Fig.8}b is $y=e^{8x}$ whose discrete values are also normalized when depicted. The parameters of network $\mathfrak{N}$ are $\Theta=10$, $c=0.00001$ and $N=5000$. The threshold $\gamma_4=0.05$. As can be seen in Figure \ref{Fig.8}b, units $u_3$ and $u_5$ are of nearly constant by $\sigma(+\infty)=1$, while $u_9$ and $u_{10}$ by $\sigma(-\infty)=-1$. Although one unit is enough for a constant as in Figure \ref{Fig.8}a, redundant ones can also do that job.

\section{Summary for Black Box}
We summarize the main thoughts of the theory as follows. Let $\mathfrak{N}$ be a two-layer neural network with generalized sigmoidal or tanh units and $f(\boldsymbol{x}) \in C^m([0,1]^n)$ for $n\ge1$ be a smooth function. The hidden-layer units of $\mathfrak{N}$ are denoted by $u_i$ for $i=1,2,\dots,\Theta$, whose activation function is $\phi_i(\boldsymbol{x})=\sigma(\boldsymbol{w}_i^T\boldsymbol{x}+b_i)$.
\begin{itemize}
\item[\rm{\Romannum{1}}.] \textbf{Local approximation}. Network $\mathfrak{N}$ can approximate $f(\boldsymbol{x})$ by realizing its Taylor series expansion. This type of solution not only appears in training solutions but also is a necessary part of global solutions associated with splines.
\item[\rm{\Romannum{2}}.] \textbf{Global approximation} Over a single strict partial order, network $\mathfrak{N}$ can approximate $f(\boldsymbol{x})$ via smooth splines constructed from a piecewise linear approximation to $f(\boldsymbol{x})$'s directional-derivative hypersurface. The hidden-layer units of $\mathfrak{N}$ can be classified into two categories. One is of global units that approximate $f(\boldsymbol{x})$ through Taylor series expansions; the other is of local units, implementing polynomial pieces by the principle of smooth splines.
\item[\rm{\Romannum{3}}.] \textbf{Meaning of the parameters}. The weight vector $\boldsymbol{w}_i$ and bias $b_i$ of a local unit $u_i$ are used to generate a knot of splines, while the output weight $\lambda_i$ produces the polynomial piece corresponding to that knot. The geometric meaning of $\lambda_i$ can be reduced to the case of two-layer ReLU networks by directional-derivative operations, which is related to the two angles derived from adjacent linear pieces.
\item[\rm{\Romannum{4}}.] \textbf{Smooth-continuity restriction}. This principle mostly distinguishes network $\mathfrak{N}$ from other types of multivariate-function approximation. It provides a new representation of smooth functions that is globally correlated and contributes to a fundamental property called ``boundary-determination principle'' analogous to the boundary-value problem of differential equations.
\item[\rm{\Romannum{5}}.] \textbf{Mechanism of realizing splines}. Over a single strict partial order, each activation function $\phi_i(\boldsymbol{x})$ is approximated by a smooth spline $s_i(\boldsymbol{x})$ and the linear combination of $s_i(\boldsymbol{x})$'s leads to a desired spline $\mathcal{S}(\boldsymbol{x})$ approximating $f(\boldsymbol{x})$. To obtain the weights of $s_i(\boldsymbol{x})$'s, the recurrence relation of the polynomial pieces of a spline and the zero-error part of $\phi_i(\boldsymbol{x})$ play a central role.
\end{itemize}

\section{Discussion}
To a two-layer neural network $\mathfrak{N}$ composed of units with smooth activation functions, the conventional method for universal approximation usually resorts to Taylor series expansions, Fourier analysis or step functions. Our investigation of training solutions contributed to the spline solution. Furthermore, for higher-dimensional input, the underlying principle is more than a spline, but also includes smooth-continuity restriction, a significant feature distinguishing $\mathfrak{N}$ from other types of multivariate-function approximation.

Network $\mathfrak{N}$ shares some identical or similar basic principles with a two-layer ReLU network, such as multiple strict partial orders, (smooth) continuity restriction, zero-error hyperplanes, polynomial construction and spline implementation. Thus, the future-research problems proposed for ReLU networks in \citet*{Huang2024} are also applicable to network $\mathfrak{N}$ with smooth activation functions.

Besides the ones constructed in this paper, more solutions to be discovered are contained in the matrix-form models, including the generalized Wronskian matrix (equations 2.13 and 5.10) and spline matrix (equations 4.20 and 6.54). The study of these matrices would yield more concrete solutions that may explain engineering applications.

Our logically established theory successfully explained the experimental solution obtained by the back-propagation algorithm; and the examples of section 7 showed that the theory can manually construct training solutions through a deterministic way, instead of the usual non-deterministic gradient-descent method. Those achievements demonstrate that the research methodology of theoretical physics can be no doubt applied in the realm of artificial neural networks.

\end{document}